\documentclass[11pt]{article}
\usepackage{fullpage}

\usepackage[utf8]{inputenc} 
\usepackage[T1]{fontenc}    
\usepackage{hyperref}       
\usepackage{url}            
\usepackage{booktabs}       
\usepackage{amsfonts}       
\usepackage{nicefrac}       
\usepackage{microtype}      
\usepackage{xcolor}         
\usepackage{amsthm}
\usepackage{graphicx}
\usepackage{float}
\usepackage{mathtools}
\usepackage{amssymb, amsmath, latexsym, bm}
\usepackage{nicefrac}
\usepackage{hhline}
\usepackage{multirow}
\usepackage{pifont}%
\usepackage{natbib}

\usepackage{makecell}

\usepackage[flushleft]{threeparttable} 
\usepackage{caption}
\usepackage{multirow}
\usepackage{colortbl}

\usepackage{sidecap}

\newcommand{\squeeze}{}

\usepackage{color}

\usepackage{colortbl}
\definecolor{bgcolor}{rgb}{0.93,0.99,1}
\definecolor{bgcolor2}{rgb}{0.8,1,0.8}
\definecolor{bgcolor3}{rgb}{0.50,0.90,0.50}

\usepackage{tcolorbox}
\usepackage{pifont}
\definecolor{mydarkgreen}{RGB}{39,130,67}
\definecolor{mydarkred}{RGB}{192,25,25}
\newcommand{\green}{\color{mydarkgreen}}
\newcommand{\red}{\color{mydarkred}}
\newcommand{\cmark}{\green\ding{51}}%
\newcommand{\xmark}{\red\ding{55}}%

\usepackage{algorithm}
\usepackage{algpseudocode}

\usepackage{caption}
\usepackage{subcaption}
\usepackage{enumitem}
\usepackage{upgreek}

\usepackage[colorinlistoftodos,bordercolor=orange,backgroundcolor=orange!20,linecolor=orange,textsize=scriptsize]{todonotes}

\newcommand{\eqdef}{\overset{\text{def}}{=}}

\newcommand{\cE}{{\cal E}}

\newcommand{\cG}{{\cal G}}
\newcommand{\cH}{{\cal H}}

\newcommand{\cM}{{\cal M}}

\newcommand{\cO}{{\cal O}}

\newcommand{\cV}{{\cal V}}

\newcommand{\cW}{{\cal W}}

\def\R{\mathbb{R}}

\def\R{\mathbb R}

\def\la{\langle}
\def\ra{\rangle}

\def\y{\mathbf{y}}

\def\x{\mathbf{x}}

\newtheorem{lemma}{Lemma}
\newtheorem{theorem}{Theorem}

\newtheorem{assumption}{Assumption}

\title{\bf \huge Communication Acceleration of Local Gradient Methods via an Accelerated Primal-Dual Algorithm  with  Inexact Prox}

\author{
	  {\bf Abdurakhmon Sadiev}\thanks{This work was written while A.\ Sadiev was a research intern at KAUST during the last semester of his MS studies at the Moscow Institute of Physics and Technology, Dolgoprudny, Russia.} \\
	  KAUST\thanks{King Abdullah University of Science and Technology}\\
	 Thuwal, Saudi Arabia \\	  
	  \texttt{abdurakhmon.sadiev@kaust.edu.sa} \\
	\and
	{\bf Dmitry Kovalev} \\
	 KAUST  \\
	 Thuwal, Saudi Arabia \\	 
        \texttt{dmitry.kovalev@kaust.edu.sa} 	  \\
	 \and
	{\bf Peter Richt\'{a}rik} \\
	 KAUST \\
	 Thuwal, Saudi Arabia \\	 
	 \texttt{peter.richtarik@kaust.edu.sa} \\
	}	

\date{May 26, 2022  (revised July 6, 2022)}

\begin{document}
	
	\maketitle
	
\begin{abstract}
Inspired by a recent breakthrough of \citet{ProxSkip}, who for the first time showed that local gradient steps can lead to provable communication acceleration, we propose an alternative algorithm which obtains the same communication acceleration as their method (ProxSkip). Our approach is very different, however: it is based on the celebrated  method of \citet{chambolle2011first}, with several nontrivial modifications: i) we allow for an inexact computation of the prox operator of a certain smooth strongly convex function via a suitable gradient-based method (e.g., GD, Fast GD or FSFOM), ii) we perform a careful modification of the dual update step in order to retain linear convergence. Our general results offer the new state-of-the-art rates for the class of strongly convex-concave saddle-point problems with bilinear coupling characterized by the absence of smoothness in the dual function. When applied to federated learning, we obtain a theoretically better alternative to ProxSkip: our method requires fewer local steps ($\cO(\kappa^{1/3})$ or $\cO(\kappa^{1/4})$, compared to $\cO(\kappa^{1/2})$ of ProxSkip), and performs a deterministic number of local steps instead. Like ProxSkip, our method can be applied to optimization over a connected network, and we obtain theoretical improvements here as well.

	\end{abstract}

\clearpage	
\tableofcontents

\section{Introduction}

Communication efficiency of distributed stochastic gradient descent (SGD) can be improved, often dramatically, via a simple trick: instead of synchronizing the parameters across the parallel workers after every SGD step, let the workers perform multiple optimization steps using their local loss and data only before synchronizing. 

This trick dates back at least to two or three decades ago \citep{Mang1995}, and may be much older. Due to its simplicity, it has been repeatedly rediscovered \citep{Povey2014,SparkNet2016,COCOA+journal}. It is the basis of the famous federated averaging (FedAvg) algorithm of \citet{FedAvg2016,FedAvg2017}, which is the workhorse of federated learning \citep{FEDLEARN, FEDOPT}; see also the recent surveys on federated learning \citep{FL_survey_2020,FL-big} and federated optimization \citep{FieldGuide2021}. 

\subsection{Towards provable communication acceleration via delayed parameter synchronization}

Until recently, this simple trick resisted all attempts at an appropriate theoretical justification. Through collective effort of the federated learning community, the bounds of various local SGD methods were progressively getting better \citep{LocalDescent2019,FL-FedProx,Li2019-local-homogeneous,localGD,localSGD,localSGD-AISTATS2020,Li-local-bounded-grad-norms--ICLR2020,Blake2020,localSGD-Stich,LSGDunified2020,LFPM,FedSplit,SCAFFOLD,Nastya,FedRR,Coop2021,FedShuffle}, 
and the assumptions required to achieve them weaker. A brief overview of the progress is provided in \citep{ProxSkip}. 

However, all known theoretical rates are worse than the rate of gradient descent,  which synchronizes after every gradient step. In a recent breakthrough, \citet{ProxSkip} developed a novel local SGD method, called ProxSkip, which performs a random number of local gradient (or stochastic gradient) steps before synchronization, and proved that it enjoys strong communication acceleration properties. 
In particular, while the method needs $\cO(\kappa \log \frac{1}{\varepsilon})$ iterations, only $\cO(\sqrt{\kappa}\log\frac{1}{\varepsilon})$ of them involve communication.

\subsection{Problem formulation}
\label{sec:problem_formulation}
	
	In this paper, we consider the composite optimization problem 
	\begin{equation}
		\label{comp_problem}
		  			\min_{x\in\R^{d_x}} G(x) + F(Kx),
	\end{equation}
		where $G:\R^{d_x}\to \R$ is a smooth and strongly convex function, $F:\R^{d_y}\to \R\cup \{+\infty\}$ is a proper, closed and convex function, and $K:\R^{d_x} \rightarrow \R^{d_y}$ is a linear map. Let us define
	\begin{equation}
		\label{L_xy}
		L_{xy} \eqdef \max\left\{\|Kx\|:~ x\in\R^{d_x}, \|x\| = 1\right\},
	\end{equation}
where $\|\cdot\|$ refers to the standard Euclidean norm. Note that $L^2_{xy} \geq \lambda_{\max}(KK^{\top}) = \lambda_{\max}(K^{\top} K)$.

 It will be useful to formalize our assumptions at this point as we will refer to the various constants involved in them throughout the text.

	\begin{assumption}
		\label{as_strongly_convex}
		Function $G: \R^{d_x} \rightarrow \R$ is $\mu_x$-strongly convex, i.e.,
		\begin{equation}
			\label{strong_convexity}
		\squeeze			G(x') - G(x'') - \la\nabla G(x''), x' - x''\ra \geq \frac{\mu_x}{2}\|x'-x''\|^2, \qquad \forall x', x'' \in \R^{d_x}.
		\end{equation}
	\end{assumption}
	
	\begin{assumption}
		\label{as_smoothness}
		The function $G: \R^{d_x} \rightarrow \R$ is $L_x$-smooth, i.e.
		\begin{equation}
			\label{smoothness}
			\|\nabla G(x') - \nabla G(x'')\|\leq L_x\|x' - x''\|,\qquad \forall x', x'' \in \R^{d_x}.
		\end{equation}
	\end{assumption}
	
	\begin{assumption}
		\label{as_convexity}
		Function $F: \R^{d_y} \rightarrow \R\cup \{+\infty\}$ is proper, closed and convex.	\end{assumption}

	\begin{assumption}[See \cite{kovalev2021accelerated}]
		\label{as_L_xy}
		There exists a constant $\mu_{xy} > 0$ such that 	
		\begin{equation*}
			\mu^2_{xy} \leq 
			\begin{cases}
				\lambda^{+}_{\min}(KK^{\top}), & \text{if} \; \partial F^{\star}(y) \in {\rm range}K \; \text{ for all } \; y\in \R^{d_y},\\
				\lambda_{\min}(KK^{\top}),& \text{otherwise.} 
			\end{cases}
		\end{equation*}
	\end{assumption}

\subsection{ProxSkip}

The most general version of the ProxSkip method\footnote{This variant is called SplitSkip in their paper.} of \citet{ProxSkip} was designed to solve problems of the form \eqref{comp_problem}. In each iteration, ProxSkip evaluates the gradient of $G$ and then flips a biased coin: with probability $p$, it additionally evaluates  the proximity operator of $F^{\star}$, and performs a matrix-vector multiplication involving $K$. The method becomes  relevant to the standard optimization formulation of federated learning (FL), i.e., to the finite-sum optimization problem 
	\begin{equation}
		\label{decent_problem}
	\squeeze	\min \limits_{x\in\R^{d}}\sum^{n} \limits_{i = 1}f_i(x), 
	\end{equation}	
through its application to its consensus reformulation \begin{equation}\label{eq:consensus}\min_{x_1,\dots,x_n\in \R^d}  \left\{\sum_{i=1}^n f_i(x_i) + \psi(x_1,\dots,x_n) \right\} = \min_{x\in \R^{d_x}} G(x) + F(x),\end{equation} where $d_x\eqdef nd$, $x\eqdef (x_1,\dots,x_n)\in \R^{nd}$, $G(x)\eqdef \sum_{i=1}^n f_i(x_i) $, and $F\eqdef  \psi$ is the indicator function of the constraint $x_1=\cdots=x_n$, i.e., 
$$\psi(x_1,\dots,x_n) \eqdef  \begin{cases} 0 & \text{if} \quad x_1=\dots=x_n \\
+\infty & \text{otherwise}\end{cases}.$$
The evaluation of the proximity operator of $F$ is equivalent to averaging of the vectors $x_1,\dots,x_n$, which necessitates communication. Therefore, if $p$ is small, ProxSkip communicates very rarely. Since $G$ is block separable, the gradient steps involving $G$, taken in between two communications, correspond to gradient steps with respect to the local loss functions $\{f_i\}$ taken by the clients. See \citep{ProxSkip} for the details; we will elaborate on this as well in Section~\ref{sec:FL}.

ProxSkip solves problem \eqref{comp_problem} in  ${\cO}\left(\kappa\log \nicefrac{1}{\varepsilon}\right)$ iterations, out of which only $${\cO}\left(\sqrt{\kappa\chi} \log \frac{1}{\varepsilon}\right)$$ involve communication, where $\chi$ is a condition number measuring the connectivity of the graph (the standard setup in FL corresponds to a fully connected graph, in which case $\chi=1$; see \eqref{eq:chi} for definition). 


\subsection{Related Work} An alternative approach to achieving improved communication efficiency is through the use of communication compression via (unbiased) quantization \citep{Ternary,alistarh2017qsgd,Cnat,IntSGD}, sketching~\citep{SEGA,DIANA+} or sparsification \citep{wangni2017gradient,ATOMO,99percent}. Modern variants offering with variance reduction for the variance caused by compression~\cite{DIANA,DIANA2,sigma_k,DIANA+,IntSGD,DIANA++,Shifted}, adaptivity~\citep{IntSGD}, bidirectional compression~\citep{Cnat,DoubleSqueeze2019,Artemis2020} or acceleration~\cite{ADIANA} enjoy better theoretical rates and practical performance. Variance-reduction for communication compression has been extended to work over arbitrary connected networks~\citep{D-DIANA}, and to second-order methods~\citep{NL2021, FedNL, BL2022,Newton-3PC}. The current state-of-the-art communication complexity in the smooth nonconvex regime is offered by the MARINA~\citep{MARINA,PermK} and DASHA~\citep{DASHA} methods. 

Greedy (biased) compressors, such as Top-K sparsification~\citep{Alistarh-EF2018} or Rank-K approximation~\citep{PowerSGD}, require a different approach via an error-feedback/compensation mechanism~\citep{StichNIPS2018-memory,stich2019,DoubleSqueeze2019,EC-SGD}. For a more modern treatment of error feedback offering current state-of-the-art rates, we refer the reader to \citep{EF21,EF21BW, 3PC, EC-Katyusha}. An alternative approach based on the transformation of a biased compressor into a related induced unbiased compressor was proposed in \citep{induced}, and a unified treatment of variance reduction and error-feedback was proposed in \citep{EF-BV}.

For a  systems-oriented survey, we recommend the reader the work of \citet{Dutta-compress-Survey-2020}.

In this work we do not consider the communication compression approach to communication efficiency since this area is much more understood, and many methods already improve on the theoretical communication complexity of  vanilla GD and SGD, often by significant data and dimension-dependent margins. Instead, we focus on the practice of delayed parameter synchronization via local training, and contribute to  the theoretical foundations of this immensely popular yet poorly understood approach to achieving communication efficiency.

\section{Summary of Contributions}\label{sec:contributions}

\begin{table}[!t]
		\centering
		\caption{Summary of the key complexity results obtained by our methods APDA (Algorithm~\ref{alg:APDA}; Theorem~\ref{thm:APDA-informal}) and APDA with Inexact Prox (Algorithm~\ref{alg:APDA-Inex}; Theorem \ref{thm:APDA-Inex}) for solving the saddle-point problem \eqref{main_problem}.}
		\label{tab:1}
		{
		 \scriptsize 
			\begin{threeparttable}
				\begin{tabular}{c c c  c c c c}
					Algorithm & \begin{tabular}{c}No\\Prox\\$G$\end{tabular}&\begin{tabular}{c}No\\Prox\\$F^{\star}$\end{tabular}&\begin{tabular}{c}Works\\with\\$L_y=\infty$\end{tabular}&\begin{tabular}{c}Linear\\rate\\with\\$\mu_y=0$\end{tabular}&\# Outer Iterations\tnote{\color{red}(1)} & \# Inner Iterations\tnote{\color{red}(1)} \\
					\midrule
					CP\tnote{\color{red}(a)}& \xmark&\xmark&\cmark&\xmark& $\frac{L_{xy}}{\sqrt{\mu_x\mu_y}}$\tnote{\color{red}(2)} & uses prox of $G$\\					
					\midrule
					AltGDA\tnote{\color{red}(b)}& \cmark&\cmark&\xmark& \cmark& $\max\left\{\frac{L}{\mu_x},\frac{L^2}{\mu^2_{xy}}\right\}$& ---\tnote{\color{red}(3)}\\
					\midrule
					APDG\tnote{\color{red}(c)}& \cmark&\cmark&\xmark& \cmark & $\max\left\{\frac{\sqrt{L_xL_y}}{\mu_{xy}},A_{xy}\right\}$& ---\tnote{\color{red}(3)}\\
					\midrule
					\cellcolor{bgcolor}Alg \ref{alg:APDA} &\cellcolor{bgcolor}\xmark&\cellcolor{bgcolor}\xmark&\cellcolor{bgcolor}\cmark&\cellcolor{bgcolor}\cmark&\cellcolor{bgcolor}$A_{xy}$\tnote{\color{red}(2)}&\cellcolor{bgcolor}uses prox of $G$\\		
					\midrule								
					\cellcolor{bgcolor}Alg \ref{alg:APDA-Inex}&\cellcolor{bgcolor}\cmark&\cellcolor{bgcolor}\xmark&\cellcolor{bgcolor}\cmark&\cellcolor{bgcolor}  \cmark&\cellcolor{bgcolor}   $A_{xy}$\tnote{\color{red}(2)}&\cellcolor{bgcolor} $\max\left\{\kappa_x \kappa_{xy},\kappa_x^{\nicefrac{1}{2}}\kappa^2_{xy}\right\}$\tnote{\color{red}(4)}\\
					\midrule
					\cellcolor{bgcolor}Alg \ref{alg:APDA-Inex}&\cellcolor{bgcolor}\cmark&\cellcolor{bgcolor}\xmark&\cellcolor{bgcolor}\cmark&\cellcolor{bgcolor}  \cmark&\cellcolor{bgcolor}   $A_{xy}$\tnote{\color{red}(2)}&\cellcolor{bgcolor} $\max\left\{\kappa_x^{\nicefrac{5}{6}}\kappa_{xy} ,\kappa_{x}^{\nicefrac{1}{3}} \kappa_{xy}^2 \right\}$\tnote{\color{red}(5)} \\
					\midrule
					\cellcolor{bgcolor}Alg \ref{alg:APDA-Inex}&\cellcolor{bgcolor}\cmark&\cellcolor{bgcolor}\xmark&\cellcolor{bgcolor} \cmark&\cellcolor{bgcolor}  \cmark&\cellcolor{bgcolor}   $A_{xy}$\tnote{\color{red}(2)}&\cellcolor{bgcolor} $\max\left\{\kappa_x^{\nicefrac{3}{4}}\kappa_{xy},\kappa_x^{\nicefrac{1}{4}}\kappa_{xy}^2\right\}$\tnote{\color{red}(6)}\\
					\bottomrule
				\end{tabular}
				\begin{tablenotes}
					{\tiny
						\item [\color{red}(a)] \citet{chambolle2011first}  assume that  $G$ and $F^{\star}$ are $\mu_x$ and $\mu_y$-strongly-convex, respectively. We do not assume $F^{\star}$ to be strongly convex.
						\item [\color{red}(b)] \citet{zhang2022near}  assume that the functions $G$ and $F^{\star}$ are $L$-smooth (i.e., $L=\max\{L_x,L_y,L_{xy}\}$), and that $G$ is $\mu_x$-strongly-convex. 
						\item [\color{red}(c)]\citet{kovalev2021accelerated} assume that the functions $G$ and $F^{\star}$ are $L_x$ and $L_y$-smooth, respectively, and that $G$ is $\mu_x$-strongly-convex.
						\item [{\color{red}(1)}] For brevity, we let $\kappa_{xy}\eqdef\frac{L_{xy}}{\mu_{xy}}$, $\kappa_x \eqdef \frac{L_{x}}{\mu_{x}}$, and $A_{xy}\eqdef\max\left\{\kappa_x^{\nicefrac{1}{2}}\kappa_{xy},\kappa_{xy}^2\right\}$. We omit constant factors and a $\log \frac{1}{\varepsilon}$ factor in all expressions, for brevity. So, for example, the expression $A_{xy}$  in the case of the method of our methods should be interpreted as $\cO\left(A_{xy} \log \frac{1}{\varepsilon}\right)$. 
\item [{\color{red}(2)}] \# outer iterations = \# evaluations of the prox of $F^{\star}$.						
						\item [{\color{red}(3)}] There is no prox operator and hence no inner iterations. 
						\item [{\color{red}(4)}] The iterative method $\cM$ evaluating the prox of $G$ inexactly in this case is: $\cM$ = GD (see Lemma~\ref{lem_convergence_norm_of_gradient}). 									
						\item [{\color{red}(5)}] The iterative method $\cM$ evaluating the prox of $G$ inexactly in this case is:  $\cM$ = FGD (Fast Gradient Descent) + GD. See Lemma~\ref{lem_convergence_norm_of_gradient}. 													
						\item [{\color{red}(6)}] The iterative method $\cM$ evaluating the prox of $G$ inexactly in this case is: $\cM$ = FSFOM + FGD (Fast Gradient Descent). See Lemma~\ref{lem_convergence_norm_of_gradient}. 				
					}
				\end{tablenotes}
		\end{threeparttable}}
	\end{table}
	
Inspired by the results of \citet{ProxSkip}, we propose an alternative and substantially different algorithm which obtains the same guarantees for the number of prox evaluations (wrt $F$) as ProxSkip, but has better guarantees for the number gradient steps (wrt $G$) in between the prox evaluations. Below we summarize the main contributions:

\subsection{Saddle-point formulation}  Unlike \citet{ProxSkip}, we consider the saddle-point reformulation of \eqref{comp_problem}
	\begin{equation}
		\label{main_problem}
		\min_{x\in \R^{d_x}}\max_{y \in \R^{d_y}}\left\{G(x) +\left\la y, K x\right\ra -F^{\star}(y)\right\},
	\end{equation}
	where $F^{\star}(y)\eqdef \sup_{y'\in \R^{d_y}} \{ \langle y, y' \rangle - F(y')\} $ is the convex conjugate of $F$. Since $F$ is proper, closed and convex, so is $F^{\star}$. We assume throughout that \eqref{main_problem} is solvable, i.e., there exists at least one solution $(x^{\star}, y^{\star})$. Such a solution then satisfies the first-oder optimality conditions\footnote{Whenever we invoke Assumption~\ref{as_smoothness} ($L_x$-smoothness of $G$), we have $\partial G(x)=\{\nabla G(x)\}$, and hence the first condition can be replaced by $0 = \nabla G(x^{\star}) + K^{\top}y^{\star}$.}
	\begin{equation}
		\label{opt_conditions}
		0 \in \partial G(x^{\star}) + K^{\top}y^{\star}, \qquad  0 \in \partial F^{\star}(y^{\star}) - Kx^{\star},
	\end{equation}
where $\partial$ denotes the subdifferential. By working with this reformulation, we can tap into the rich and powerful philosophical and technical toolbox offered by proximal-point theory, fixed point theory,  and primal-dual methods, which facilitates the algorithm development and analysis. This ultimately enables us to shed new light on the nature of local gradient-type steps as inexact computations of the prox operator of $G$ in a new  appropriately designed Accelerated Primal-Dual Algorithm (APDA; see Algorithm~\ref{alg:APDA}).

\subsection{Modifications of Chambolle-Pock} Our  Algorithms~\ref{alg:APDA} and \ref{alg:APDA-Inex} are inspired by the celebrated Chambolle-Pock method \citep{chambolle2011first},
but with several  important and nontrivial modifications. While Chambolle-Pock achieves linear convergence when both $G$ and $F^{\star}$ are strongly convex, $F^{\star}$ is merely convex in our setting.  Our modifications are:
	
\begin{itemize}
\item[i)] Inspired by the ideas of \citet{kovalev2021accelerated}, we perform a careful modification of the dual update step (update of $y$)  in order to retain linear convergence despite lack of strong convexity in $F^{\star}$. On the other hand, in contrast to the method of \citet{kovalev2021accelerated}, we do not assume $F^{\star}$ to be smooth. This modification leads to a new method, which we call APDA (Algorithm~\ref{alg:APDA}). APDA relies on the evaluation of the prox operators of both $G$ and $F^{\star}$. 
\item [ii)] Next, we remove the reliance on the prox operator of $G$, and instead allow for its inexact evaluation via a suitable user-defined gradient-based method, which we call $\cM$ (see \eqref{alg:APDA-Inex} and Lemma~\ref{lem_convergence_norm_of_gradient}). We call the resulting method ``APDA with Inexact Prox'' (Algorithm~\ref{alg:APDA-Inex}). The choice of  method $\cM$ will have a strong impact on the number of inexact/local steps, and this is one of the  places in which we improve upon the results of \citet{ProxSkip}. 
\end{itemize}

\subsection{General theory} Our general complexity results for Algorithms \ref{alg:APDA} and \ref{alg:APDA-Inex}, covered in Theorems~\ref{thm:APDA} and \ref{thm:APDA-Inex}, respectively, contrasted with the key baselines, are summarized in Table~\ref{tab:1}). While the method of \citep{chambolle2011first} needs $F^{\star}$ to be strongly convex to obtain a linear rate, we only need convexity. While the AltGDA~\citep{zhang2022near} and APDG~\citep{kovalev2021accelerated} methods enjoy linear rates without strong convexity of $F^{\star}$,  both require $F^{\star}$ to be $L_y$-smooth. In contrast, we do not need this assumption (i.e., we allow $L_y=\infty$). Our methods are the first to obtain linear convergence rates in the regime when $G$ is $L_x$-smooth and $\mu_x$-strongly-convex, and $F^{\star}$ is merely (proper closed and) convex, without requiring it to be $L_y$-smooth, nor $\mu_y$-strongly-convex. This is important in some applications. Our two methods offer two alternative ways of dealing with this regime: while APDA (Algorithm~\ref{alg:APDA}) relies on the evaluation of the prox of $G$, APDA with Inexact Prox (Algorithm~\ref{alg:APDA-Inex}) does not. As we shall see, the latter method has an important application in federated learning.

	\begin{table}[t]
		\centering
		\caption{Summary of our general convergence results provided by Theorem~\ref{thm:APDA-Inex} for Algorithm~\ref{alg:APDA-Inex} (APDA with Inexact Prox) and Theorem~\ref{thm:APDA-3} for Algorithm~\ref{alg:APDA-3} (APDA with Inexact Prox and Accelerated Gossip) for solving the saddle-point  reformulation \eqref{new_problem} of the federated learning problem \eqref{decent_problem}.}
		\label{tab:2}
		{
			 \scriptsize
			\begin{threeparttable}
				\begin{tabular}{ccccccc}
					\multirow{5}{*}{Algorithm}&\multirow{5}{*}{\begin{tabular}{c} Method $\cM$\tnote{\color{red}(2)} \\ for \\Inexact Prox\end{tabular}}&\multirow{5}{*}{\begin{tabular}{c} Deter-\\ ministic \\  \# comm.\\ rounds\end{tabular}}&\multicolumn{2}{c}{Centralized case} &\multicolumn{2}{c}{Decentralized case}\\
					\cmidrule(lr){4-7}
					&&&\begin{tabular}{c} Optimal \\ \#comm.\\ rounds?\end{tabular} &\begin{tabular}{c} \#Local \\steps\\ per round \end{tabular} & \begin{tabular}{c} Optimal\\ \#comm.\\ rounds?\end{tabular} & \begin{tabular}{c} \#Local \\steps\\ per round \end{tabular} \\
					\midrule
					\multirow{2}{*}{\begin{tabular}{c}ProxSkip\\ \citep{ProxSkip}\end{tabular}}&\multirow{2}{*}{GD}& \multirow{2}{*}{\xmark}& \multirow{2}{*}{\cmark}&\multirow{2}{*}{$\mathcal{O}\left(\sqrt{\kappa}\right)$} &\multirow{2}{*}{\cmark\tnote{\color{red}(1)}}&\multirow{2}{*}{$\widetilde{\mathcal{O}}\left(\sqrt{\kappa}\right)$\tnote{\color{red}(1)}}\\
					&&&&&& \\
					\midrule
					\multirow{4}{*}{Alg \ref{alg:APDA-Inex}; Thm~\ref{thm:APDA-Inex}} &\cellcolor{bgcolor} GD&\cellcolor{bgcolor}\cmark&\cellcolor{bgcolor}\cmark&\cellcolor{bgcolor}$\widetilde{\mathcal{O}}\left(\sqrt{\kappa}\right)$ &\cellcolor{bgcolor}\xmark&\cellcolor{bgcolor}$\widetilde{\mathcal{O}}\left(\sqrt{\kappa}\right)$ \\
					\cmidrule(lr){2-7}
					&\cellcolor{bgcolor}FGD+GD&\cellcolor{bgcolor}\cmark&\cellcolor{bgcolor}\cmark&\cellcolor{bgcolor}$\widetilde{\mathcal{O}}\left(\sqrt[3]{\kappa}\right)$ &\cellcolor{bgcolor}\xmark&\cellcolor{bgcolor}$\widetilde{\mathcal{O}}\left(\sqrt[3]{\kappa}\right)$ \\
					
					\cmidrule(lr){2-7}
					&\cellcolor{bgcolor}FGD+FSFOM&\cellcolor{bgcolor}\cmark&\cellcolor{bgcolor}\cmark&\cellcolor{bgcolor}$\widetilde{\mathcal{O}}\left(\sqrt[4]{\kappa}\right)$ &\cellcolor{bgcolor}\xmark&\cellcolor{bgcolor}$\widetilde{\mathcal{O}}\left(\sqrt[4]{\kappa}\right)$ \\
					
					\midrule
					\multirow{4}{*}{Alg \ref{alg:APDA-3};  Thm~\ref{thm:APDA-3}} &\cellcolor{bgcolor}GD&\cellcolor{bgcolor}\cmark&\cellcolor{bgcolor}\cmark&\cellcolor{bgcolor}$\widetilde{\mathcal{O}}\left(\sqrt{\kappa}\right)$ &\cellcolor{bgcolor}\cmark&\cellcolor{bgcolor}$\widetilde{\mathcal{O}}\left(\sqrt{\kappa}\right)$ \\
					
					\cmidrule(lr){2-7}
					&\cellcolor{bgcolor}FGD+GD&\cellcolor{bgcolor}\cmark&\cellcolor{bgcolor}\cmark&\cellcolor{bgcolor}$\widetilde{\mathcal{O}}\left(\sqrt[3]{\kappa}\right)$ &\cellcolor{bgcolor}\cmark&\cellcolor{bgcolor}$\widetilde{\mathcal{O}}\left(\sqrt[3]{\kappa}\right)$ \\
					
					\cmidrule(lr){2-7}
					&\cellcolor{bgcolor}FGD+FSFOM&\cellcolor{bgcolor}\cmark&\cellcolor{bgcolor}\cmark&\cellcolor{bgcolor}$\widetilde{\mathcal{O}}\left(\sqrt[4]{\kappa}\right)$ &\cellcolor{bgcolor}\cmark&\cellcolor{bgcolor}$\widetilde{\mathcal{O}}\left(\sqrt[4]{\kappa}\right)$ \\
					
					\bottomrule
					
				\end{tabular}
				\begin{tablenotes}
					{\tiny
						\item [{\color{red}(1)}] This is true only when $\kappa \leq \chi$. 
						\item [{\color{red}(1)}] GD = Gradient Descent; FGD = Fast Gradient Descent (i.e., Nesterov's accelerated GD); FSFOM = a fixed-step first-order method from \citep{kim2021optimizing}. 
					}
				\end{tablenotes}
		\end{threeparttable}}
	\end{table}

\subsection{Federated learning and a third method} When applied to the distributed/federated problem \eqref{decent_problem} (see Section~\ref{sec:FL}), APDA with Inexact Prox (Algorithm~\ref{alg:APDA-Inex}) turns out to be a theoretically better alternative to ProxSkip \citep{ProxSkip}. In the centralized case, our method requires the same optimal number of communication rounds ($\widetilde{\cO}(\sqrt{\kappa})$, where $\kappa = \nicefrac{L_x}{\mu_x}$) as ProxSkip, but requires fewer local gradient-type steps ($\cO(\kappa^{1/3})$ or $\cO(\kappa^{1/4})$, compared to $\cO(\kappa^{1/2})$ of ProxSkip, depending on the choice of the inner method $\cM$). Like ProxSkip, our method can be applied to optimization over a connected network, and we obtain theoretical improvements in this decentralized scenario as well. However, in the decentralized regime, neither ProxSkip nor Algorithm~\ref{alg:APDA-Inex}  obtain the optimal bound for the number of communication rounds. For this reason, we propose a third method (Algorithm \ref{alg:APDA-3}) which employs an accelerated gossip routine to remedy this situation. It is also notable that while ProxSkip uses a random number of local steps, all our methods perform a deterministic number of local steps.  Our complexity results are summarized in Table \ref{tab:2}.

\section{From Proximal Point Method to Chambolle-Pock} \label{sec:insights1}

In this section, we briefly motivate the development of the celebrated Chambolle-Pock method which acts as a starting point of our algorithm design.

\subsection{Proximal-Point Method for finding zeros of monotone operators}
	
Our starting point is the general problem of finding a zero of an (set-valued) operator  $A:\cH\rightarrow 2^\cH$, where $\cH$ is a Hilbert space, i.e., find $z\in \cH$ such that \begin{equation}\label{eq:A}0\in A(z).\end{equation} If $A$ is maximally monotone, its resolvent $({\rm Id} + \eta A)^{-1}$ is single valued, nonexpansive, and has full domain. Moreover, $0\in A(z)$ iff $z = ({\rm Id} + \eta A)^{-1}(z)$.  The corresponding fixed point iteration, i.e.,  $z^{k+1} = ({\rm Id} + \eta A)^{-1}(z^k)$, is called the proximal point method (PPM)  \citep{rockafellar1976monotone}.  This can be equivalently written as $z^k \in ({\rm Id} + \eta A)(z^{k+1})$, and subsequently as $$z^{k+1} \in z^k -\eta A(z^{k+1}).$$ From now on, for simplicity only, we will ignore the fact that in general, $A(z^{k+1})$ is a set, and will write $z^{k+1} = z^k -\eta A(z^{k+1})$ instead to mean the same thing, i.e., that there exists $u \in A(z^{k+1})$ such that $z^{k+1} = z^k -\eta u$.

	
\subsection{PPM applied to the saddle-point problem}\label{sec:PPM}	

The optimality conditions \eqref{opt_conditions} of the saddle point problem \eqref{main_problem} can be written in the form \eqref{eq:A} with $z = (x;y) \in \R^{d_x}\times \R^{d_y}$ as follows\footnote{We replaced $\nabla G$ by $\partial G$ here as the beginning of our story does not require $G$ to be smooth.}:
\begin{equation}\label{eq:898hfddf}\begin{pmatrix} 0 \\ 0\end{pmatrix} \in  A\begin{pmatrix} x \\ y\end{pmatrix} \eqdef \begin{pmatrix} \partial G(x) + K^\top y \\ \partial F^{\star}(y) - K x\end{pmatrix} .\end{equation}
Allowing for different stepsizes $\eta_x, \eta_y$ for each block of the vector $z=(x;y)$, PPM applied to \eqref{eq:898hfddf} 
takes the form
\begin{eqnarray*}x^{k+1} &=& x^k -\eta_x \left(\partial G(x^{k+1}) + K^{\top}{\color{red}{y^{k+1}}} \right) \label{eq:98f8fd}\\
y^{k+1} &=& y^k -\eta_y \left(\partial F^{\star}(y^{k+1}) - Kx^{k+1}\right). \label{eq:kjb9d8h9fd}\end{eqnarray*}
	
	The main advantage of this method is its unboundedly fast convergence rate under weak assumptions. According to Theorem \ref{th_conv_prox_point_method}, the proof of which we provide in the appendix for completeness, if $G$ and $F^{\star}$ are proper and closed, $G$ is $\mu_x$ strongly convex and  $F^{\star}$ is $\mu_y$ strongly convex, then any choice of stepsizes $\eta_x > 0$ and $\eta_y > 0$ (yes, without an upper bound!), PPM find  an $\varepsilon$-accurate solution in
	\begin{equation}\label{eq:PPM-rate-main}
	\squeeze			\cO\left(\left(1+\frac{1}{\min\{\eta_x\mu_x, \eta_y\mu_y\}}\right)\log\frac{1}{\varepsilon}\right)
	\end{equation}
	iterations.  Unfortunately, PPM is not implementable since in order to compute $x^{k+1}$, we need to know $\color{red}{y^{k+1}}$, and vice versa. 
			
\subsection{Chambolle-Pock: Making PPM implementable, and fast}		\label{sec:CP}

In order to overcome the above problem, \citet{chambolle2011first} proposed to replace $\color{red}{y^{k+1}}$ with $\color{red}y^k$ (see Algorithm 1 in \citep{chambolle2011first}), which leads to
\begin{eqnarray*}x^{k+1} &=& x^k -\eta_x \left(\partial G(x^{k+1}) + K^{\top}{\color{red}{y^{k}}} \right)\\
y^{k+1} &=& y^k -\eta_y \left(\partial F^{\star}(y^{k+1}) - Kx^{k+1}\right).\end{eqnarray*}
	Although this method is implementable, it has its own disadvantages, one of which is its weak iteration complexity bound
	\begin{equation}\label{eq:CP-1}
		\squeeze		\cO\left(\frac{L^2_{xy}}{\mu_x\mu_y}\log\frac{1}{\varepsilon}\right).
	\end{equation}
	
\citet{chambolle2011first} proposed to fix this problem via an extrapolation step	of the dual variable (see Algorithm 3 in \citep{chambolle2011first}):
\begin{eqnarray*}x^{k+1} &=& x^k -\eta_x \left(\partial G(x^{k+1}) + K^{\top}{\color{red}{\bar{y}^{k}}} \right) \\
y^{k+1} &=& y^k -\eta_y \left(\partial F^{\star}(y^{k+1}) - Kx^{k+1}\right) \\
{\color{red}{\bar{y}^{k+1}}} &=& y^{k+1}+\theta(y^{k+1} - y^k). \end{eqnarray*}
	
This new method enjoys the much better iteration complexity bound 
	\begin{equation}\label{eq:CP-2}
		\squeeze		\cO\left(\frac{L_{xy}}{\sqrt{\mu_x\mu_y}}\log\frac{1}{\varepsilon}\right).
	\end{equation}

\section{Accelerated Primal-Dual Algorithm (Algorithm~\ref{alg:APDA})}\label{sec:APDA}
	
	
Recall that the Chambolle-Pock method requires $F^{\star}$ to be strongly-convex to obtain a linear convergence rate. However, in our setting, $F^{\star}$ is not strongly convex\footnote{We would need to assume $F$ to be smooth to ensure that $F^{\star}$ is strongly convex. However, we do not want to do this as this is not satisfied ion many scenarios,  in particular, in our key application to federated learning.} (see Assumption~\ref{as_convexity}), and Chambolle-Pock method does not converge linearly in this scenario. 
	
\subsection{Modifying Chambolle-Pock to preserve linear rate without strong convexity of $F^{\star}$}	

		\begin{algorithm}[!t]
		\caption{APDA}
		\label{alg:APDA}
		\begin{algorithmic}[1]
			\State \textbf{Input}: Initial point $(x^0, y^0) \in \R^{d_x}\times\R^{d_y}$, $\Bar{y}^0 = y^0$; Step sizes $\eta_x, \eta_y, \beta_y >0$, $\theta \in [0, 1]$
			\For{$k = 0,1,\dots$}			
			\State ${\color{red}x^{k+1}} = x^k -\eta_x\left(\nabla G({\color {red}x^{k+1}}) + K^{\top}\Bar{y}^k\right)$
			\State $y^{k+1} = y^k -\eta_y\left(\partial F^{\star}(y^{k+1}) - K {\color{red}x^{k+1}}\right) - \eta_y\beta_y K\left(K^{\top}y^k+\nabla G({\color{red}x^{k+1}})\right)$


			\State $\Bar{y}^{k+1} = y^{k+1} + \theta\left(y^{k+1}-y^k\right)$
			\EndFor 
		\end{algorithmic}
	\end{algorithm}

 To obtain a linear rate, we modify the dual update step of the algorithm using a trick proposed by \citet{kovalev2021accelerated} that was shown  to work in the regime when $F^{\star}$ is smooth; the innovation here is that we do not need this assumption (see Table~\ref{tab:1}). From this point onwards,  we will also need to assume  $G$ to be $L_x$-smooth (see Assumption~\ref{as_smoothness}). In particular, we propose to {\color{blue}modify the update step for $y^{k+1}$ in the Chambolle-Pock method  as follows:}
\begin{eqnarray*}x^{k+1} &=& x^k -\eta_x \left(\nabla G(x^{k+1}) + K^{\top}{\color{red}{\bar{y}^{k}}} \right) \\
y^{k+1} &=& y^k -\eta_y \left(\partial F^{\star}(y^{k+1}) - Kx^{k+1}\right) - \color{blue}{\eta_y\beta_y K\left(K^{\top}y^k+\nabla G(x^{k+1})\right)} \\
{\color{red}{\bar{y}^{k+1}}} &=& y^{k+1}+\theta(y^{k+1} - y^k) .\end{eqnarray*}
This is a new method, which we call APDA (formalized as Algorithm~\ref{alg:APDA}). 

\subsection{APDA converges linearly}

Our first result shows that APDA indeed converges linearly, without the need for $F^{\star}$ to be strongly convex.

\begin{theorem}[Convergence of APDA; informal]\label{thm:APDA-informal}Let Assumptions~\ref	{as_strongly_convex}, \ref{as_smoothness},\ref{as_convexity} and \ref{as_L_xy} hold. Then,  with a suitable selection of stepsizes, APDA (Algorithm~\ref{alg:APDA})  solves problem \eqref{main_problem} in \begin{equation}
		\squeeze		\cO\left(\max\left\{\sqrt{\frac{L_x}{\mu_x}}\frac{L_{xy}}{\mu_{xy}}, \frac{L^2_{xy}}{\mu^2_{xy}}\right\}\log\frac{1}{\varepsilon}\right)
		\end{equation}
 iterations.		
		\end{theorem}
The formal statement and proof can be found in the appendix (see Theorem~\ref{thm:APDA}).


\section{Accelerated Primal-Dual Algorithm with Inexact Prox (Algorithm~\ref{alg:APDA-Inex})}	

The key disadvantage of APDA is that it requires the evaluations of the proximity operator of $G$, which can be  very expensive in some applications. To remedy the situation, we first
notice that Step 3 of APDA can be equivalently written in the form
\begin{equation}\squeeze x^{k+1} =\arg\min \limits_{x\in \R^{d_x}} \left\{\Psi^k (x) \eqdef G(x) + \frac{1}{2\eta_x} \left\|x - \left(x^k - \eta_xK^{\top}\Bar{y}^k\right)\right\|^2 \right\}; \label{auxiliary_problem}\end{equation}
that is, this step involves the evaluation of the prox of $G$. 
The key idea of this section is to  replace this by an inexact prox computation via a suitably selected iterative method $\cM$ (this is the method performing the inner iterations in Table~\ref{tab:1}). This leads to our next method: APDA with Inexact Prox (Algorithm~\ref{alg:APDA-Inex}).

	\begin{algorithm}[!h]
		\caption{APDA with Inexact Prox}
		\label{alg:APDA-Inex}
		\begin{algorithmic}[1]
			\State \textbf{Input}: Initial point $(x^0, y^0) \in \R^{d_x}\times\R^{d_y}$, $\Bar{y}^0 = y^0$; Step sizes $\eta_x, \eta_y, \beta_y >0$, $\theta \in [0, 1]$; \# inner iterations $T$
			\For{$k = 0,1,\dots$}
			\State Find $\hat{x}^k$ as a final point of $T$ iteration of some method $\cM$ for following problem:
			\begin{equation*}				
					\squeeze	{\color{red}\hat{x}^k} \approx \arg\min \limits_{x\in \R^{d_x}}\left\{\Psi^k (x) \eqdef G(x) + \frac{1}{2\eta_x} \left\|x - \left(x^k - \eta_xK^{\top}\Bar{y}^k\right) \right\|^2  \right\}
			\end{equation*}
			\State $x^{k+1} = x^k -\eta_x\left(\nabla G({\color{red}\hat{x}^k}) + K^{\top}\Bar{y}^k\right)$
			\State $y^{k+1} = y^k -\eta_y\left(\partial F^{\star}(y^{k+1}) - K{\color{red}\hat{x}^k}\right) - \eta_y\beta_y K\left(K^{\top}y^k+\nabla G({\color{red}\hat{x}^k})\right)$
			\State $\Bar{y}^{k+1} = y^{k+1} + \theta\left(y^{k+1}-y^k\right)$
			\EndFor 
		\end{algorithmic}
	\end{algorithm}
	

\subsection{Gradient methods for finding a stationary point of convex functions}
	
A key feature of Algorithm \ref{alg:APDA-Inex} is its reliance on a subroutine $\cM$ for an inexact evaluation of the prox of $G$ via solving the auxiliary problem \eqref{auxiliary_problem}. Our theory requires the method $\cM$ to be able to output, after $T$ iterations, a point $\hat{x}^k$ such that
	\begin{equation}
		\squeeze		\|\nabla \Psi^k(\hat{x}^k)\|^2 \leq \cO\left(\frac{1}{T^{\alpha}}\right),
	\end{equation}
where  $\alpha\geq 2$. In other words, we require a reduction of the squared norm of the gradient with a fast sublinear rate. In the next lemma, we present three examples of such methods.	

	\begin{lemma}
		\label{lem_convergence_norm_of_gradient}
		Let $\Psi:\R^{d_x}\to \R$ be an  $L$-smooth convex function, and let $w^{\star}$ be a minimizer of $\Psi$.   Then there exists a gradient-based method $\cM$ which after $T$ iterations outputs a point $w^T$ satisfying	 
			\begin{equation}
			\label{convergence_norm_grad}
		\squeeze			\|\nabla \Psi(w^{T})\|^2 \leq \frac{AL^2\|w^0-w^{\star}\|^2}{T^{\alpha}},
		\end{equation}
for all starting points $x^0\in \R^{d_x}$ and some universal constant  $A>0$.
		In particular, 
\begin{itemize}		
\item[(i)]   if  $\cM$ is GD, then	$\|\nabla \Psi(w^{T})\|^2 \leq \frac{4L^2\|w^0-w^{\star}\|^2}{T^{2}}$,\\
\item[(ii)]  if  $\cM$ is  a combination\footnote{The first half of the iterations is solved via the Fast Gradient Descent (FGD) method of \citet{NesterovBook}, and the second half via Gradient Descent (GD).} of Fast GD~\citep{NesterovBook}  and GD, then $\|\nabla \Psi(w^{T})\|^2 \leq \frac{64L^2\|w^0-w^{\star}\|^2}{T^{3}}$,\\
\item[(iii)]  if  $\cM$ is  a combination\footnote{The first half of the iterations is solved via the Fast Gradient Descent (FGD) method of \citet{NesterovBook}, and the second half via the fixed-step first-order method (FSFOM) of \citet{kim2021optimizing}.} of Fast GD~\citep{NesterovBook} and FSFOM~\citep{kim2021optimizing}, then $\|\nabla \Psi(w^{T})\|^2 \leq \frac{256L^2\|w^0-w^{\star}\|^2}{T^{4}}.$
\end{itemize}	
	\end{lemma}

Let $w^{\star k}=\arg\min\limits_{w\in \R^{d_x}} \Psi^k(w)$.	Since $\Psi^k$ is $\left(L_x +\eta_x^{-1} \right)$-smooth,  Lemma \ref{lem_convergence_norm_of_gradient} implies that $T$ iterations of a method $\cM$ satisfying \eqref{convergence_norm_grad} applied to the auxiliary problem \eqref{auxiliary_problem} with starting point $w^0=x^k$ yield  point  $w^T=\hat{x}^{k}$	for which \begin{equation}
		\label{complexity_aux_problem}
	\squeeze			\|\nabla \Psi^k(\hat{x}^{k})\|^2 \leq \frac{A\left(\eta^{-1}_x+ L_x\right)^2\|x^k-w^{\star k}\|^2}{T^{\alpha}} = \frac{A\left(1+ \eta_x L_x\right)^2\|x^k-w^{\star k}\|^2}{\eta^2_x T^{\alpha}}.
	\end{equation}

\subsection{APDA with Inexact Prox converges linearly}

Now we can provide the main theorem with the total complexity of gradient computation $\nabla G$ and proximity operator computation $\partial F^{\star}$. 
	\begin{theorem}
		\label{thm:APDA-Inex}		
		Let Assumptions \ref{as_strongly_convex}, \ref{as_smoothness}, \ref{as_convexity}, \ref{as_L_xy} hold. Then there exist  parameters of Algorithm \ref{alg:APDA-Inex} such that the total \# of evaluations of prox $F^\star$ and the total number evaluations of the gradient of $\nabla G$ to find  an $\varepsilon$ solution of problem \eqref{main_problem} are
		\begin{equation}
			\label{complexity_alg_1}
		\squeeze			\mathcal{O}\left(\max\left\{\sqrt{\frac{L_x}{\mu_x}}\frac{L_{xy}}{\mu_{xy}},\frac{L^2_{xy}}{\mu^2_{xy}}\right\}\log\frac{1}{\varepsilon}\right), \; \cO\left(\max\left\{\left(\frac{L_x}{\mu_x}\right)^{\frac{2+\alpha}{2\alpha}}\frac{L_{xy}}{\mu_{xy}},\sqrt[\alpha]{\frac{L_x}{\mu_x}}\frac{L^2_{xy}}{\mu^2_{xy}}\right\}\log\frac{1}{\varepsilon}\right),
		\end{equation}
		respectively.
		In particular,
\begin{itemize}
\item [(i)] if the inner method $\cM$  is GD, then the total number of $\nabla G$ computations is equal to
			\begin{equation}\label{eq:u9fd8g9f8dh--}
		\squeeze				\cO\left(\max\left\{\frac{L_x}{\mu_x}\frac{L_{xy}}{\mu_{xy}},\sqrt{\frac{L_x}{\mu_x}}\frac{L^2_{xy}}{\mu^2_{xy}}\right\}\log\frac{1}{\varepsilon}\right),
			\end{equation}
\item [(ii)] if the inner method $\cM$ is combination of Fast GD and GD, then the total number of $\nabla G$ computations is equal to
			\begin{equation}\label{eq:biogh98df-9878yhfd}
		\squeeze				\cO\left(\max\left\{\left(\frac{L_x}{\mu_x}\right)^{\frac{5}{6}}\frac{L_{xy}}{\mu_{xy}},\sqrt[3]{\frac{L_x}{\mu_x}}\frac{L^2_{xy}}{\mu^2_{xy}}\right\}\log\frac{1}{\varepsilon}\right),
			\end{equation}
\item [(iii)] if the inner method $\cM$  is combination of Fast GD and FSFOM, then the total number of $\nabla G$ computations is equal to
			\begin{equation}\label{eq:hgiuygfdUYGUYG089f}
		\squeeze		\cO\left(\max\left\{\left(\frac{L_x}{\mu_x}\right)^{\frac{3}{4}}\frac{L_{xy}}{\mu_{xy}},\sqrt[4]{\frac{L_x}{\mu_x}}\frac{L^2_{xy}}{\mu^2_{xy}}\right\}\log\frac{1}{\varepsilon}\right).
			\end{equation}
\end{itemize}			

	\end{theorem}
	
The proof relies  on several lemmas; their statements and the proof of the theorem  can be found in the appendix.

Note that our way of performing inexact computation of prox of $G$ allows us to keep the same complexity as APDA (Algorithm \ref{alg:APDA}) in terms of the number of evaluations of the prox of $F^{\star}$. When  $\cM$ is chosen to be Gradient Decent ($\alpha =2$), the number of computations of the gradient of $\nabla G$,  given by \eqref{eq:u9fd8g9f8dh--},
 is larger than the number of computations of the prox of $G$ for APDA. Fortunately, we can reduce this if GD is replaced with a faster method. For example, if we choose $\cM$ to be a simple combination of Fast GD (FGD) and GD, in which case $\alpha=3$, Theorem \ref{thm:APDA-Inex} says that the number of computations of the gradient of $\nabla G$ is can be reduced to \eqref{eq:biogh98df-9878yhfd} (this choice is mentioned in the second-to-last row of Table~\ref{tab:2}). A further reduction is possible if we instead employ a combination of FSFOM and Fast GD; see \eqref{eq:hgiuygfdUYGUYG089f} and 
the last row of Table~\ref{tab:2}.

%
	


	\section{Accelerated Primal Dual Algorithm with Inexact Prox and Accelerated Gossip (Algorithm~\ref{alg:APDA-3})}	\label{sec:FL}

	In this section, we consider the most significant applications of Algorithm \ref{alg:APDA-Inex}: decentralized optimization over a network $\cG = \left(\cV, \cE\right)$, and federated learning. In particular, we consider the finite-sum optimization problem 
	$$\min_{x\in \R^d} \sum^{n}_{i = 1}f_i(x)$$ (see \eqref{decent_problem}) interpreted as follows:  $n=|\cV|$ is the number of clients/agents in the network. Communication can only happen between clients connected by an edge. The prevalent paradigm in federated learning, where a single server orchestrates communication in rounds,  arises as a special case of this with $\cG$ being the fully connected network.
	
If $\hat{W}$ is the Laplacian\footnote{In fact, it is enough for $\hat{W}$ to satisfy the less restrictive  Assumption \ref{asm:gossip_matix}.} of graph $\cG$, and let $W = \hat{W}\otimes I_{dn}$. Then problem \eqref{decent_problem} can be rewritten in following equivalent way:
	\begin{equation}\label{eq:P}
		\min_{\sqrt{W}\x = 0}P(\x) = \min_{\x\in\R^{dn}}P(\x) + \psi(\sqrt{W}\x),
	\end{equation}
	where $\x^{\top} = (x^{\top}_1, x^{\top}_2, \dots, x^{\top}_n)$, $P(\x) = \sum^{n}_{i = 1}f_i(x_i) $ and  $\psi(\x) = 0$ iff $\x = 0$, and $\psi(\x) = +\infty$ otherwise.  Problem \eqref{eq:P} is a special case of 
\eqref{comp_problem}.	By dualizing the nonsmooth (but proper, closed and convex) penalty, we get the equivalent saddle point formulation
	\begin{equation}\label{eq:P-saddle}
		\min_{\sqrt{W}\x = 0}P(\x) = \min_{\x\in\R^{dn}}\max_{\y\in \R^{dn}} \left\{P(\x) +\la\y,\sqrt{W}\x\ra - \psi^\star(\y)\right\},
	\end{equation}
	
As we can see, this problem \eqref{eq:P-saddle} is the particular case of the problem \eqref{main_problem}. It means that we can solve it by Algorithm \ref{alg:APDA-Inex}, for example.  Moreover, we do not have to compute the prox of $\psi^{\star}$ due to the fact that $\psi^{\star}(\cdot) \equiv 0$ for every indicator function $\psi(\cdot)$. We thus arrive at the final formulation
	\begin{equation}
		\label{new_problem}
		\min_{\x\in\R^{dn}}\max_{\y\in \R^{dn}} \left\{P(\x) +\la\y,\sqrt{W}\x\ra \right\}.
	\end{equation}

\subsection{Application of Algorithm~\ref{alg:APDA-Inex} to \eqref{new_problem}}	
		Before providing the complexity results related to the application of Algorithm \ref{alg:APDA-Inex} to problem \eqref{new_problem}, note that $L_{xy} = \sqrt{\lambda_{\max}(W)}$ and  $\mu_{xy} = \sqrt{\lambda^+_{\min}(W)}$ and define \begin{equation}\label{eq:chi}\chi \eqdef \frac{\lambda_{\max}(W)}{\lambda^+_{\min}(W)}.\end{equation} According to Theorem \ref{alg:APDA-Inex}, the total number of evaluations of the prox of $\psi^\star$, i.e., the communication complexity, and  the total number of  evaluations of $\nabla P$, i.e., computation complexity, are 
	\begin{equation}
		\label{complexity_alg_1_for_decent_problem}
		\sharp\text{comm} = \widetilde{\mathcal{O}}\left(\max\left\{\sqrt{\kappa\chi},\chi\right\}\right), 	\quad
		\sharp\text{comp} = \widetilde{\mathcal{O}}\left(\max\left\{\kappa^{\frac{2+\alpha}{2\alpha}}\sqrt{\chi},\sqrt[\alpha]{\kappa}\chi\right\}\right),
	\end{equation}
	respectively.	
	For example, in centralized case, when $\cG$ is the complete graph ($\chi = 1$) and  $\cM$ is  chosen to be GD ($\alpha=2$), we obtain the same complexities as ProxSkip~\citep{ProxSkip}:
	\begin{equation}
		\sharp\text{comm} = \widetilde{\mathcal{O}}\left(\sqrt{\kappa}\right),\qquad
		\sharp\text{comp} = \widetilde{\mathcal{O}}\left(\kappa\right).
	\end{equation}
	
However, this can be improved by using a more elaborate subroutine $\cM$. If instead of GD we use either FGD + GD ($\alpha=3$) or FGD + FSFOM  ($\alpha=4$)  in place of $\cM$, the number of communication rounds will be the same as in the case of ProxSkip (or Algorithm \ref{alg:APDA-Inex} used with $\cM$ = GD). However, the total number of gradient computations gets improved to
	$
		\sharp\text{comp} = \widetilde{\mathcal{O}}\left(\kappa^{\frac{5}{6}}\right)
$ in the first case, and to 	$
		\sharp\text{comp} = \widetilde{\mathcal{O}}\left(\kappa^{\frac{3}{4}}\right)
$ in the second case.

\subsection{Improvement on Algorithm~\ref{alg:APDA-Inex} via accelerated gossip}
	\begin{algorithm}[!t]
		\caption{APDA with Inexact Prox and Accelerated Gossip }
		\label{alg:APDA-3}
		\begin{algorithmic}[1]
			\State \textbf{Input}: Initial point $(\x^0, \y^0) \in \R^{d_x}\times\R^{d_y}$, $\Bar{\y}^0 = \y^0$; Step sizes $\eta_x, \eta_y, \beta_y >0$, $\theta \in [0, 1]$; Number of inner iterations $T$; Number of iterations of Accelerated Gossip $N$;
			\For{$k = 0,1, 2, \dots$}
			\State Find $\hat{\x}^k$ as a final point of $T$ iteration of method $\cM$ for following problem:
			\begin{equation}
				\label{auxiliary_problem_acc_gossip}
			\squeeze			\min \limits_{\x\in \R^{d_x}}\left\{P(\x) + \frac{1}{2\eta_x}\left\|\x - \left(\x^k - \eta_x\textsc{AG}(W,\Bar{\y}^k,N)\right)\right\|^2  \right\}
			\end{equation}
			\State $\x^{k+1} = \x^k -\eta_x\left(\nabla P(\hat{\x}^k) + \textsc{AG}(W,\Bar{\y}^k,N)\right)$
			\State $\y^{k+1} = \y^k +\eta_y\left(\textsc{AG}(W,\hat{\x}^k,N) - \beta_y \textsc{AG}(W,\textsc{AG}(W,\y^k,N)+\nabla P(\hat{\x}^k) ,N)\right)$
			\State $\Bar{\y}^{k+1} = \y^{k+1} + \theta\left(\y^{k+1}-\y^k\right)$
			\EndFor 
			\Procedure{AG}{$W, \x, N$} \hfill = Accelerated Gossip
			\State Set $a_0 = 1$, $a_1 = c_2$, $\x_0 = \x$, $\x_1 = c_2\left(I - c_3W\right)\x$
			\For{$i = 1, \dots, N-1$}
			\State $a_{i+1} = 2c_2a_{i} - a_{i-1}$, $\x_{i+1} = 2c_2\left(I-c_2W\right)\x_{i} - \x_{i-1}$
			\EndFor
			\State \Return $\x-\frac{\x_N}{a_N}$
			\EndProcedure
		\end{algorithmic}
	\end{algorithm}
	
 Compared to the ProxSkip computation complexity $\widetilde{\cO}\left(\sqrt{\kappa\chi}\right)$, in the general decentralized case (i.e., $\chi>1$), complexity \eqref{complexity_alg_1_for_decent_problem} of our Algorithm~\ref{alg:APDA-Inex} is worse if $\kappa \leq \chi$. To tackle this problem, we propose to enhance Algorithm~\ref{alg:APDA-Inex} using the accelerated gossip technique \citep{scaman2017optimal}. Based on this approach, we propose one more (and final) method: Algorithm \ref{alg:APDA-3}. For this method, we prove the following result.

	\begin{theorem}
		\label{thm:APDA-3}
		Let Assumptions \ref{as_strongly_convex} and \ref{as_smoothness} hold for function $P$. Then there exist parameters of Algorithm~\ref{alg:APDA-3} such that in order to find  an $\varepsilon$-solution of problem \eqref{main_problem}, the total number of communications and the total number of gradient computations can be bounded by
		\begin{equation}
			\label{complexity_alg_acc_gossip_1}
			\sharp\text{comm}  = \widetilde{\cO}\left(\sqrt{\kappa\chi}\right), \qquad  \sharp\text{comp}  = \widetilde{\cO}\left(\kappa^{\frac{2+\alpha}{2\alpha}}\right),
		\end{equation}
		respectively.
		In particular, if the inner method $\cM$ is
\begin{itemize}		
\item[(i)]  GD, then the number of local steps is equal to $\widetilde{\cO}\left(\kappa^{\nicefrac{1}{2}}\right)$ and the  total number of gradient computations is equal to $\widetilde{\cO}\left(\kappa\right)$; 
\item[(ii)]	 FGD + GD, then the number of local steps is equal to $\widetilde{\cO}\left(\kappa^{\nicefrac{1}{3}}\right)$ and the total number of gradient computations is equal to $\widetilde{\cO}\left(\kappa^{\nicefrac{5}{6}}\right)$;\\
\item[(iii)] FGD + FSFOM, then the number of local steps is equal to $\widetilde{\cO}\left(\kappa^{\nicefrac{1}{4}}\right)$ and the total number of gradient computations is equal to $\widetilde{\cO}\left(\kappa^{\nicefrac{3}{4}}\right)$.
\end{itemize}	
	\end{theorem}

The communication complexities obtained this way are substantially better than those of decentralized ProxSkip; see Table~\ref{tab:results_summary_2} and the commentary in the next subsection. 

\subsection{Summary of complexity results for Algorithms \ref{alg:APDA-Inex}  and \ref{alg:APDA-3} applied to  decentralized optimization}
\label{sec:table_3}

	\begin{table}[!t]
		\vspace{-8px}
		\centering
		\caption{Complexity results for Algorithms \ref{alg:APDA-Inex}  and \ref{alg:APDA-3} applied to solving the decentralized optimization problem \eqref{decent_problem} formulated as the saddle-point problem \eqref{new_problem}. Our results improve upon those of ProxSkip, both in communication complexity (for Algorithm \ref{alg:APDA-3}), and the number of local gradient steps (for both our methods, given a proper choice of $\cM$).}
		\label{tab:results_summary_2}
		{
			\footnotesize
			\begin{threeparttable}
				\begin{tabular}{c c c c c c c c}
					\toprule
					\multirow{2}{*}{Alg.}&\multirow{2}{*}{\begin{tabular}{c} Inner \\ Method\\$\cM$\end{tabular}}&\multirow{2}{*}{\begin{tabular}{c} Deter-\\ministic \\  \#local\\ steps\end{tabular}}&\multicolumn{2}{c}{Centralized case} &\multicolumn{3}{c}{Decentralized case}\\
					\cmidrule(lr){4-5} \cmidrule(lr){6-8}
					&&&\begin{tabular}{c}\#comm.\\ rounds\end{tabular} &\begin{tabular}{c} \#local \\steps\\ per round \end{tabular} &  GA\tnote{\color{red} (1)}  & \begin{tabular}{c}  \#comm.\\ rounds\end{tabular} & \begin{tabular}{c} \#local \\steps\\ per round \end{tabular} \\
					\midrule
					
					D-SGD\tnote{\color{red}(3)} &GD& \cmark&$\widetilde{\mathcal{O}}\left(\kappa\right)$&$\tau$ &\xmark&$\widetilde{\mathcal{O}}\left(\kappa\chi\right)$&$\tau$\\
					\midrule
					ProxSkip\tnote{\color{red}(4)}&GD& \xmark&$\widetilde{\mathcal{O}}\left(\sqrt{\kappa}\right)$&$\mathcal{O}\left(\sqrt{\kappa}\right)$ &\xmark&$\widetilde{\mathcal{O}}\left(\sqrt{\kappa\chi}\right)$\tnote{\color{red}(1)}&$\widetilde{\mathcal{O}}\left(\sqrt{\kappa}\right)$\tnote{\color{red}(2)}\\
					
					\midrule
					\multirow{3}{*}{Alg.~\ref{alg:APDA-Inex}} &\cellcolor{bgcolor}GD&\cellcolor{bgcolor}\cmark& \multirow{3}{*}{$\widetilde{\mathcal{O}}\left(\sqrt{\kappa}\right)$}&\cellcolor{bgcolor}$\widetilde{\mathcal{O}}\left(\sqrt{\kappa}\right)$&\cellcolor{bgcolor} \xmark &\multirow{3}{*}{$\widetilde{\mathcal{O}}\left(\sqrt{\kappa\chi}\vee \chi\right)$\tnote{\color{red} (5)}}
					&\cellcolor{bgcolor}$\widetilde{\mathcal{O}}\left(\sqrt{\kappa}\right)$ \\
					
					\cmidrule(l){2-3} \cmidrule(lr){5-6}\cmidrule(lr){8-8}
					&\cellcolor{bgcolor}FGD + GD &\cellcolor{bgcolor} \cmark & & \cellcolor{bgcolor}$\widetilde{\mathcal{O}}\left(\sqrt[3]{\kappa}\right)$ &\cellcolor{bgcolor} \xmark & &\cellcolor{bgcolor} $\widetilde{\mathcal{O}}\left(\sqrt[3]{\kappa}\right)$ \\
					
					\cmidrule(l){2-3} \cmidrule(lr){5-6} \cmidrule(lr){8-8}
					&\cellcolor{bgcolor} FGD + FSFOM &\cellcolor{bgcolor}\cmark & &\cellcolor{bgcolor} $\widetilde{\mathcal{O}}\left(\sqrt[4]{\kappa}\right)$ &\cellcolor{bgcolor} \xmark & & \cellcolor{bgcolor}$\widetilde{\mathcal{O}}\left(\sqrt[4]{\kappa}\right)$ \\
					
					\midrule
					\multirow{3}{*}{Alg.~\ref{alg:APDA-3}} &\cellcolor{bgcolor}GD&\cellcolor{bgcolor}\cmark& \multirow{3}{*}{$\widetilde{\mathcal{O}}\left(\sqrt{\kappa}\right)$}&\cellcolor{bgcolor}$\widetilde{\mathcal{O}}\left(\sqrt{\kappa}\right)$&\cellcolor{bgcolor} \cmark &\multirow{3}{*}{$\widetilde{\mathcal{O}}\left(\sqrt{\kappa\chi}\right)$}
					&\cellcolor{bgcolor}$\widetilde{\mathcal{O}}\left(\sqrt{\kappa}\right)$ \\
					
					\cmidrule(l){2-3} \cmidrule(lr){5-6}\cmidrule(lr){8-8}
					&\cellcolor{bgcolor}FGD + GD &\cellcolor{bgcolor} \cmark & & \cellcolor{bgcolor}$\widetilde{\mathcal{O}}\left(\sqrt[3]{\kappa}\right)$ &\cellcolor{bgcolor} \cmark & & \cellcolor{bgcolor}$\widetilde{\mathcal{O}}\left(\sqrt[3]{\kappa}\right)$ \\
					
					\cmidrule(l){2-3} \cmidrule(lr){5-6}\cmidrule(lr){8-8}
					& \cellcolor{bgcolor}FGD + FSFOM &\cellcolor{bgcolor}\cmark & &\cellcolor{bgcolor} $\widetilde{\mathcal{O}}\left(\sqrt[4]{\kappa}\right)$ & \cellcolor{bgcolor}\cmark & & \cellcolor{bgcolor}$\widetilde{\mathcal{O}}\left(\sqrt[4]{\kappa}\right)$ \\
					
					\bottomrule
					
				\end{tabular}
				\begin{tablenotes}
					{\tiny
						\item [{\color{red}(1)}] Does not use Accelerated Gossip technique (AG).
						\item [{\color{red}(2)}] Valid only when $\kappa \leq \chi$.
						\item [{\color{red}(3)}] This method was analyzed by \citet{koloskova2020unified}.
						\item [{\color{red}(4)}] This method was proposed and analyzed by \citet{ProxSkip}.
						\item [{\color{red}(5)}] For $a, b \in \R$, we denote $a\vee b = \max\{a, b\}$
					}
				\end{tablenotes}
		\end{threeparttable}}
	\end{table}
	
Recall that in Table~\ref{tab:2}, we already compared ProxSkip and our Algorithm \ref{alg:APDA-Inex} in the centralized case. 	
In Table~\ref{tab:results_summary_2} we add to this a comparison in the decentralized case, and include our Algorithm~\ref{alg:APDA-3} as well. In particular, in Table~\ref{tab:results_summary_2} we compare the complexity results of our methods (Algorithms \ref{alg:APDA-Inex} and \ref{alg:APDA-3}) for solving the decentralized optimization problem \eqref{decent_problem}  to two selected benchmarks: D-SGD~\citep{koloskova2020unified} and ProxSkip~\citep{ProxSkip}. 	
	
	First, observe that ProxSkip has vastly superior communication complexity to D-SGD, both in the centralized case (i.e., for fully-connected network; $\chi=1$), where the improvement is from $\cO(\kappa)$ to $\widetilde{\cO}(\sqrt{\kappa})$, and the decentralized case ($\chi>1$), where the improvement is from $\widetilde{\cO}(\kappa \chi)$ to $\widetilde{\cO}(\sqrt{\kappa \chi})$.

In the decentralized case, both our methods match the $\widetilde{\cO}(\sqrt{\kappa \chi})$ communication complexity of ProxSkip. However, Algorithm \ref{alg:APDA-Inex} does so only when $\sqrt{\kappa \chi}\geq \chi$ (i.e., when $\kappa>\chi$). On the other hand, both our methods have an improved bound on the number of local gradient computations, depending on what subroutine $\cM$ they employ. The improvement is from $\widetilde{\cO}(\sqrt{\kappa})$ to $\widetilde{\mathcal{O}}\left(\sqrt[3]{\kappa}\right)$ (when $\cM=$ FGD + GD) to $\widetilde{\mathcal{O}}\left(\sqrt[4]{\kappa}\right)$ (when $\cM=$ FGD + FSFOM).

\subsection*{Acknowledgments and Disclosure of Funding}
The work of all authors was supported by the KAUST Baseline Research Grant awarded to P.\ Richt\'{a}rik. The work of A.\ Sadiev was supported by the Visiting Student Research Program (VSRP) at KAUST.	The work of A.\ Sadiev was also partially supported by a grant for research centers in the field of artificial intelligence, provided by the Analytical Center for the Government of the Russian Federation in accordance with the subsidy agreement (agreement identifier 000000D730321P5Q0002) and the agreement with the Moscow Institute of Physics and Technology dated November 1, 2021 No. 70-2021-00138.

	\bibliography{refs}

	\clearpage
	\part*{Appendix}
	\appendix


\section{Analysis of the Proximal-Point Method}

In this section we justify the claims we made in Section~\ref{sec:PPM}  about the Proximal-Point Method (Algorithm~\ref{alg_prox_point}). In particular, we prove the complexity result \eqref{eq:PPM-rate-main}. The result is not new of course, but we could not find a source for the proof we include here.

\subsection{The Proximal-Point Method}	

We have described the Proximal-Point Method informally in Section~\ref{sec:PPM}. Here  we state it formally as Algorithm~\ref{alg_prox_point}.
	
	\begin{algorithm}[h]
		\caption{Proximal-Point Method}
		\label{alg_prox_point}
		\begin{algorithmic}[1]
			\State \textbf{Input}: Initial point $(x^0, y^0) \in \R^{d_x}\times\R^{d_y}$; Stepsizes $\eta_x, \eta_y>0$
			\For{$k = 0,1,\dots$}
			\State $x^{k+1} = x^k -\eta_x\left(\partial G(x^{k+1}) + K^{\top}y^{k+1}\right)$
			\State $y^{k+1} = y^k -\eta_y\left(\partial F^{\star}(y^{k+1}) - Kx^{k+1}\right)$
			\EndFor 
		\end{algorithmic}
	\end{algorithm}

\subsection{Bonding the distance of the primal iterates to the primal solution}	
	
In our first lemma, we will provide a bound on $\|x^{k+1}-x^{\star}\|^2$.

	\begin{lemma}
		\label{descent_lemma_by_x_prox_point}
	Let Assumption \ref{as_strongly_convex} hold and choose  any $\eta_x,\eta_y>0$. Then the iterates of the Proximal-Point Method (Algorithm \ref{alg_prox_point})  for all $k\geq 0$ satisfy
				\begin{equation*}
			\left(1+2\mu_x\eta_x\right)\frac{1}{\eta_x}\|x^{k+1}-x^{\star}\|^2 \leq 
			\frac{1}{\eta_x}\|x^k-x^{\star}\|^2 - \frac{1}{\eta_x}\|x^{k+1}-x^{k}\|^2 -2\la K^{\top}y^{k+1}- K^{\top}y^{\star},x^{k+1}-x^{\star}\ra .
		\end{equation*}
	\end{lemma}
	\begin{proof} By writing $x^{k+1}$ as $x^k + (x^{k+1}-x^k)$, and using line 3 of Algorithm \ref{alg_prox_point}, which reads $x^{k+1} = x^k -\eta_x\left(\partial G(x^{k+1}) + K^{\top}y^{k+1}\right)$, we get
		\begin{eqnarray*}
			\frac{1}{\eta_x}\|x^{k+1}-x^{\star}\|^2 &=& \frac{1}{\eta_x}\|x^{k}-x^{\star}\|^2 + \frac{2}{\eta_x}\la x^{k+1}-x^{k}, x^{k}-x^{\star}\ra + \frac{1}{\eta_x}\|x^{k+1}-x^{k}\|^2
			\\
			&=& \frac{1}{\eta_x}\|x^{k}-x^{\star}\|^2 + \frac{2}{\eta_x}\la x^{k+1}-x^{k}, x^{k+1}-x^{\star}\ra -  \frac{1}{\eta_x}\|x^{k+1}-x^{k}\|^2\\
			&=& \frac{1}{\eta_x}\|x^{k}-x^{\star}\|^2 -2\la \partial G(x^{k+1})+ K^{\top}y^{k+1}, x^{k+1}-x^{\star}\ra -  \frac{1}{\eta_x}\|x^{k+1}-x^{k}\|^2 .
		\end{eqnarray*}
		
We now split the inner product into two parts by applying the  optimality condition (see \eqref{opt_conditions})
$${\color{red} 0\in \partial G(x^{\star}) + K^{\top}y^{\star},}$$
obtaining\footnote{Abusing notation, here by $\color{red}\partial G(x^{\star})$ we refer to the subgradient $g_G(x^{\star}) \in \partial G(x^{\star})$ for which $\color{red}0 = g_G(x^{\star}) + K^{\top}y^{\star}$, i.e., $g_G(x^{\star}) =- K^{\top}y^{\star}$.}
		\begin{eqnarray*}
			\frac{1}{\eta_x}\|x^{k+1}-x^{\star}\|^2 &=&
			\frac{1}{\eta_x}\|x^{k}-x^{\star}\|^2 -2\la \partial G(x^{k+1}) {\color{red}- \partial G(x^{\star})} , x^{k+1}-x^{\star}\ra -  \frac{1}{\eta_x}\|x^{k+1}-x^{k}\|^2 \\
			&&-2\la K^{\top}y^{k+1} {\color{red}- K^{\top}y^{\star}}, x^{k+1}-x^{\star}\ra .
		\end{eqnarray*}
Finally, this allows us to replace the first inner product using strong convexity of $G$ as follows:		\begin{eqnarray*}
			\frac{1}{\eta_x}\|x^{k+1}-x^{\star}\|^2 
			&\leq& \frac{1}{\eta_x}\|x^{k}-x^{\star}\|^2 -2\mu_x\|x^{k+1}-x^{\star}\|^2 -  \frac{1}{\eta_x}\|x^{k+1}-x^{k}\|^2 \\
			&&-2\la K^{\top}y^{k+1} - K^{\top}y^{\star}, x^{k+1}-x^{\star}\ra .
		\end{eqnarray*}
	\end{proof}
	
\subsection{Bonding the distance of the dual iterates to the dual solution}	

In our second lemma, we will provide a bound on $\|y^{k+1}-y^{\star}\|^2$. For this, we will rely on an additional assumption (strong convexity of $F^\star$):

	\begin{assumption}
		\label{as_strong_convexity_F}
		The function $F^{\star}: \R^{d_y} \rightarrow \R$ is $\mu_y$-strongly convex (but can be non-differentiable), i.e.,
		\begin{equation}
			\label{convexity_F}
			F^{\star}(y') - F^{\star}(y'') - \la g_{F^\star}(y''), y' - y'' \ra \geq \frac{\mu_y}{2}\|y'-y''\|^2, \qquad \forall y', y'' \in \R^{d_y},
		\end{equation}
		where $g_{F^\star}(y'') \in \partial F^{\star}(y'')$ is any subgradient of $F^\star$ at $y''\in \R^{d_y}$.
	\end{assumption}	
	\begin{lemma}
		\label{descent_lemma_by_y_prox_point}
Let Assumptions \ref{as_convexity} and \ref{as_strong_convexity_F} hold and choose  any $\eta_x,\eta_y>0$. Then the iterates of the Proximal-Point Method (Algorithm \ref{alg_prox_point})  for all $k\geq 0$ satisfy
				\begin{equation*}
			\left(1+2\mu_y\eta_y\right)\frac{1}{\eta_y}\|y^{k+1}-y^{\star}\|^2 \leq 
			\frac{1}{\eta_y}\|y^k-y^{\star}\|^2 - \frac{1}{\eta_y}\|y^{k+1}-y^{k}\|^2 +2\la K^{\top}y^{k+1}- K^{\top}y^{\star},x^{k+1}-x^{\star}\ra .
		\end{equation*}
	\end{lemma}
	\begin{proof} 
By writing $y^{k+1}$ as $y^k + (y^{k+1}-y^k)$, and using line 4 of Algorithm \ref{alg_prox_point}, which reads $y^{k+1} = y^k -\eta_y\left(\partial F^{\star}(y^{k+1}) - Kx^{k+1}\right)$, we get
		\begin{eqnarray*}
			\frac{1}{\eta_y}\|y^{k+1}-y^{\star}\|^2 &=& \frac{1}{\eta_y}\|y^{k}-y^{\star}\|^2 + \frac{2}{\eta_y}\la y^{k+1}-y^{k}, y^{k}-y^{\star}\ra +  \frac{1}{\eta_y}\|y^{k+1}-y^{k}\|^2\\
			&=& \frac{1}{\eta_y}\|y^{k}-y^{\star}\|^2 + \frac{2}{\eta_y}\la y^{k+1}-y^{k}, y^{k+1}-y^{\star}\ra -  \frac{1}{\eta_y}\|y^{k+1}-y^{k}\|^2\\
			&=& \frac{1}{\eta_y}\|y^{k}-y^{\star}\|^2 -2\la \partial F^{\star}(y^{k+1})-Kx^{k+1}, y^{k+1}-y^{\star}\ra -  \frac{1}{\eta_y}\|y^{k+1}-y^{k}\|^2 .
		\end{eqnarray*}
We now split the inner product into two parts by applying the  optimality condition (see \eqref{opt_conditions})
$${\color{red}0 \in \partial F^{\star}(y^{\star}) - Kx^{\star},}$$
obtaining\footnote{Abusing notation, here by $\color{red}\partial F^{\star}(y^{\star})$ we refer to the subgradient $g_{F^\star}(y^{\star}) \in \partial F^{\star}(y^{\star})$ for which $\color{red}0 = g_{F^\star}(y^{\star}) - Kx^{\star}$, i.e., $g_{F^\star}(y^{\star}) = Kx^\star$.}	
		\begin{eqnarray*}
			\frac{1}{\eta_y}\|y^{k+1}-y^{\star}\|^2 &=&
			\frac{1}{\eta_y}\|y^{k}-y^{\star}\|^2 -2\la \partial F^{\star}(y^{k+1}) {\color{red}- \partial F^{\star}(y^{\star}) }, y^{k+1}-y^{\star}\ra -  \frac{1}{\eta_y}\|y^{k+1}-y^{k}\|^2 \\
			&&+2\la Kx^{k+1} - {\color{red}Kx^{\star}},  y^{k+1} -y^{\star} \ra .
		\end{eqnarray*}
Finally, this allows us to replace the first inner product using strong convexity of $F^{\star}$ as follows:		
		\begin{eqnarray*}
			\frac{1}{\eta_y}\|y^{k+1}-y^{\star}\|^2 
			&\leq& \frac{1}{\eta_y}\|y^{k}-y^{\star}\|^2 -2\mu_y\|y^{k+1}-y^{\star}\|^2 -  \frac{1}{\eta_y}\|y^{k+1}-y^{k}\|^2 \\
			&&+2\la K^{\top}y^{k+1} - K^{\top}y^{\star}, x^{k+1}-x^{\star}\ra  .
		\end{eqnarray*}
	\end{proof}
	
\subsection{Complexity of the Proximal-Point Method}	

We now formulate the main result describing the iteration complexity of the Proximal-Point Method. As we shall see, it  follows by combining the above  two lemmas. Note that the theorem postulates {\em arbitrarily fast linear convergence}, with the speed controlled via the primal and dual stepsizes $\eta_x$ and $\eta_y$, respectively.  The larger the stepsizes, the faster the rate becomes. In the limit, as $\eta_x\to +\infty$ and $\eta_y\to +\infty$, the result obtained predicts convergence in a single iteration.

	\begin{theorem}
		\label{th_conv_prox_point_method}
		Let Assumptions \ref{as_strongly_convex} and \ref{as_strong_convexity_F} hold, and choose  any $\eta_x,\eta_y>0$. 		Then the iterates of the Proximal-Point Method (Algorithm \ref{alg_prox_point})  for all $k\geq 0$ satisfy
						\[\Delta^{k+1} \leq \frac{\Delta^k}{\min\left\{1+2\mu_x\eta_x, 1+ 2\mu_y\eta_y\right\}} , \]
				where the Lyapunov function is defined by
$$\Delta^k \eqdef \frac{1}{\eta_x}\|x^k-x^{\star}\|^2 +\frac{1}{\eta_y}\|y^k-y^{\star}\|^2 .$$
This implies the following statement: 
		\begin{equation*}
			k \geq  \left(1 + \frac{1}{\min\left\{2\mu_x\eta_x, 2\mu_y\eta_y\right\}} \right) \log\frac{1}{\varepsilon} \qquad \Rightarrow \qquad  \Delta^k \leq \varepsilon \Delta^0.
		\end{equation*}
	\end{theorem}
	\begin{proof} 
		By adding the inequalities from  Lemma \ref{descent_lemma_by_x_prox_point} and Lemma \ref{descent_lemma_by_y_prox_point}, we obtain 
		\begin{eqnarray*}
			\left(1+2\mu_x\eta_x\right)\frac{1}{\eta_x}\|x^{k+1}-x^{\star}\|^2  &+& \left(1+2\mu_y\eta_y\right)\frac{1}{\eta_y}\|y^{k+1}-y^{\star}\|^2  
			\\
			&\leq&\frac{1}{\eta_x}\|x^k-x^{\star}\|^2 - \frac{1}{\eta_x}\|x^{k+1}-x^{k}\|^2 + \frac{1}{\eta_y}\|y^k-y^{\star}\|^2 - \frac{1}{\eta_y}\|y^{k+1}-y^{k}\|^2  \\
			&&-2\la K^{\top}y^{k+1}- K^{\top}y^{\star},x^{k+1}-x^{\star}\ra +2\la K^{\top}y^{k+1}- K^{\top}y^{\star},x^{k+1}-x^{\star}\ra\\
			&\leq&\frac{1}{\eta_x}\|x^k-x^{\star}\|^2 +\frac{1}{\eta_y}\|y^k-y^{\star}\|^2  .
		\end{eqnarray*}
		
Thus, if we denote $\Delta^k \eqdef \frac{1}{\eta_x}\|x^k-x^{\star}\|^2 +\frac{1}{\eta_y}\|y^k-y^{\star}\|^2 $ and $m\eqdef \min\left\{1+2\mu_x\eta_x, 1+ 2\mu_y\eta_y\right\}$, the above inequality can be written in the following way:
		\begin{equation*}
			\Delta^{k+1}\leq \frac{1}{m}\Delta^k = \left(1 -  \frac{m-1}{m} \right)\Delta^k .
		\end{equation*}
Using standard arguments, the above implies that
	\begin{equation*}
		k \geq   \frac{m}{m-1} \log\frac{1}{\varepsilon} \quad \Rightarrow \quad \Delta^k \leq \varepsilon \Delta^0.
	\end{equation*}
	\end{proof}
	
\clearpage	
\section{Analysis of the Chambolle-Pock Method}

In this section we justify the claims we made in Section~\ref{sec:CP}  about the Chambolle-Pock Method (Algorithm~\ref{alg:Chambolle_Pock}). In particular, we provide formal statements and proofs of the complexity results \eqref{eq:CP-1} and \eqref{eq:CP-2} mentioned in Sections~\ref{sec:CP-1} and~\ref{sec:CP-2}, respectively. 

\subsection{The Chambolle-Pock Method}

We have described the Chambolle-Pock Method informally in Section~\ref{sec:PPM}. Here  we state it formally as Algorithm~\ref{alg:Chambolle_Pock}.
		
		
		\begin{algorithm}[!h]
			\caption{Chambolle-Pock Method \citep{chambolle2011first}}
			\label{alg:Chambolle_Pock}
			\begin{algorithmic}[1]
				\State \textbf{Input}: Initial point $(x^0, y^0) \in \R^{d_x}\times\R^{d_y}$, $\Bar{y}^0 = y^0$; Stepsizes $\eta_x, \eta_y >0$, Extrapolation parameter $\theta \in [0, 1]$
				\For{$k = 0,1,\dots$}			
				\State ${x^{k+1}} = x^k -\eta_x\left(\partial G({x^{k+1}}) + K^{\top}\Bar{y}^k\right)$
				\State $y^{k+1} = y^k -\eta_y\left(\partial F^{\star}(y^{k+1}) - K {x^{k+1}}\right)$
				
				
				\State $\Bar{y}^{k+1} = y^{k+1} + \theta\left(y^{k+1}-y^k\right)$
				\EndFor 
			\end{algorithmic}
		\end{algorithm}
	
\subsection{Bonding the distance of the primal  iterates to the primal solution}


In our first lemma, we will provide a bound on $\|x^{k+1}-x^{\star}\|^2$. The result is identical to Lemma~\ref{descent_lemma_by_x_prox_point} with a single exception: instead of the fresh dual point $y^{k+1}$, we now have $\Bar{y}^k$ on the right-hand side. This lemma applies both to the $\theta=0$ and $\theta>0$ cases.
	
	\begin{lemma}
		\label{descent_lemma_by_x_Champolle_Pock_theta_0}
Let Assumption \ref{as_strongly_convex} hold and choose  any $\eta_x>0$. Then the iterates of the Chambolle-Pock Method (Algorithm~\ref{alg:Chambolle_Pock}) for all $k\geq 0$ satisfy
		\begin{equation*}
			\left(1+2\mu_x\eta_x\right)\frac{1}{\eta_x}\|x^{k+1}-x^{\star}\|^2 \leq 
			\frac{1}{\eta_x}\|x^k-x^{\star}\|^2 - \frac{1}{\eta_x}\|x^{k+1}-x^{k}\|^2 -2\la K^{\top}\Bar{y}^{k}- K^{\top}y^{\star},x^{k+1}-x^{\star}\ra .
		\end{equation*}
	\end{lemma}
	\begin{proof} Identical to the proof of Lemma~\ref{descent_lemma_by_x_prox_point}; one just needs to replace $y^{k+1}$ by $\Bar{y}^k$ everywhere.
	\end{proof}
	
\subsection{Bonding the distance of the dual iterates to  dual solution}	
	
	We do not need a new bound on $\|y^{k+1}-y^{\star}\|^2$ for Algorithm~\ref{alg:Chambolle_Pock} since 
Lemma~\ref{descent_lemma_by_y_prox_point} we proved for the Proximal-Point Method applies here as well.

		\subsection{Formal statement and proof of \eqref{eq:CP-1} (Chambolle-Pock in the $\theta = 0$ case)}\label{sec:CP-1}

We are now ready to state and prove the main complexity result for the Chambolle-Pock Method in the $\theta=0$ case. This is the formal version of the informal complexity result \eqref{eq:CP-1}.

\begin{theorem}[Complexity of the Chambolle-Pock Method in the $\theta=0$ case]\label{thm:CS-theta=0} Let Assumptions~\ref{as_strongly_convex}, \ref{as_strong_convexity_F} hold. Consider the Chambolle-Pock Method (Algorithm \ref{alg:Chambolle_Pock}) with the extrapolation parameter set as $$\theta=0,$$ and the primal and dual stepsizes set as $$\eta_x = \frac{\mu_y}{L^2_{xy}}, \qquad \eta_y = \frac{\mu_x}{L^2_{xy}}.$$  
Then for the Lyapunov function 
$$\Delta^k \eqdef \frac{1}{\eta_x}\|x^k-x^{\star}\|^2 +\frac{1}{\eta_y}\|y^k-y^{\star}\|^2 $$		
and all $k\geq 0$ we have
						\[\Delta^{k+1} \leq \frac{\Delta^k}{\min\left\{1+\mu_x\eta_x, 1+ 2\mu_y\eta_y\right\}}  = \frac{\Delta^k}{1+\frac{\mu_x\mu_y}{L^2_{xy}}} .\]
						 This implies the following statement: 
		\begin{equation*}
			k \geq  \left(1 + \frac{1}{\min\left\{\mu_x\eta_x, 2\mu_y\eta_y\right\}} \right) \log\frac{1}{\varepsilon} = \left(1 + \frac{L^2_{xy}}{\mu_x \mu_y} \right) \log\frac{1}{\varepsilon} \qquad \Rightarrow \qquad  \Delta^k\leq \varepsilon \Delta^0.
		\end{equation*}
		
	\end{theorem}
	
	\begin{proof} 
By adding the inequalities from  Lemma \ref{descent_lemma_by_x_Champolle_Pock_theta_0} (and noting that $\Bar{y}^k=y^k$ in the $\theta=0$ case) and Lemma \ref{descent_lemma_by_y_prox_point}, we obtain 
		\begin{eqnarray}
			\left(1+2\mu_x\eta_x\right)\frac{1}{\eta_x}\|x^{k+1}-x^{\star}\|^2  &+& \left(1+2\mu_y\eta_y\right)\frac{1}{\eta_y}\|y^{k+1}-y^{\star}\|^2  \notag 
			\\
			&\leq&\frac{1}{\eta_x}\|x^k-x^{\star}\|^2 - \frac{1}{\eta_x}\|x^{k+1}-x^{k}\|^2 + \frac{1}{\eta_y}\|y^k-y^{\star}\|^2 - \frac{1}{\eta_y}\|y^{k+1}-y^{k}\|^2 \notag  \\
			&-& 2\la K^{\top}y^{k}- K^{\top}y^{\star},x^{k+1}-x^{\star}\ra +2\la K^{\top}y^{k+1}- K^{\top}y^{\star},x^{k+1}-x^{\star}\ra \label{eq:889900-siguyids}\\
			&\leq&\frac{1}{\eta_x}\|x^k-x^{\star}\|^2 +\frac{1}{\eta_y}\|y^k-y^{\star}\|^2 \notag \\
			&-&  \frac{1}{\eta_y}\|y^{k+1}-y^{k}\|^2 + 2\la K^{\top}y^{k+1}- K^{\top}y^{k},x^{k+1}-x^{\star}\ra ,\label{eq:uiui-09d0-9fuif}
		\end{eqnarray}
		where in the last step we have used the bound $- \frac{1}{\eta_x}\|x^{k+1}-x^{k}\|^2 \leq 0$.

	Using the Cauchy-Schwarz inequality, the definition of $L_{xy}$ as the norm of $K$ (see \eqref{L_xy}), and applying Young's inequality, we can bound the inner product by
		\begin{eqnarray*}
		2\la K^{\top}y^{k+1}- K^{\top}y^{k},x^{k+1}-x^{\star}\ra  &\leq & 2 \|K^{\top}y^{k+1}- K^{\top}y^{k}\| \| x^{k+1}-x^{\star}\|\\
		&\overset{\eqref{L_xy}}{\leq} &	2 L_{xy }\|y^{k+1}- y^{k}\| \| x^{k+1}-x^{\star}\|\\
		&\leq & L_{xy} \left(C \|y^{k+1}- y^{k}\|^2 + \frac{1}{C} \| x^{k+1}-x^{\star}\|^2\right),
			\end{eqnarray*}
for any $C>0$. Plugging this into \eqref{eq:uiui-09d0-9fuif}, and rearranging the inequality so that all terms involving $\|x^{k+1}-x^{\star}\|^2$ appear on the left-hand side, we get			
		\begin{eqnarray*}
			\left(1+2\mu_x\eta_x - \frac{L_{xy}\eta_x}{C}\right)\frac{1}{\eta_x}\|x^{k+1}-x^{\star}\|^2  &+& \left(1+2\mu_y\eta_y\right)\frac{1}{\eta_y}\|y^{k+1}-y^{\star}\|^2  \\
			&\leq&\frac{1}{\eta_x}\|x^k-x^{\star}\|^2 +\frac{1}{\eta_y}\|y^k-y^{\star}\|^2 \\
			&&  - \left(1 -CL_{xy}\eta_y\right)\frac{1}{\eta_y}\|y^{k+1}-y^{k}\|^2 .
	\end{eqnarray*}
	
	Taking $C = \frac{L_{xy}}{\mu_x}$ and $\eta_y = \frac{\mu_x}{L^2_{xy}}$, we obtain the simplified bound	\begin{eqnarray}
			\left(1+\mu_x\eta_x\right)\frac{1}{\eta_x}\|x^{k+1}-x^{\star}\|^2  &+& \left(1+2\mu_y\eta_y\right)\frac{1}{\eta_y}\|y^{k+1}-y^{\star}\|^2  \notag \\
			&\leq&\frac{1}{\eta_x}\|x^k-x^{\star}\|^2 +\frac{1}{\eta_y}\|y^k-y^{\star}\|^2   - \left(1 -\frac{L^2_{xy}\eta_y}{\mu_x}\right)\frac{1}{\eta_y}\|y^{k+1}-y^{k}\|^2 \notag \\
			&\leq& \frac{1}{\eta_x}\|x^k-x^{\star}\|^2 +\frac{1}{\eta_y}\|y^k-y^{\star}\|^2.\label{eq:jjjjd==du8d}
	\end{eqnarray}

If we let $$m\eqdef \min\left\{1+\mu_x\eta_x, 1+ 2\mu_y\eta_y\right\} = \min\left\{1+\frac{\mu_x\mu_y}{L^2_{xy}}, 1+ 2\frac{\mu_x\mu_y}{L^2_{xy}}\right\} = 1+\frac{\mu_x\mu_y}{L^2_{xy}}, $$ then inequality \eqref{eq:jjjjd==du8d} implies		
			\begin{equation*}
			\Delta^{k+1} \leq \frac{1}{m}\Delta^k  = \left(1 -  \frac{m-1}{m} \right)\Delta^k \qquad \forall k\geq 0.
		\end{equation*}
 Using standard arguments, the above implies that
	\begin{equation*}
		k \geq   \frac{m}{m-1} \log\frac{1}{\varepsilon} = \left(1 + \frac{1}{m-1} \right) \log\frac{1}{\varepsilon} = \left(1 + \frac{L^2_{xy}}{\mu_x \mu_y} \right) \log\frac{1}{\varepsilon} \qquad \Rightarrow \qquad \Delta^k \leq \varepsilon \Delta^0.
	\end{equation*}
\end{proof}
	
	\subsection{Formal statement and proof of \eqref{eq:CP-2} (Chambolle-Pock in the $\theta >0$ case)}\label{sec:CP-2}
	
We now study the case when the extrapolation parameter $\theta$ is set to a positive value, and show that this helps to get better rates.

		\begin{lemma}
		\label{inner_product_Chambolle_Pock}		
		Under Assumption \ref{as_strongly_convex}, the iterates of the Chambolle-Pock Method (Algorithm~\ref{alg:Chambolle_Pock}) with $\theta>0$ for any $C >0$ and all $k\geq 1$  satisfy
		\begin{eqnarray}
			2\la K^{\top}y^{k+1} -K^{\top}\Bar{y}^k, x^{k+1} - x^{\star}\ra &\leq&  2\la K^{\top}y^{k+1} -K^{\top}y^k, x^{k+1} - x^{\star}\ra \notag \\
			&& - 2\theta\la K^{\top}y^{k} -K^{\top}y^{k-1}, x^{k} - x^{\star}\ra \notag \\
			&& +\theta L_{xy} C\|x^{k+1} -x^k\|^2  + \frac{\theta L_{xy}}{C}\|y^k-y^{k-1}\|^2.
			\label{eq:099dhshds}
		\end{eqnarray}
	\end{lemma}
	\begin{proof} Using line 5 from Algorithm \ref{alg:Chambolle_Pock}, which reads
	$\Bar{y}^{k} = y^{k} + \theta\left(y^{k}-y^{k-1}\right)$,
	we obtain 
		\begin{eqnarray}
			2\la K^{\top}y^{k+1} -K^{\top}\Bar{y}^k, x^{k+1} - x^{\star}\ra &=&  2\la K^{\top}y^{k+1} -K^{\top}y^k, x^{k+1} - x^{\star}\ra \notag \\
			&& -2\theta\la K^{\top}y^{k} -K^{\top}y^{k-1}, x^{k+1} - x^{\star}\ra \notag \\
			&=&  2\la K^{\top}y^{k+1} -K^{\top}y^k, x^{k+1} - x^{\star}\ra \notag \\
			&& -2\theta\la K^{\top}y^{k} -K^{\top}y^{k-1}, x^{k} - x^{\star}\ra \notag \\
			&& -2\theta\la K^{\top}y^{k} -K^{\top}y^{k-1}, x^{k+1} - x^{k}\ra .\label{eq:hh-92gdd}
		\end{eqnarray}

Using the Cauchy-Schwarz inequality, the definition of $L_{xy}$ as the norm of $K$ (see \eqref{L_xy}), and applying Young's inequality, we can bound the last inner product by
		\begin{eqnarray*}
			-2\theta\la K^{\top}y^{k} -K^{\top}y^{k-1}, x^{k+1} - x^{k}\ra 
			&\leq&   2\theta\| K^{\top}y^{k} -K^{\top}y^{k-1}\|\|x^{k+1} - x^{k}\| \\
			&\leq&  2\theta L_{xy}\|y^{k} -y^{k-1}\|\|x^{k+1} - x^{k}\| \\
			&\leq&  \theta L_{xy} \left(C\|y^{k} -y^{k-1}\|^2 + \frac{1}{C}\|x^{k+1} - x^{k}\|^2\right).
		\end{eqnarray*}
It only remains to plug this inequality into 		\eqref{eq:hh-92gdd}.
		
	\end{proof}
	
We are now ready to state and prove the general theorem.

	\begin{theorem}[Complexity of the Chambolle-Pock Method in the $\theta>0$ case]
		\label{thm:Chambolle_Pock}
		Let Assumptions \ref{as_strongly_convex}, \ref{as_strong_convexity_F}  hold. Consider the Chambolle-Pock Method (Algorithm \ref{alg:Chambolle_Pock}) with the extrapolation parameter set as 
		\begin{equation}
			\label{theta_chambolle_pock}
			\theta = \max\left\{\frac{1}{1+ 2\mu_x\eta_x}, \frac{1}{1+ 2\mu_y\eta_y}\right\},
		\end{equation}		
and primal and dual stepsizes set as		
		\begin{equation}
			\label{step_sizes_chambolle_pock}
			\eta_x = \frac{1}{L_{xy}}\sqrt{\frac{\mu_y}{\mu_x}},  \qquad \eta_y=\frac{1}{L_{xy}}\sqrt{\frac{\mu_x}{\mu_y}}.
		\end{equation}
Then for the Lyapunov function defined for $k\geq 0$ via
		\begin{eqnarray}
			\label{Lyapunov_function_chambolle_pock}
			\Delta^{k+1}  &\eqdef & \left(1+2\mu_x\eta_x\right)\frac{1}{\eta_x}\|x^{k+1}-x^{\star}\|^2+ \left(1+2\mu_y\eta_y\right)\frac{1}{\eta_y}\|y^{k+1} - y^{\star}\|^2\notag\\
			&& +\frac{1}{\eta_y}\|y^{k+1}-y^{k}\|^2 
			- 2\la K^{\top}y^{k+1} - K^{\top}y^{k},x^{k+1}-x^{\star}\ra 
		\end{eqnarray}
 and for $k=0$ via \begin{equation}\Delta^0  \eqdef \left(1+2\mu_x\eta_x\right)\frac{1}{\eta_x}\|x^{0}-x^{\star}\|^2+ \left(1+2\mu_y\eta_y\right)\frac{1}{\eta_y}\|y^{0} - y^{\star}\|^2,\label{eq:ggffDDerw43-833}\end{equation}
we have
		\begin{equation}
			0 \leq \Delta^k  \leq \theta^k \Delta^0  \qquad \forall k\geq 0.
		\end{equation}

	\end{theorem}
	
	\begin{proof}
	
	We will proceed in several steps.
	
	\paragraph{Showing that $\Delta^k\geq 0$ for all $k\geq 0$.}		
		First, let us show that $\Delta^{k}\geq 0$ for every $k$. This is clear for $k=0$ from \eqref{eq:ggffDDerw43-833}. Let us show that $\Delta^{k+1}\geq 0$ for $k\geq 0$. Using Young's inequality and \eqref{L_xy}, we obtain the inequality
		\begin{eqnarray}
			\Delta^{k+1} 
			&\overset{\eqref{Lyapunov_function_chambolle_pock}}{=}& \left(1+2\mu_x\eta_x\right)\frac{1}{\eta_x}\|x^{k+1}-x^{\star}\|^2+ \left(1+2\mu_y\eta_y\right)\frac{1}{\eta_y}\|y^{k+1} - y^{\star}\|^2\notag\\
			&& +\frac{1}{\eta_y}\|y^{k+1}-y^{k}\|^2 - 2\la K^{\top}y^{k+1} - K^{\top}y^{k},x^{k+1}-x^{\star}\ra\notag\\
			&\geq& {\color{red}\left(1+2\mu_x\eta_x\right)\frac{1}{\eta_x}\|x^{k+1}-x^{\star}\|^2} + \left(1+2\mu_y\eta_y\right)\frac{1}{\eta_y}\|y^{k+1} - y^{\star}\|^2\notag\\
			&& {\color{blue}+\frac{1}{\eta_y}\|y^{k+1}-y^{k}\|^2 - L_{xy}B\|y^{k+1} - y^{k}\|^2} {\color{red}-\frac{L_{xy}}{B}\|x^{k+1}-x^{\star}\|^2} \notag\\
&=& 		{\color{red}	\left(1+2\mu_x\eta_x - \frac{L_{xy}\eta_x}{B}\right)\frac{1}{\eta_x}\|x^{k+1}-x^{\star}\|^2} + \left(1+2\mu_y\eta_y\right)\frac{1}{\eta_y}\|y^{k+1} - y^{\star}\|^2\notag\\
&& {\color{blue}+ \left( 1 - L_{xy}B \eta_y\right)\frac{1}{\eta_y}\|y^{k+1}-y^{k}\|^2} \label{eq:88-ttgshs},
		\end{eqnarray}
which holds for all $B>0$.		Selecting $B = \frac{1}{L_{xy}\eta_y}$, the blue term is zeroed out, and using the primal and dual stepsizes \eqref{step_sizes_chambolle_pock}, we get 
		\begin{eqnarray*}
			\Delta^{k+1} &\overset{\eqref{eq:88-ttgshs}}{\geq} & 			\left(1+2\mu_x\eta_x - L_{xy}^2 \eta_x \eta_y\right)\frac{1}{\eta_x}\|x^{k+1}-x^{\star}\|^2 + \left(1+2\mu_y\eta_y\right)\frac{1}{\eta_y}\|y^{k+1} - y^{\star}\|^2\notag\\
			&\overset{\eqref{step_sizes_chambolle_pock}}{=}& \frac{2\sqrt{\mu_x \mu_y}}{L_{xy}} \|x^{k+1}-x^{\star}\|^2+ \left(1+2\mu_y\eta_y\right)\frac{1}{\eta_y}\|y^{k+1} - y^{\star}\|^2  
			\\&\geq& 0 .
		\end{eqnarray*}
		
\paragraph{Establishing technical bounds.}		
Denote
		\begin{equation}
			\mathcal{W}^{k+1} \eqdef \left(1+2\mu_x\eta_x\right)\frac{1}{\eta_x}\|x^{k+1}-x^{\star}\|^2+ \left(1+2\mu_y\eta_y\right)\frac{1}{\eta_y}\|y^{k+1} - y^{\star}\|^2  . \label{eq:00-77-23-8h9fd}
		\end{equation}
		Combining Lemma~\ref{descent_lemma_by_x_Champolle_Pock_theta_0}, which provides a bound on $\|x^{k+1}-x^\star\|^2$, and Lemma~\ref{descent_lemma_by_y_prox_point}, which provides a bound on $\|y^{k+1}-y^\star\|^2$, we obtain 
		\begin{eqnarray}
			\mathcal{W}^{k+1}
			&\leq&\frac{1}{\eta_x}\|x^k-x^{\star}\|^2 - \frac{1}{\eta_x}\|x^{k+1}-x^{k}\|^2  -2\la K^{\top}\Bar{y}^k - K^{\top}y^{\star},x^{k+1}-x^{\star}\ra \notag \\
			&&+\frac{1}{\eta_y}\|y^k-y^{\star}\|^2 - \frac{1}{\eta_y}\|y^{k+1}-y^{k}\|^2 +2\la K^{\top}y^{k+1}- K^{\top}y^{\star},x^{k+1}-x^{\star}\ra \notag \\
			&\leq& \frac{1}{\eta_x}\|x^k-x^{\star}\|^2 - \frac{1}{\eta_x}\|x^{k+1}-x^{k}\|^2  +  \frac{1}{\eta_y}\|y^k-y^{\star}\|^2 - \frac{1}{\eta_y}\|y^{k+1}-y^{k}\|^2\notag \\
			&&+2\la K^{\top}y^{k+1}- K^{\top}\Bar{y}^k,x^{k+1}-x^{\star}\ra .\label{eq:ggGGffFF}
		\end{eqnarray}
This is the same inequality as \eqref{eq:889900-siguyids} with the exception that $y^k$ was replaced by $\Bar{y}^k$.	
	Using Lemma \ref{inner_product_Chambolle_Pock} to bound the inner product in \eqref{eq:ggGGffFF}, we get
		\begin{eqnarray}
			\mathcal{W}^{k+1}
			&\overset{\eqref{eq:ggGGffFF}+\eqref{eq:099dhshds}}{\leq}& \frac{1}{\eta_x}\|x^k-x^{\star}\|^2 - \frac{1}{\eta_x}\|x^{k+1}-x^{k}\|^2  +  \frac{1}{\eta_y}\|y^k-y^{\star}\|^2 -\frac{1}{\eta_y}\|y^{k+1}-y^{k}\|^2\notag\\
			&& + 2\la K^{\top}y^{k+1} -K^{\top}y^k, x^{k+1} - x^{\star}\ra - 2\theta\la K^{\top}y^{k} -K^{\top}y^{k-1}, x^{k} - x^{\star}\ra \notag\\
			&& +\theta L_{xy} C\|x^{k+1} -x^k\|^2  + \frac{\theta L_{xy}}{C}\|y^k-y^{k-1}\|^2\notag\\
			&\leq& \frac{1}{\eta_x}\|x^k-x^{\star}\|^2 - \left(\frac{1}{\eta_x} - \theta L_{xy}C\right)\|x^{k+1}-x^{k}\|^2  +  \frac{1}{\eta_y}\|y^k-y^{\star}\|^2 -\frac{1}{\eta_y}\|y^{k+1}-y^{k}\|^2 \notag \\
			&& + 2\la K^{\top}y^{k+1} -K^{\top}y^k, x^{k+1} - x^{\star}\ra - 2\theta\la K^{\top}y^{k} -K^{\top}y^{k-1}, x^{k} - x^{\star}\ra \notag \\ 
			&&+ \frac{\theta L_{xy}}{C}\|y^k-y^{k-1}\|^2 \notag\\			
			&=& \frac{1}{\eta_x}\|x^k-x^{\star}\|^2  +  \frac{1}{\eta_y}\|y^k-y^{\star}\|^2 -\frac{1}{\eta_y}\|y^{k+1}-y^{k}\|^2 \notag \\ 
			&& + 2\la K^{\top}y^{k+1} -K^{\top}y^k, x^{k+1} - x^{\star}\ra - 2\theta\la K^{\top}y^{k} -K^{\top}y^{k-1}, x^{k} - x^{\star}\ra \notag \\
			&&+ \theta^2 L_{xy}^2 \eta_x \eta_y \frac{1}{\eta_y}\|y^k-y^{k-1}\|^2 			\label{eq:iuiudf-f9d8y9fd-fd8h9fd=fdf},
		\end{eqnarray}
		where in the last step we have made the choice $C \eqdef (\theta\eta_xL_{xy})^{-1}$.
				
\paragraph{Showing that $\Delta^{k+1}\leq \theta \Delta^k$ for $k\geq 1$.}				
				By combining \eqref{Lyapunov_function_chambolle_pock} and \eqref{eq:00-77-23-8h9fd}, and applying inequality \eqref{eq:iuiudf-f9d8y9fd-fd8h9fd=fdf}, for $k\geq 1$ we get
		\begin{eqnarray}
			\Delta^{k+1} &\overset{\eqref{Lyapunov_function_chambolle_pock}+\eqref{eq:00-77-23-8h9fd}}{=}& \cW^{k+1} +\frac{1}{\eta_y}\|y^{k+1}-y^{k}\|^2 
			- 2\la K^{\top}y^{k+1} - K^{\top}y^{k},x^{k+1}-x^{\star}\ra \notag\\
			&\overset{\eqref{eq:iuiudf-f9d8y9fd-fd8h9fd=fdf}}{\leq}& 
 \frac{1}{\eta_x}\|x^k-x^{\star}\|^2  +  \frac{1}{\eta_y}\|y^k-y^{\star}\|^2 {\color{blue}-\frac{1}{\eta_y}\|y^{k+1}-y^{k}\|^2 }\notag \\ 
			&& {\color{red}+ 2\la K^{\top}y^{k+1} -K^{\top}y^k, x^{k+1} - x^{\star}\ra} - 2\theta\la K^{\top}y^{k} -K^{\top}y^{k-1}, x^{k} - x^{\star}\ra \notag \\
			&& +\theta^2 L_{xy}^2 \eta_x \eta_y \frac{1}{\eta_y}\|y^k-y^{k-1}\|^2 	\notag\\
			&&	{\color{blue}+\frac{1}{\eta_y}\|y^{k+1}-y^{k}\|^2} 
			{\color{red}- 2\la K^{\top}y^{k+1} - K^{\top}y^{k},x^{k+1}-x^{\star}\ra} \notag\\		
			&\overset{(A)}{\leq} & \frac{1}{\eta_x}\|x^k-x^{\star}\|^2   +  \frac{1}{\eta_y}\|y^k-y^{\star}\|^2    - 2\theta\la K^{\top}y^{k} -K^{\top}y^{k-1}, x^{k} - x^{\star}\ra  + \theta ^2 L^2_{xy}\eta_x\eta_y\frac{1}{\eta_y}\|y^k-y^{k-1}\|^2 \notag\\
			&\overset{(B)}{\leq} & \frac{1}{\eta_x}\|x^k-x^{\star}\|^2   +  \frac{1}{\eta_y}\|y^k-y^{\star}\|^2    - 2\theta\la K^{\top}y^{k} -K^{\top}y^{k-1}, x^{k} - x^{\star}\ra  + \theta ^2 \frac{1}{\eta_y}\|y^k-y^{k-1}\|^2 \notag\\		
			&\overset{(C) }{\leq}& \theta \left((1+ 2\mu_x\eta_x)\frac{1}{\eta_x}\|x^k-x^{\star}\|^2   +  (1+ 2\mu_y\eta_y)\frac{1}{\eta_y}\|y^k-y^{\star}\|^2 \right) \notag\\
			&& + \theta \left(\frac{1}{\eta_y}\|y^k-y^{k-1}\|^2- 2\la K^{\top}y^{k} -K^{\top}y^{k-1}, x^{k} - x^{\star}\ra \right) \notag\\
			&\overset{\eqref{Lyapunov_function_chambolle_pock}}{=} & \theta \Delta^k, \label{eq:ggGGhhGGdd}
		\end{eqnarray}
where in step (A) we annihilated the red and blue terms as their sum is zero, in step (B) we used the fact that $L^2_{xy}\eta_x\eta_y \leq 1$, which follows from the stepsize choice, and in (C) we used the inequalities $1 \leq \theta(1+ 2\mu_x\eta_x)$, $1 \leq \theta(1+ 2\mu_y\eta_y)$, and $\theta\leq 1$, which follow from \eqref{theta_chambolle_pock}.

\paragraph{Showing that $\Delta^{k+1}\leq \theta \Delta^k$ for $k=0$.}
We start with inequality \eqref{eq:ggGGffFF} for $k=0$:
\begin{eqnarray}
	\mathcal{W}^{1}	
	&\overset{\eqref{eq:ggGGffFF}}{\leq}& \frac{1}{\eta_x}\|x^0-x^{\star}\|^2 - \frac{1}{\eta_x}\|x^{1}-x^{0}\|^2  +  \frac{1}{\eta_y}\|y^0-y^{\star}\|^2 - \frac{1}{\eta_y}\|y^{1}-y^{0}\|^2 \notag \\
	&&+2\la K^{\top}y^{1}- K^{\top}y^0,x^{1}-x^{\star}\ra . \label{eq:llKK-00987}
\end{eqnarray}

According to \eqref{Lyapunov_function_chambolle_pock} and \eqref{theta_chambolle_pock}, we have 
\begin{eqnarray*}
	\Delta^{1} &\overset{\eqref{Lyapunov_function_chambolle_pock}+\eqref{eq:00-77-23-8h9fd}}{=}& \cW^{1} +\frac{1}{\eta_y}\|y^{1}-y^{0}\|^2 
			- 2\la K^{\top}y^{1} - K^{\top}y^{0},x^{1}-x^{\star}\ra \notag\\
	&\overset{\eqref{eq:llKK-00987}}{\leq}& \frac{1}{\eta_x}\|x^0-x^{\star}\|^2 - \frac{1}{\eta_x}\|x^{1}-x^{0}\|^2  +  \frac{1}{\eta_y}\|y^0-y^{\star}\|^2 \\ 
	&\leq & \frac{1}{\eta_x}\|x^0-x^{\star}\|^2  +  \frac{1}{\eta_y}\|y^0-y^{\star}\|^2 \\ 
	&\leq& \theta\left((1+2\mu_x\eta_x)\frac{1}{\eta_x}\|x^0-x^{\star}\|^2   + (1+2\mu_y\eta_y) \frac{1}{\eta_y}\|y^0-y^{\star}\|^2\right)\\
	&= & \theta \Delta^0
\end{eqnarray*}
where in the  last inequality we used the inequalities $1 \leq \theta(1+ 2\mu_x\eta_x)$ and $1 \leq \theta(1+ 2\mu_y\eta_y)$, which follow from \eqref{theta_chambolle_pock}.

	\end{proof}
	
	The informal result \eqref{eq:CP-2} mentioned in Section~\ref{sec:CP-2} is a simple corollary of the above theorem. Indeed, using the definition of $\theta$, we can obtain:
	\begin{eqnarray*}
		k \geq \frac{1}{1-\theta}\log\frac{1}{\varepsilon} \qquad \Rightarrow \qquad \Delta^k \leq \varepsilon \Delta^0.
	\end{eqnarray*}
	
It remains to remark that				
	\begin{eqnarray*}		
\frac{1}{1-\theta}\log\frac{1}{\varepsilon}		&=&  \mathcal{O}\left(\left(1+\max\left\{\frac{1}{2\mu_x\eta_x},\frac{1}{2\mu_{y}\eta_y}\right\}\right)\log\frac{1}{\varepsilon}\right)\\
		&=&\mathcal{O}\left(\left(1+\frac{L_{xy}}{\sqrt{\mu_x\mu_y}}\right)\log\frac{1}{\varepsilon}\right) ,
	\end{eqnarray*}
	which is the result from \eqref{eq:CP-2}.

\clearpage		
\section{Analysis of the Accelerated Primal-Dual Algorithm (APDA; Algorithm~\ref{alg:APDA})}
	
In this section we perform convergence analysis for our new method APDA (Algorithm~\ref{alg:APDA}). We start by establishing three lemmas, followed by the proof of the main theorem.

\subsection{Three lemmas}	

 The first result is a variant of Lemma~\ref{descent_lemma_by_x_Champolle_Pock_theta_0}, offering a bound on the distance of the primal iterates from the primal optimal solution. Compared to  Lemma~\ref{descent_lemma_by_x_Champolle_Pock_theta_0}, this result is strengthened by the additional assumption of $L_x$-smoothness of $G$ (Assumption~\ref{as_smoothness}).

	\begin{lemma}
		\label{descent_lemma_by_x_alg_2}
Let Assumptions \ref{as_strongly_convex} and \ref{as_smoothness} hold and choose any $\eta_x,\eta_y>0$. Then the  iterates of APDA (Algorithm~\ref{alg:APDA}) for all $k\geq 0$ satisfy
		\begin{eqnarray*}
			\left(1+\mu_x\eta_x\right)\frac{1}{\eta_x}\|x^{k+1}-x^{\star}\|^2&\leq&
			\frac{1}{\eta_x}\|x^k-x^{\star}\|^2 - \frac{1}{\eta_x}\|x^{k+1}-x^{k}\|^2 -2\la K^{\top}\Bar{y}^k - K^{\top}y^{\star}, x^{k+1}-x^{\star}\ra  \\
			&& \quad   - \frac{1}{L_x}\|\nabla G(x^{k+1}) - \nabla G(x^{\star})\|^2.
		\end{eqnarray*}		
\end{lemma}		

\begin{proof}		
By writing $x^{k+1}$ as $x^k + (x^{k+1}-x^k)$, and using line 3 of Algorithm \ref{alg:APDA}, which reads $x^{k+1} = x^k -\eta_x\left(\nabla G(x^{k+1}) + K^{\top}\Bar{y}^{k}\right)$, we get
		\begin{eqnarray*}
			\frac{1}{\eta_x}\|x^{k+1}-x^{\star}\|^2 &=& \frac{1}{\eta_x}\|x^{k}-x^{\star}\|^2 + \frac{2}{\eta_x}\la x^{k+1}-x^{k}, x^{k+1}-x^{\star}\ra -  \frac{1}{\eta_x}\|x^{k+1}-x^{k}\|^2\\
			&=& \frac{1}{\eta_x}\|x^{k}-x^{\star}\|^2 -2\la \nabla G(x^{k+1})+ K^{\top}\Bar{y}^{k}, x^{k+1}-x^{\star}\ra -  \frac{1}{\eta_x}\|x^{k+1}-x^{k}\|^2 .
		\end{eqnarray*}

We now split the inner product into two parts by applying the  optimality condition (see \eqref{opt_conditions})
$${\color{red}\nabla G(x^{\star}) + K^{\top}y^{\star} = 0},$$
obtaining		\begin{eqnarray}
			\frac{1}{\eta_x}\|x^{k+1}-x^{\star}\|^2 &=& \frac{1}{\eta_x}\|x^{k}-x^{\star}\|^2 -2\la \nabla G(x^{k+1}) {\color{red}- \nabla G(x^{\star})} , x^{k+1}-x^{\star}\ra -  \frac{1}{\eta_x}\|x^{k+1}-x^{k}\|^2 \notag \\
			&&\quad -2\la K^{\top}\Bar{y}^{k} {\color{red}- K^{\top}y^{\star}}, x^{k+1}-x^{\star}\ra . \label{eq:yyUUttRR09098d}
		\end{eqnarray}
Since $G$ is $\mu_x$-strongly-convex and $L_x$-smooth, we can lower-bound \citep{NesterovBook}  the inner product appearing in \eqref{eq:yyUUttRR09098d} as follows:
\begin{equation}\label{eq:gdig9f7d97gf))88d} 2\la \nabla G(x^{k+1}) - \nabla G(x^{\star}) , x^{k+1}-x^{\star}\ra \geq \mu_x \| x^{k+1}-x^{\star} \|^2 + \frac{1}{L_x} \|\nabla G(x^{k+1}) - \nabla G(x^{\star}) \|^2.\end{equation} Finally, by plugging \eqref{eq:gdig9f7d97gf))88d} into \eqref{eq:yyUUttRR09098d}, we get
		\begin{eqnarray*}
			\frac{1}{\eta_x}\|x^{k+1}-x^{\star}\|^2 
			&\leq& \frac{1}{\eta_x}\|x^{k}-x^{\star}\|^2 -\mu_x\|x^{k+1}-x^{\star}\|^2 -  \frac{1}{\eta_x}\|x^{k+1}-x^{k}\|^2 \\
			&&-2\la K^{\top}\Bar{y}^k - K^{\top}y^{\star}, x^{k+1}-x^{\star}\ra - \frac{1}{L_x}\|\nabla G(x^{k+1}) - \nabla G(x^{\star})\|^2.
		\end{eqnarray*}		
\end{proof}

%

The second result  is a technical lemma borrowed from \citep{kovalev2021accelerated} for lower-bounding the term $\|K^{\top}y - K^{\top}y^{\star}\|$.

	\begin{lemma}[See \cite{kovalev2021accelerated}] 
		\label{lem:convenience}
		Let Assumption \ref{as_L_xy} hold and let $(x^{\star}, y^{\star})$ be a solution of  \eqref{main_problem}. Then 		
		\begin{equation*}
			\|K^{\top}y - K^{\top}y^{\star}\| \geq \mu_{xy}\|y - y^{\star}\|, \qquad \forall y \in \R^{d_y}.
		\end{equation*}
	\end{lemma}

The third result offers a bound on the distance between the dual iterates and the dual optimal solution.	When analyzing the Proximal Point method and the Chambolle-Pock method, in this step  we relied on Lemma~\ref{descent_lemma_by_y_prox_point}, in the proof of which we required  $F^\star$ to be $\mu_x$-strongly convex. However, for APDA we explicitly wish to avoid using this assumption. As we shall see, in order to obtain linear convergence, it is enough will invoke Assumption \ref{as_L_xy}. The next lemma is an analogue of Lemma~\ref{descent_lemma_by_y_prox_point} without the need to assume $\mu_x$-strong-convexity of $F^\star$.
	
	\begin{lemma}
		\label{descent_lemma_by_y_alg_2}
		Let Assumptions \ref{as_smoothness}, \ref{as_convexity}, and \ref{as_L_xy} be satisfied, and choose
		\begin{equation}
			\label{beta_y_2}
			\beta_y \leq \frac{1}{2L^2_{xy}\eta_y}.
		\end{equation}
Then the iterates of APDA (Algorithm~\ref{alg:APDA}) for all $k\geq 0$ satisfy
		\begin{align*}
			\frac{1}{\eta_y}\|y^{k+1} - y^{\star}\|^2  
			&\leq (1-\mu^2_{xy}\beta_y\eta_y)\frac{1}{\eta_y}\|y^k-y^{\star}\|^2 - \frac{1}{2\eta_y}\|y^{k+1}-y^{k}\|^2 + 2\la K^{\top}y^{k+1} - K^{\top}y^{\star},x^{k+1}-x^{\star}\ra  \\
			&\quad + \beta_y\|\nabla G(x^{k+1}) - \nabla G(x^{\star})\|^2 .
		\end{align*}
	\end{lemma}
	\begin{proof} By writing $y^{k+1}$ as $y^k + (y^{k+1}-y^k)$, and using line 4  from Algorithm~\ref{alg:APDA}, which reads $$y^{k+1} = y^k -\eta_y\left(\partial F^{\star}(y^{k+1}) - K x^{k+1} \right) - \eta_y\beta_y K\left(K^{\top}y^k+\nabla G(x^{k+1})\right)
,$$ we get
		\begin{eqnarray}
			\frac{1}{\eta_y}\|y^{k+1}-y^{\star}\|^2 &=& \frac{1}{\eta_y}\|y^{k}-y^{\star}\|^2 + \frac{2}{\eta_y}\la y^{k+1}-y^k,y^{k+1}-y^{\star}\ra - \frac{1}{\eta_y}\|y^{k+1}-y^{k}\|^2 \notag \\
			&=& \underbrace{\frac{1}{\eta_y}\|y^{k}-y^{\star}\|^2 - 2\la \partial F^{\star}(y^{k+1}) - Kx^{k+1},y^{k+1}-y^{\star}\ra - \frac{1}{\eta_y}\|y^{k+1}-y^{k}\|^2}_{A^k} \notag \\
			&&\qquad \underbrace{ - 2 \beta_y \la K^{\top}y^k+\nabla G(x^{k+1}),K^{\top}y^{k+1}-K^{\top}y^{\star}\ra }_{B^k}.\label{eq:A^k_and_B^k}
		\end{eqnarray}

We split the inner product appearing in $A^k$ into two parts by applying the optimality condition (see \eqref{opt_conditions})
\[{\color{red} 0 \in \partial F^\star(y^\star) - K x^\star},\]
obtaining
\begin{eqnarray}
A^k &=& \frac{1}{\eta_y}\|y^{k}-y^{\star}\|^2 - 2\la \partial F^{\star}(y^{k+1}) {\color{red}- \partial F^{\star}(y^{\star})},y^{k+1}-y^{\star}\ra- \frac{1}{\eta_y}\|y^{k+1}-y^{k}\|^2 \notag \\
			&& \quad + 2\la  Kx^{k+1} - {\color{red}K x^{\star}},y^{k+1}-y^{\star}\ra . \label{eq:A^k}
\end{eqnarray}			
Applying  the parallelogram identity\footnote{Here we refer to the identity $2\la a+b, c-d\ra = -\|b+d\|^2 + \|a-d\|^2 - \|a-c\|^2 + \|b+c\|^2$ which holds for all $a,b,c,d \in \R^{d_x}$.}
to $B^k$, and 
using the optimality condition (see \eqref{opt_conditions}),
\begin{equation}\label{eq:opt_g9798fd98fd}{\color{blue} \nabla G(x^\star) + K^\top y^\star = 0},\end{equation}
we can write
		\begin{eqnarray}
			B^k &=& \beta_y \|\nabla G(x^{k+1}) + {\color{blue}K^{\top}y^{\star}}\|^2  - \beta_y \|K^{\top}y^k - K^{\top}y^{\star}\|^2\notag \\ 
			&& \quad + \beta_y \|K^{\top}y^k -K^{\top}y^{k+1}\|^2 - \beta_y \|\nabla G(x^{k+1}) + K^{\top}y^{k+1}\|^2 \notag \\
&\overset{\eqref{eq:opt_g9798fd98fd}}{=}&	\beta_y \|\nabla G(x^{k+1})  {\color{blue}- \nabla G(x^\star)}\|^2  - \beta_y \|K^{\top}y^k - K^{\top}y^{\star}\|^2\notag \\ 
			&& \quad + \beta_y \|K^{\top}y^k -K^{\top}y^{k+1}\|^2 - \beta_y \|\nabla G(x^{k+1}) + K^{\top}y^{k+1}\|^2	.	
			\label{eq:B^k}
		\end{eqnarray}
Plugging \eqref{eq:A^k} and \eqref{eq:B^k} back into \eqref{eq:A^k_and_B^k}, and using convexity of $F^\star$, we get		
\begin{eqnarray}
			\frac{1}{\eta_y}\|y^{k+1}-y^{\star}\|^2 
			&\overset{\eqref{eq:A^k_and_B^k}}{=}& A^k + B^k \notag \\
&\overset{\eqref{eq:A^k}+\eqref{eq:B^k}}{=}&			\frac{1}{\eta_y}\|y^{k}-y^{\star}\|^2 \underbrace{- 2\la \partial F^{\star}(y^{k+1}) {\color{red}- \partial F^{\star}(y^{\star})},y^{k+1}-y^{\star}\ra}_{\leq 0}- \frac{1}{\eta_y}\|y^{k+1}-y^{k}\|^2 \notag \\
			&& + 2\la  Kx^{k+1} - {\color{red}K x^{\star}},y^{k+1}-y^{\star}\ra \notag \\
			&& + \beta_y \|\nabla G(x^{k+1}) {\color{blue}- \nabla G(x^{\star})}\|^2 - \beta_y \|K^{\top}y^k - K^{\top}y^{\star}\|^2 \notag \\
			&&  + \beta_y \|K^{\top}y^k -K^{\top}y^{k+1}\|^2 \underbrace{- \beta_y \|\nabla G(x^{k+1}) + K^{\top}y^{k+1}\|^2}_{\leq 0} \notag \\
			&\leq &
		\frac{1}{\eta_y}\|y^{k}-y^{\star}\|^2- \frac{1}{\eta_y}\|y^{k+1}-y^{k}\|^2\notag \\
			&& + 2\la  Kx^{k+1} - {\color{red}K x^{\star}},y^{k+1}-y^{\star}\ra \notag \\
			&& + \beta_y \|\nabla G(x^{k+1}) {\color{blue}- \nabla G(x^{\star})}\|^2 - \beta_y \|K^{\top}y^k - K^{\top}y^{\star}\|^2 \notag \\
			&&  + \beta_y \|K^{\top}y^k -K^{\top}y^{k+1}\|^2 .	\label{eq:hhHHggvtyd555d}
		\end{eqnarray}
		
It now only remains to plug  the bounds 
$\|K^{\top}y^k - K^{\top}y^{\star}\|^2 \geq \mu_{xy}^2\|y^k-y^{\star}\|^2$ (see Lemma~\ref{lem:convenience}) and $\|K^{\top}y^k -K^{\top}y^{k+1}\|^2 \leq L_{xy}^2 \|y^k-y^{k+1}\|^2$
into \eqref{eq:hhHHggvtyd555d}, and apply the restriction \eqref{beta_y_2} on $\beta_y$.


			\end{proof}
	
\subsection{Main result}	

We are now ready to state the formal version of Theorem~\ref{thm:APDA-informal}.
	
	\begin{theorem}[Covergence of APDA; formal]
		\label{thm:APDA}
		Let Assumptions \ref{as_strongly_convex}, \ref{as_smoothness}, \ref{as_convexity}, and \ref{as_L_xy}  hold and choose the various parameters of the method as follows:
		\begin{equation}
			\label{def_beta_y_2}
			\beta_y = \min\left\{\frac{1}{L_x}, \frac{1}{2L^2_{xy}\eta_y}\right\},
		\end{equation}
		\begin{equation}
			\label{eta_x_2}
			\eta_x = \frac{1}{2\sqrt{L_x\mu_x}}\frac{\mu_{xy}}{L_{xy}},
		\end{equation}
		\begin{equation}
			\label{eta_y_2}
			   \eta_y = \frac{\sqrt{L_x\mu_x}}{L_{xy}\mu_{xy}},
		\end{equation}
		\begin{equation}
			\label{theta_2}
			\theta = \max\left\{\frac{1}{1+ \mu_x\eta_x}, 1 - \mu^2_{xy}\beta_y\eta_y\right\}.
		\end{equation}
		Then for the Lyapunov function  defined for $k\geq 0$ via
		\begin{eqnarray}
			\label{Lyapunov_function_alg_2}
			\Delta^{k+1} &\eqdef & \left(1+\mu_x\eta_x\right)\frac{1}{\eta_x}\|x^{k+1}-x^{\star}\|^2+ \frac{1}{\eta_y}\|y^{k+1} - y^{\star}\|^2\notag\\
			&& +\frac{1}{2\eta_y}\|y^{k+1}-y^{k}\|^2 - 2\la K^{\top}y^{k+1} - K^{\top}y^{k},x^{k+1}-x^{\star}\ra
		\end{eqnarray}
		and for $k=0$ via
		$$\Delta^0 \eqdef \left(1+\mu_x\eta_x\right)\frac{1}{\eta_x}\|x^{0}-x^{\star}\|^2+ \frac{1}{\eta_y}\|y^{0} - y^{\star}\|^2, $$
we have
		\begin{equation}\label{eq:jdd8t*gduI*_T*}
			0\leq \mu_x\|x^{k+1}-x^{\star}\|^2+ \frac{1}{\eta_y}\|y^{k+1} - y^{\star}\|^2   \leq \Delta^{k}\leq \theta^k\Delta^0 \qquad \forall k\geq 0.
		\end{equation}
	\end{theorem}
	
	\begin{proof}
We will proceed in several steps.
\paragraph{Showing that $\Delta^k\geq 0$ for all $k\geq 0$.}		
		First, we will show that $\Delta^{k}\geq 0$ for every $k$. This is clear for $k=0$. Let us show that $\Delta^{k+1}\geq 0$ for every $k\geq 0$. Using Young's inequality with any parameter $B>0$, and the definition of $L_{xy}$ from \eqref{L_xy},  we obtain	
			\begin{eqnarray}
			\Delta^{k+1} 
			&\overset{\eqref{Lyapunov_function_alg_2}}{=}& \left(1+\mu_x\eta_x\right)\frac{1}{\eta_x}\|x^{k+1}-x^{\star}\|^2+ \frac{1}{\eta_y}\|y^{k+1} - y^{\star}\|^2 +\frac{1}{2\eta_y}\|y^{k+1}-y^{k}\|^2\notag\\
			&&- 2\la K^{\top}y^{k+1} - K^{\top}y^{k},x^{k+1}-x^{\star}\ra\notag\\
&\geq &  \left(1+\mu_x\eta_x\right)\frac{1}{\eta_x}\|x^{k+1}-x^{\star}\|^2+ \frac{1}{\eta_y}\|y^{k+1} - y^{\star}\|^2 +\frac{1}{2\eta_y}\|y^{k+1}-y^{k}\|^2 \notag\\
			&& - L_{xy}B\|y^{k+1}-y^{k}\|^2  -\frac{L_{xy}}{B}\|x^{k+1}-x^{\star}\|^2 \notag\\		
			& =& \left(1+\mu_x\eta_x\right)\frac{1}{\eta_x}\|x^{k+1}-x^{\star}\|^2+ \frac{1}{\eta_y}\|y^{k+1} - y^{\star}\|^2 + \left(\frac{1}{2} - L_{xy}B\eta_y\right)\frac{1}{\eta_y}\|y^{k+1}-y^{k}\|^2 \notag\\
			&& -\frac{L_{xy}}{B}\|x^{k+1}-x^{\star}\|^2.\label{eq:h0d9h0fd-09y886FF}
		\end{eqnarray}
Using \eqref{eq:h0d9h0fd-09y886FF} with 	 $B = \frac{1}{2L_{xy}\eta_y}$, and noticing that  $2L_{xy}^2\eta_x\eta_y =1$ (this follows from \eqref{eta_x_2} and \eqref{eta_y_2}), we get		\begin{eqnarray*}
			\Delta^{k+1} 
			&\geq& \left(1+\mu_x\eta_x\right)\frac{1}{\eta_x}\|x^{k+1}-x^{\star}\|^2+ \frac{1}{\eta_y}\|y^{k+1} - y^{\star}\|^2  -\frac{2L_{xy}^2\eta_x\eta_y}{\eta_x}\|x^{k+1}-x^{\star}\|^2\\
			&=& \mu_x\|x^{k+1}-x^{\star}\|^2+ \frac{1}{\eta_y}\|y^{k+1} - y^{\star}\|^2  
			\\&\geq& 0.
		\end{eqnarray*}

\paragraph{Establishing technical bounds.}	
Denote 		\begin{equation}
			\mathcal{W}^{k+1} \eqdef \left(1+\mu_x\eta_x\right)\frac{1}{\eta_x}\|x^{k+1}-x^{\star}\|^2+ \frac{1}{\eta_y}\|y^{k+1} - y^{\star}\|^2 .\label{eq:W^k-APDA-999}
		\end{equation}
By adding the inequalities from Lemmas \ref{descent_lemma_by_x_alg_2} and \ref{descent_lemma_by_y_alg_2}, we obtain 
		\begin{eqnarray}
			\mathcal{W}^{k+1}
			&\leq& \frac{1}{\eta_x}\|x^k-x^{\star}\|^2 - \frac{1}{\eta_x}\|x^{k+1}-x^{k}\|^2 - \frac{1}{L_x}\|\nabla G(x^{k+1}) - \nabla G(x^{\star})\|^2 \notag \\
			&&+ (1-\mu^2_{xy}\beta_y\eta_y)\frac{1}{\eta_y}\|y^k-y^{\star}\|^2 - \frac{1}{2\eta_y}\|y^{k+1}-y^{k}\|^2  + \beta_y\|\nabla G(x^{k+1}) - \nabla G(x^{\star})\|^2\notag \\
			&&+ 2\la K^{\top}y^{k+1} - K^{\top}y^{\star},x^{k+1}-x^{\star}\ra -2\la K^{\top}\Bar{y}^k - K^{\top}y^{\star},x^{k+1}-x^{\star}\ra\notag \\
			&\leq& \frac{1}{\eta_x}\|x^k-x^{\star}\|^2 - \frac{1}{\eta_x}\|x^{k+1}-x^{k}\|^2 + (1-\mu^2_{xy}\beta_y\eta_y)\frac{1}{\eta_y}\|y^k-y^{\star}\|^2 - \frac{1}{2\eta_y}\|y^{k+1}-y^{k}\|^2 \notag \\
			&&+ 2\la K^{\top}y^{k+1} - K^{\top}\Bar{y}^k,x^{k+1}-x^{\star}\ra , \label{eq:pp-00o-88d}
		\end{eqnarray}
where in the last step we used the bound 	$$\beta_y\|\nabla G(x^{k+1}) - \nabla G(x^{\star})\|^2- \frac{1}{L_x}\|\nabla G(x^{k+1}) - \nabla G(x^{\star})\|^2\leq 0,$$ which follows from the restriction $\beta_y \leq \frac{1}{L_x}$; see \eqref{def_beta_y_2}.	Using line 5 from Algorithm \ref{alg:APDA}, which says
\begin{equation}\label{eq:biug9f7dgidf-989fd}\bar{y}^{k+1} = y^{k+1} + \theta (y^{k+1}-y^k),\end{equation}
applying Cauchy-Schwarz inequality, and subsequently using Young's inequality with constant $C>0$, the inner product from \eqref{eq:pp-00o-88d} can for $k\geq 1$ be further bounded as follows		
		\begin{eqnarray}
		2\la K^{\top}y^{k+1} - K^{\top}\Bar{y}^k,x^{k+1}-x^{\star}\ra
			&\overset{\eqref{eq:biug9f7dgidf-989fd}}{=}& 2\la K^{\top}y^{k+1} - K^{\top}y^k,x^{k+1}-x^{\star}\ra \notag \\ 
			&& - 2\theta\la K^{\top}y^{k} - K^{\top}y^{k-1},x^{k}-x^{\star}\ra \notag \\
			&&+  2\theta\la K^{\top}y^{k} - K^{\top}y^{k-1},x^{k+1}-x^{k}\ra\notag \\
			&\leq& 		 2\la K^{\top}y^{k+1} - K^{\top}y^k,x^{k+1}-x^{\star}\ra \notag \\
			&& - 2\theta\la K^{\top}y^{k} - K^{\top}y^{k-1},x^{k}-x^{\star}\ra \notag \\
			&&+  2\theta L_{xy}\|y^{k} - y^{k-1}\|\|x^{k+1}-x^{k}\|\notag \\
			&\leq& + 2\la K^{\top}y^{k+1} - K^{\top}y^k,x^{k+1}-x^{\star}\ra \notag \\
			&&- 2\theta\la K^{\top}y^{k} - K^{\top}y^{k-1},x^{k}-x^{\star}\ra \notag \\
			&&+  \theta L_{xy}C\|y^{k} - y^{k-1}\|^2+\frac{\theta L_{xy}}{C}\|x^{k+1}-x^{k}\|^2.\label{eq:kjijhids-98ygfduf-dd}
		\end{eqnarray}

\paragraph{Showing that $\Delta^{k+1} \leq \theta \Delta^k$ for all $k\geq 1$.}	
Plugging the bounds  \eqref{eq:W^k-APDA-999} and \eqref{eq:kjijhids-98ygfduf-dd} into		
the Lyapunov function \eqref{Lyapunov_function_alg_2}, for any $k\geq 1$  we get 
		\begin{eqnarray*}
			\Delta^{k+1}
			&\overset{\eqref{Lyapunov_function_alg_2}+\eqref{eq:W^k-APDA-999}+\eqref{eq:kjijhids-98ygfduf-dd}}{\leq}& \frac{1}{\eta_x}\|x^k-x^{\star}\|^2 - \left(\frac{1}{\eta_x}-\frac{\theta L_{xy}}{C}\right)\|x^{k+1}-x^{k}\|^2 + (1-\mu^2_{xy}\beta_y\eta_y)\frac{1}{\eta_y}\|y^k-y^{\star}\|^2\\
			&& - 2\theta\la K^{\top}y^{k} - K^{\top}y^{k-1},x^{k}-x^{\star}\ra +  \theta L_{xy}C\|y^{k} - y^{k-1}\|^2\\
			&\overset{C = \theta\eta_xL_{xy}}{\leq}& \frac{1}{\eta_x}\|x^k-x^{\star}\|^2 + (1-\mu^2_{xy}\beta_y\eta_y)\frac{1}{\eta_y}\|y^k-y^{\star}\|^2\\
			&&- 2\theta\la K^{\top}y^{k} - K^{\top}y^{k-1},x^{k}-x^{\star}\ra +  \theta^2 L^2_{xy}\eta_x\eta_y\frac{1}{\eta_y}\|y^{k} - y^{k-1}\|^2\\
			&\overset{\eqref{eta_x_2}+\eqref{eta_y_2} \; \& \; \theta \leq 1}{\leq}& \frac{1}{\eta_x}\|x^k-x^{\star}\|^2 + (1-\mu^2_{xy}\beta_y\eta_y)\frac{1}{\eta_y}\|y^k-y^{\star}\|^2\\
			&&- 2\theta\la K^{\top}y^{k} - K^{\top}y^{k-1},x^{k}-x^{\star}\ra +  \theta\frac{1}{2\eta_y}\|y^{k} - y^{k-1}\|^2\\
			&\leq& \max\left\{\frac{1}{1+\mu_x\eta_x}, 1-\mu^2_{xy}\beta_y\eta_y\right\}\left((1+\mu_x\eta_x)\frac{1}{\eta_x}\|x^k-x^{\star}\|^2 + \frac{1}{\eta_y}\|y^k-y^{\star}\|^2\right)\\
			&&+ \theta\left(\frac{1}{2\eta_y}\|y^{k} - y^{k-1}\|^2- 2\la K^{\top}y^{k} - K^{\top}y^{k-1},x^{k}-x^{\star}\ra\right) \\
			&\overset{\eqref{theta_2}+\eqref{Lyapunov_function_alg_2}}{=}& \theta \Delta^k.
		\end{eqnarray*}
		
\paragraph{Showing that $\Delta^{k+1}\leq \theta \Delta^k$ for $k=0$.}	
This can be done using similar arguments as those used in the proof of Theorem~\ref{thm:Chambolle_Pock}.

	\end{proof}
	
\subsection{Proof of Theorem~\ref{thm:APDA-informal} (informal)}

We now provide the iteration  complexity of Algorithm \ref{alg:APDA} as a corollary of Theorem~\ref{thm:APDA}. 
Note that in view of \eqref{eq:jdd8t*gduI*_T*}, we get
\begin{eqnarray*}
		k
		\geq  \mathcal{O}\left(\frac{1}{1-\theta}\log\frac{1}{\varepsilon}\right) \qquad \Rightarrow \qquad \Delta^{k} \leq \varepsilon \Delta^{0}.	\end{eqnarray*}	

Using the definition of $\theta$ given in \eqref{theta_2}, we  have
	\begin{eqnarray*}
		\mathcal{O}\left(\frac{1}{1-\theta}\log\frac{1}{\varepsilon}\right)
		&=&  \mathcal{O}\left(\max\left\{1+\frac{1}{\mu_x\eta_x},\frac{1}{\mu^2_{xy}\beta_y\eta_y}\right\}\log\frac{1}{\varepsilon}\right)\\
		&=& \mathcal{O}\left(\max\left\{1+\frac{1}{\mu_x\eta_x},\frac{L_{x}}{\mu^2_{xy}\eta_y},\frac{L^2_{xy}}{\mu^2_{xy}}\right\}\log\frac{1}{\varepsilon}\right) \\
		&=& \mathcal{O}\left(\max\left\{1+\sqrt{\frac{L_x}{\mu_x}}\frac{L_{xy}}{\mu_{xy}},\frac{L^2_{xy}}{\mu^2_{xy}}\right\}\log\frac{1}{\varepsilon}\right).
	\end{eqnarray*}	
This proves the statement of Theorem~\ref{thm:APDA-informal} mentioned in Section~\ref{sec:APDA} in the main body of the paper.

\clearpage		
\section{Analysis of the Accelerated Primal-Dual Algorithm with Inexact Prox (Algorithm~\ref{alg:APDA-Inex})}
	
	In this section we provide convergence analysis for our second new method, APDA with Inexact Prox (Algorithm~\ref{alg:APDA-Inex}). We start with statements of three lemmas, followed by the proof of the main theorem.

\subsection{Three Lemmas}	
	The first result is essentially a modification of Lemma \ref{descent_lemma_by_x_alg_2}, providing a bound on the distance of the primal iterates from the primal optimal solution. Compare to Lemma  \ref{descent_lemma_by_x_alg_2}, these changes consist in using the definition of function $\Psi^k$ (see the problem \eqref{auxiliary_problem}) to prove this key fact.
	
	\begin{lemma}
		\label{descent_lemma_by_x}
		Let $w^{\star k}$ be an exact solution to the problem \eqref{auxiliary_problem}. Then under Assumptions \ref{as_strongly_convex}, \ref{as_smoothness}, we have		\begin{eqnarray*}
			\left(1+\frac{\mu_x\eta_x}{2}\right)\frac{1}{\eta_x}\|x^{k+1}-x^{\star}\|^2&\leq&
			\frac{1}{\eta_x}\|x^k-x^{\star}\|^2 + (2\eta_x+\mu_x\eta^2_x)\|\nabla\Psi^k(\hat{x}^k)\|^2 - \frac{1}{2\eta_x}\|x^k-w^{\star k}\|^2\\
			&& -2\la K^{\top}\Bar{y}^k - K^{\top}y^{\star},\hat{x}^k-x^{\star}\ra - \frac{1}{L_x}\|\nabla G(\hat{x}^k) - \nabla G(x^{\star})\|^2.
		\end{eqnarray*}
	\end{lemma}
	\begin{proof} In view of line 4 of  Algorithm~\ref{alg:APDA-Inex}, which reads \begin{equation}\label{eq:line4-98y98yfd}\color{blue}x^{k+1} = x^k -\eta_x\left(\nabla G(\hat{x}^k) + K^{\top}\Bar{y}^k\right),\end{equation} and writing $x^{k+1}$ as $x^{k}+ (x^{k+1}-x^k)$,  we get
		\begin{eqnarray*}
			\frac{1}{\eta_x}\|x^{k+1}-x^{\star}\|^2 &=& \frac{1}{\eta_x}\|x^{k}-x^{\star}\|^2 + \frac{2}{\eta_x}\la {\color{blue}x^{k+1}-x^{k}}, x^{k+1}-x^{\star}\ra -  \frac{1}{\eta_x}\|{\color{blue}x^{k+1}-x^{k}}\|^2\\
			&\overset{\eqref{eq:line4-98y98yfd}}{=}& \frac{1}{\eta_x}\|x^{k}-x^{\star}\|^2 -\frac{2}{\eta_x}\la {\color{blue}\eta_x\left(\nabla G(\hat{x}^k)+K^{\top}\Bar{y}^k\right)}, x^{k+1}-x^{\star}\ra -  \frac{1}{\eta_x}\left\|{\color{blue}\eta_x\left(\nabla G(\hat{x}^k) + K^{\top}\Bar{y}^k\right)} \right\|^2\\
			&=&  \frac{1}{\eta_x}\|x^{k}-x^{\star}\|^2 -2\la \nabla G(\hat{x}^k)+K^{\top}\Bar{y}^k, \hat{x}^{k}-x^{\star}\ra -  \eta_x \left\| \nabla G(\hat{x}^k)+K^{\top}\Bar{y}^k \right\|^2\\
			&& - \frac{2}{\eta_x} \left\la \eta_x\left(\nabla G(\hat{x}^k)+K^{\top}\Bar{y}^k\right), x^{k+1}-\hat{x}^{k} \right\ra.
		\end{eqnarray*}
		
		Using the identity $-2\la a,b\ra = \|a\|^2 +\|b\|^2 - \|a+b\|^2$ to rewrite the second inner product from the previous equation, we get
		\begin{eqnarray}
			\label{eq:12345}
			\frac{1}{\eta_x}\|x^{k+1}-x^{\star}\|^2 &=&
			\frac{1}{\eta_x}\|x^{k}-x^{\star}\|^2 -2\la \nabla G(\hat{x}^k)+K^{\top}\Bar{y}^k, \hat{x}^{k}-x^{\star}\ra -  \eta_x\| \nabla G(\hat{x}^k)+K^{\top}\Bar{y}^k\|^2\notag\\
			&& +\frac{1}{\eta_x}\left( \left\|\eta_x\left(\nabla G(\hat{x}^k)+K^{\top}\Bar{y}^k\right)\right\|^2 + \|x^{k+1}-\hat{x}^{k}\|^2\right)\notag\\
			&& -\frac{1}{\eta_x} \left\| \eta_x\left(\nabla G(\hat{x}^k)+K^{\top}\Bar{y}^k\right)+ x^{k+1}-\hat{x}^{k} \right\|^2\notag\\
			&\overset{\eqref{eq:line4-98y98yfd}}{=}& \frac{1}{\eta_x}\|x^{k}-x^{\star}\|^2 -2\la \nabla G(\hat{x}^k) - {\color{red} \nabla G(x^{\star})}, \hat{x}^{k}-x^{\star}\ra \notag\\
			&&-2\la K^{\top}\Bar{y}^k - {\color{red}K^{\top}y^{\star}}, \hat{x}^{k}-x^{\star}\ra   +\frac{1}{\eta_x}\|x^{k+1}-\hat{x}^{k}\|^2 -\frac{1}{\eta_x}\| x^k-\hat{x}^{k}\|^2,
		\end{eqnarray}
	where in the last equation we applied the optimality condition (see  \eqref{opt_conditions})
	$${\color{red} \nabla G(x^{\star})+ K^{\top}y^{\star} = 0}.$$

		Due to the fact that $G$ is  $\mu_x$-strongly-convex and $L_x$-smooth, we can estimate the inner product $2\la \nabla G(\hat{x}^k)- \nabla G(x^{\star}), \hat{x}^{k}-x^{\star}\ra $ from below similarly  as in the proof of Lemma \ref{descent_lemma_by_x_alg_2}:
		\begin{equation}
			\label{eq:uation1}
			2\la \nabla G(\hat{x}^{k}) - \nabla G(x^{\star}) , \hat{x}^{k}-x^{\star}\ra \geq \mu_x \| \hat{x}^{k}-x^{\star} \|^2 + \frac{1}{L_x} \|\nabla G(\hat{x}^{k}) - \nabla G(x^{\star}) \|^2.
		\end{equation}
	
		Now, by plugging \eqref{eq:uation1} into \eqref{eq:12345}, we obtain 
		\begin{eqnarray*}
			\frac{1}{\eta_x}\|x^{k+1}-x^{\star}\|^2 
			&\leq& \frac{1}{\eta_x}\|x^{k}-x^{\star}\|^2  - \mu_x\|\hat{x}^{k}-x^{\star}\|^2 - \frac{1}{L_x}\|\nabla G(\hat{x}^k)+\nabla G(x^{\star})\|^2 \\
			&& +\frac{1}{\eta_x}\|x^{k+1}-\hat{x}^{k}\|^2 -\frac{1}{\eta_x}\|x^k-\hat{x}^{k}\|^2 -2\la K^{\top}\Bar{y}^k - K^{\top}y^{\star}, \hat{x}^{k}-x^{\star}\ra.
		\end{eqnarray*}
		
		Applying the inequality $\|a-b\|^2 \geq \frac{1}{2}\|a-c\|^2 - \|b-c\|^2$ to  $\|\hat{x}^{k}-x^{\star}\|^2$ and taking $c = x^{k+1}$, we obtain
		\begin{eqnarray}
			\label{eq:tryruir}
			\frac{1}{\eta_x}\|x^{k+1}-x^{\star}\|^2 
			&\leq& \frac{1}{\eta_x}\|x^{k}-x^{\star}\|^2  - \frac{\mu_x}{2}\|x^{k+1}-x^{\star}\|^2 - \frac{1}{L_x}\|\nabla G(\hat{x}^k)-\nabla G(x^{\star})\|^2\notag \\
			&& +\left(1+\mu_x\eta_x\right)\frac{1}{\eta_x}\|x^{k+1}-\hat{x}^{k}\|^2-\frac{1}{\eta_x}\|x^k-\hat{x}^{k}\|^2 \notag\\
			&&-2\la K^{\top}\Bar{y}^k - K^{\top}y^{\star}, \hat{x}^{k}-x^{\star}\ra \notag\\
			&\leq& \frac{1}{\eta_x}\|x^{k}-x^{\star}\|^2 -\frac{\mu_x}{2}\|x^{k+1}-x^{\star}\|^2 - \frac{1}{L_x}\|\nabla G(\hat{x}^k)-\nabla G(x^{\star})\|^2 \notag \\
			&& +\left(1+\mu_x\eta_x\right)\frac{1}{\eta_x}\|x^{k+1}-\hat{x}^{k}\|^2 -\frac{1}{2\eta_x}\|x^k-w^{\star k}\|^2 + \frac{1}{\eta_x}\|\hat{x}^{k}-w^{\star k}\|^2 \notag\\
			&&-2\la K^{\top}\Bar{y}^k - K^{\top}y^{\star}, \hat{x}^{k}-x^{\star}\ra,
		\end{eqnarray}
		where in the  last inequality we used the bound $\|a-b\|^2 \geq \frac{1}{2}\|a-c\|^2 - \|b-c\|^2$ to estimate  $\frac{1}{\eta_x}\|x^k-\hat{x}^{k}\|^2$. From line 3 of Algorithm~\ref{alg:APDA-Inex}, according to the definition of function $\Psi^k(x)$, we get
		\begin{eqnarray}
			\label{eq:ghvdscg}
			\nabla \Psi^k(\hat{x}^k) &=& \nabla G(\hat{x}^k) + K^{\top}\bar{y}^k + \frac{1}{\eta_x}\left(\hat{x}^k -x^k\right)\notag\\
			&=&\nabla G(\hat{x}^k) + K^{\top}\bar{y}^k + \frac{1}{\eta_x}\left(\hat{x}^k -x^{k+1}\right) + \frac{1}{\eta_x}\left(x^{k+1} -x^k\right) \notag\\
			&\overset{\eqref{eq:line4-98y98yfd}}{=}&  \frac{1}{\eta_x}\left(\hat{x}^k -x^{k+1}\right) .
		\end{eqnarray}
		
	 	Finally, substituting \eqref{eq:ghvdscg} into \eqref{eq:tryruir}, we gain
		\begin{eqnarray*}
			\frac{1}{\eta_x}\|x^{k+1}-x^{\star}\|^2 
			&=& \frac{1}{\eta_x}\|x^{k}-x^{\star}\|^2 -\frac{\mu_x}{2}\|x^{k+1}-x^{\star}\|^2 - \frac{1}{L_x}\|\nabla G(\hat{x}^k)-\nabla G(x^{\star})\|^2 \\
			&& +\left(2\eta_x+\mu_x\eta^2_x\right)\|\nabla\Psi^k(\hat{x}^k)\|^2 -\frac{1}{2\eta_x}\|x^k-w^{\star k}\|^2 \\
			&&-2\la K^{\top}\Bar{y}^k - K^{\top}y^{\star}, \hat{x}^{k}-x^{\star}\ra,
		\end{eqnarray*}
	where we also use  $\frac{1}{\eta_x}$-strong convexity of function $\Psi^k(x)$ (see problem \eqref{auxiliary_problem}). 
	\end{proof}

	The second result offers a bound on the distance between the dual itarates and the dual optimal solution as Lemma~\ref{descent_lemma_by_y_alg_2}. Since the line 4 of Algorithm~\ref{alg:APDA} coincides with the line 5 of Algorithm~\ref{alg:APDA-Inex}, the statement and proof of the following lemma coincides with Lemma~\ref{descent_lemma_by_y_alg_2} and its proof with the only change  of replacing $x^{k+1}$ with $\hat{x}^{k}$.  
	\begin{lemma}
		\label{descent_lemma_by_y}
Then under Assumptions~\ref{as_smoothness}, \ref{as_convexity}, and \ref{as_L_xy} be satisfied, and choose
		\begin{equation}
			\label{beta_y}
			\beta_y \leq \frac{1}{2L^2_{xy}\eta_y}.
		\end{equation}
Then the iterates of APDA with Inexact Prox (Algorithm~\ref{alg:APDA-Inex}) for all $k\geq 0$ satisfy		
				\begin{align*}
			\frac{1}{\eta_y}\|y^{k+1} - y^{\star}\|^2  
			&\leq (1-\mu^2_{xy}\beta_y\eta_y)\frac{1}{\eta_y}\|y^k-y^{\star}\|^2 - \frac{1}{2\eta_y}\|y^{k+1}-y^{k}\|^2 + 2\la K^{\top}y^{k+1} - K^{\top}y^{\star},\hat{x}^k-x^{\star}\ra \\
			& \quad + \beta_y\|\nabla G(\hat{x}^k) - \nabla G(x^{\star})\|^2.
		\end{align*}
	\end{lemma}
	
	\begin{proof} 
The proof is identical to the proof of Lemma~\ref{descent_lemma_by_y_alg_2}, with the sole distinction that $x^{k+1}$ should be replaced everywhere by $\hat{x}^k$.
	\end{proof}
	
	The third result is a technical lemma, which offers a bound on an inner product . 
	\begin{lemma}
		\label{inner_product}
		Let $w^{\star k}$ be an exact solution to the problem \eqref{auxiliary_problem}. Then the following inequality holds
		\begin{eqnarray*}
			-2\theta\la K^{\top}y^k - K^{\top}y^{k-1}, \hat{x}^k -\hat{x}^{k-1}\ra &\leq& 
			16L^2_{xy}\theta\eta_x \|y^k -y^{k-1}\|^2 + \frac{\theta}{4\eta_x}\|\hat{x}^k -w^{\star k}\|^2 \\
			&&+\frac{\theta}{4\eta_x}\|x^k - w^{\star k}\|^2+ \frac{\theta\eta_x}{8}\|\nabla \Psi^{k-1}(\hat{x}^{k-1})\|^2.
		\end{eqnarray*}
	\end{lemma}
	\begin{proof}
		Using the Cauchy-Schwarz inequality, the definition of $L_{xy}$ as the norm of $K$ (see \eqref{L_xy}),  and applying Young's inequality, we can bound the inner product by 
		\begin{eqnarray*}
			-2\theta\la K^{\top}y^k - K^{\top}y^{k-1}, \hat{x}^k -\hat{x}^{k-1}\ra &\leq& 2\theta \|K^{\top}y^k - K^{\top}y^{k-1}\|\|\hat{x}^k -\hat{x}^{k-1}\|\\
			&\leq& L_{xy}\theta C\|y^k -y^{k-1}\|^2 + \frac{L_{xy}\theta}{C}\|\hat{x}^k -\hat{x}^{k-1}\|^2\\
			&\leq&L_{xy}\theta C\|y^k -y^{k-1}\|^2 + \frac{2L_{xy}\theta}{C}\|\hat{x}^k -x^k\|^2\\
			&& + \frac{2L_{xy}\theta}{C}\|x^k -\hat{x}^{k-1}\|^2
		\end{eqnarray*}
	for any $C > 0 $. 
	Taking $C = 16L_{xy}\eta_x$, we obtain
	\begin{eqnarray*}
		-2\theta\la K^{\top}y^k - K^{\top}y^{k-1}, \hat{x}^k -\hat{x}^{k-1}\ra &\leq&L_{xy}\theta C\|y^k -y^{k-1}\|^2 + \frac{4L_{xy}\theta}{C}\|\hat{x}^k -w^{\star k}\|^2\\
		&& \frac{4L_{xy}\theta}{C}\|\hat{x}^k -w^{\star k}\|^2+ \frac{2L_{xy}\theta}{C}\|x^k -\hat{x}^{k-1}\|^2\\
		&=&16L^2_{xy}\theta\eta_x \|y^k -y^{k-1}\|^2 + \frac{\theta}{4\eta_x}\|\hat{x}^k -w^{\star k}\|^2 \\
		&&+\frac{\theta}{4\eta_x}\|x^k - w^{\star k}\|^2+ \frac{\theta\eta_x}{8}\|\nabla \Psi^{k-1}(\hat{x}^{k-1})\|^2,
	\end{eqnarray*}
		where in the last equation we use \eqref{eq:ghvdscg} for $k-1$, which reads $\nabla \Psi^k(\hat{x}^{k-1}) = \frac{1}{\eta_x}\left(\hat{x}^{k-1} -x^{k}\right) $.
	\end{proof}

\subsection{Detailed theorem}	
	\begin{theorem}[Convergence of APDA with Inexact Prox; formal]
		\label{th_conv_sc_1}
		Let Assumptions \ref{as_strongly_convex}, \ref{as_smoothness}, \ref{as_convexity} and \ref{as_L_xy} hold and  select the various parameters of the method as follows: 
		\begin{equation}
			\label{def_beta_y}
			\beta_y = \min\left\{\frac{1}{L_x}, \frac{1}{2L^2_{xy}\eta_y}\right\};
		\end{equation}
		\begin{equation}
			\label{local_iterations}
			T = \sqrt[\alpha]{20A}\left(1+ \sqrt{\frac{L_x}{\mu_x}}\right)^{\nicefrac{2}{\alpha}};
		\end{equation} 
		\begin{equation}
			\label{eta_x}
			\eta_x = \frac{1}{4\sqrt{L_x\mu_x}}\frac{\mu_{xy}}{L_{xy}};
		\end{equation}
		\begin{equation}
			\label{eta_y}
			\eta_y = \frac{\sqrt{L_x\mu_x}}{8L_{xy}\mu_{xy}};
		\end{equation}
		\begin{equation}
			\label{theta}
			\theta = \max\left\{\frac{2}{2+ \mu_x\eta_x}, 1 - \mu^2_{xy}\beta_y\eta_y\right\},
		\end{equation}
		Then for the Lyapunov function  defined for $k \geq 0$ via
		\begin{eqnarray}
			\label{Lyapunov_function}
			\Delta^{k+1} &\eqdef & \left(1+\frac{\mu_x\eta_x}{2}\right)\frac{1}{\eta_x}\|x^{k+1}-x^{\star}\|^2+ \frac{1}{\eta_y}\|y^{k+1} - y^{\star}\|^2 +\frac{1}{2\eta_y}\|y^{k+1}-y^{k}\|^2 \notag\\
			&&+ \frac{1}{8\eta_x}\|x^k-w^{\star k}\|^2 - 2\la K^{\top}y^{k+1} - K^{\top}y^{k},\hat{x}^k-x^{\star}\ra,
		\end{eqnarray}
	and for $k = 0$ via 
	\begin{equation}
		\label{Lyapunov_function_0}
		\Delta^0 \eqdef \left(1+\frac{\mu_x\eta_x}{2}\right)\frac{1}{\eta_x}\|x^{0}-x^{\star}\|^2+ \frac{1}{\eta_y}\|y^{0} - y^{\star}\|^2,
	\end{equation}
		we have
		\begin{equation}
			 \frac{1}{2\eta_x}\|x^{k+1}-x^{\star}\|^2+\frac{1}{\eta_y}\|y^{k+1} - y^{\star}\|^2 \leq \Delta^{k}\leq \theta^k\Delta^0,\quad \forall~k\geq 0.
		\end{equation}
	\end{theorem}
	\begin{proof}
		We denote $\mathcal{V}^{k+1}$ as follows:
		\begin{equation}
			\label{fist_part_of_lyapunov_function}
			\mathcal{V}^{k+1} = \left(1+\frac{\mu_x\eta_x}{2}\right)\frac{1}{\eta_x}\|x^{k+1}-x^{\star}\|^2+ \frac{1}{\eta_y}\|y^{k+1} - y^{\star}\|^2 +\frac{1}{2\eta_y}\|y^{k+1}-y^{k}\|^2  .
		\end{equation}
		Combining Lemmas \ref{descent_lemma_by_x}, \ref{descent_lemma_by_y} and \eqref{def_beta_y}, we obtain 
		\begin{eqnarray}
			\label{eq:tyfytsdtyvs}
			\mathcal{V}^{k+1}
			&\leq& \frac{1}{\eta_x}\|x^k-x^{\star}\|^2 + (1-\mu^2_{xy}\beta_y\eta_y)\frac{1}{\eta_y}\|y^k-y^{\star}\|^2  +(2\eta_x+\mu_x\eta^2_x)\|\nabla\Psi^k(\hat{x}^k)\|^2 \notag\\
			&&- \frac{1}{2\eta_x}\|x^k-w^{\star k}\|^2  +\left(\beta_y- \frac{1}{L_x}\right)\|\nabla G(\hat{x}^k) - \nabla G(x^{\star})\|^2\notag\\
			&& -2\la K^{\top}\Bar{y}^k - K^{\top}y^{\star},\hat{x}^k-x^{\star}\ra +2\la K^{\top}y^{k+1} - K^{\top}y^{\star},\hat{x}^k-x^{\star}\ra \\
			&\leq& \frac{1}{\eta_x}\|x^k-x^{\star}\|^2 + (1-\mu^2_{xy}\beta_y\eta_y)\frac{1}{\eta_y}\|y^k-y^{\star}\|^2  +(2\eta_x+\mu_x\eta^2_x)\|\nabla\Psi^k(\hat{x}^k)\|^2\notag \\
			&& - \frac{1}{2\eta_x}\|x^k-w^{\star k}\|^2 + 2\la K^{\top}y^{k+1} - K^{\top}y^{k},\hat{x}^k-x^{\star}\ra \notag\\
			&& - 2\theta \la K^{\top}y^{k} - K^{\top}y^{k-1},\hat{x}^{k-1}-x^{\star}\ra - 2\theta \la K^{\top}y^{k} - K^{\top}y^{k-1},\hat{x}^k-\hat{x}^{k-1}\ra , \notag
		\end{eqnarray}
		where in last inequality the line 6 from Algorithm~\ref{alg:APDA-Inex} is used. Next,
		according to Lemma \ref{inner_product}, we gain
		\begin{eqnarray*}
			\mathcal{V}^{k+1}
			&\leq& \frac{1}{\eta_x}\|x^k-x^{\star}\|^2 + (1-\mu^2_{xy}\beta_y\eta_y)\frac{1}{\eta_y}\|y^k-y^{\star}\|^2  +(2\eta_x+\mu_x\eta^2_x)\|\nabla\Psi^k(\hat{x}^k)\|^2 - \frac{1}{2\eta_x}\|x^k-w^{\star k}\|^2\\
			&& + 2\la K^{\top}y^{k+1} - K^{\top}y^{k},\hat{x}^k-x^{\star}\ra - 2\theta \la K^{\top}y^{k} - K^{\top}y^{k-1},\hat{x}^{k-1}-x^{\star}\ra \\ 
			&& + 16\theta L^2_{xy}\eta_x\|y^k-y^{k-1}\|^2+ \frac{\theta}{4\eta_x} \|\hat{x}^k-w^{\star k}\|^2 + \frac{\theta}{4\eta_x}\|x^k-w^{\star k}\|^2 + \frac{\theta\eta_x}{8}\|\nabla \Psi^{k-1}_{\eta_x}(\hat{x}^{k-1})\|^2\\
			&\leq& \frac{1}{\eta_x}\|x^k-x^{\star}\|^2 + (1-\mu^2_{xy}\beta_y\eta_y)\frac{1}{\eta_y}\|y^k-y^{\star}\|^2 + 16\theta L^2_{xy}\eta_x\eta_y\frac{1}{\eta_y}\|y^k-y^{k-1}\|^2 \\
			&&+(2\eta_x+\mu_x\eta^2_x)\|\nabla\Psi^k(\hat{x}^k)\|^2 - \frac{1}{4\eta_x}\|x^k-w^{\star k}\|^2 + {\color{red}\frac{\theta}{4\eta_x} \|\hat{x}^k-w^{\star k}\|^2 } + \frac{\theta\eta_x}{8}\|\nabla \Psi^{k-1}_{\eta_x}(\hat{x}^{k-1})\|^2 \\
			&& + 2\la K^{\top}y^{k+1} - K^{\top}y^{k},\hat{x}^k-x^{\star}\ra - 2\theta \la K^{\top}y^{k} - K^{\top}y^{k-1},\hat{x}^{k-1}-x^{\star}\ra .
		\end{eqnarray*}
		
		Using definition of function $\Psi^k$ and its $\frac{1}{\eta_x}$-strong convexity in the following form $${\color{red}\|\nabla\Psi^k(\hat{x}^k)\|^2 \geq \frac{1}{\eta_x}\|\hat{x}^{k} - w^{\star k}\|^2},$$ and assuming that $32L^2_{xy}\eta_x\eta_y \leq 1$ and $0 < \theta \leq 1$, we get 
		\begin{eqnarray*}
			\mathcal{V}^{k+1}
			&\leq& \frac{1}{\eta_x}\|x^k-x^{\star}\|^2 + (1-\mu^2_{xy}\beta_y\eta_y)\frac{1}{\eta_y}\|y^k-y^{\star}\|^2 +  \frac{\theta}{2\eta_y}\|y^k-y^{k-1}\|^2 \\
			&&+(2\eta_x+\mu_x\eta^2_x)\|\nabla\Psi^k(\hat{x}^k)\|^2 - \frac{1}{4\eta_x}\|x^k-w^{\star k}\|^2 + {\color{red}\frac{\theta\eta_x}{4} \|\nabla\Psi^k(\hat{x}^k)\|^2 }+ \frac{\theta\eta_x}{8}\|\nabla \Psi^{k-1}_{\eta_x}(\hat{x}^{k-1})\|^2 \\
			&& + 2\la K^{\top}y^{k+1} - K^{\top}y^{k},\hat{x}^k-x^{\star}\ra - 2\theta \la K^{\top}y^{k} - K^{\top}y^{k-1},\hat{x}^{k-1}-x^{\star}\ra\\
			&\overset{0 < \theta \leq 1}{\leq}& \frac{1}{\eta_x}\|x^k-x^{\star}\|^2 + (1-\mu^2_{xy}\beta_y\eta_y)\frac{1}{\eta_y}\|y^k-y^{\star}\|^2 +  \frac{\theta}{2\eta_y}\|y^k-y^{k-1}\|^2 \\
			&&+\left(\frac{9}{4}\eta_x+\mu_x\eta^2_x\right){\color{blue}\|\nabla\Psi^k(\hat{x}^k)\|^2 }- \frac{1}{4\eta_x}\|x^k-w^{\star k}\|^2   + \frac{\theta\eta_x}{8}{\color{blue}\|\nabla \Psi^{k-1}_{\eta_x}(\hat{x}^{k-1})\|^2} \\
			&& + 2\la K^{\top}y^{k+1} - K^{\top}y^{k},\hat{x}^k-x^{\star}\ra - 2\theta \la K^{\top}y^{k} - K^{\top}y^{k-1},\hat{x}^{k-1}-x^{\star}\ra .
		\end{eqnarray*}
		
		To estimate $\|\nabla \Psi^k(\hat{x}^{k})\|^2$ we apply Lemma \ref{convergence_norm_grad} for the problem~\eqref{auxiliary_problem} as follows (see \eqref{complexity_aux_problem}):
		$${\color{blue}\|\nabla \Psi^k(\hat{x}^{k})\|^2 \leq \frac{A\left(1+ \eta_x L_x\right)^2\|x^k-w^{\star k}\|^2}{\eta^2_x T^{\alpha}}}.$$
		
		According to \eqref{complexity_aux_problem}, \eqref{local_iterations} and \eqref{eta_x}, we obtain
		\begin{eqnarray*}
			\mathcal{V}^{k+1}
			&\overset{\eqref{complexity_aux_problem}}{\leq}& \frac{1}{\eta_x}\|x^k-x^{\star}\|^2 + (1-\mu^2_{xy}\beta_y\eta_y)\frac{1}{\eta_y}\|y^k-y^{\star}\|^2 +  \frac{\theta}{2\eta_y}\|y^k-y^{k-1}\|^2 \\
			&&+\left(\frac{9}{4}\eta_x+\mu_x\eta^2_x\right){\color{blue}\frac{A\left(1+ \eta_x L_x\right)^2}{\eta^2_x T^{\alpha}}\|x^k-w^{\star k}\|^2 }- \frac{1}{4\eta_x}\|x^k-w^{\star k}\|^2  \\
			&& + \frac{\theta\eta_x}{8}{\color{blue}\frac{A\left(1+ \eta_x L_x\right)^2}{\eta^2_x T^{\alpha}}\|x^{k-1}-w^{\star k-1}\|^2 }\\
			&& + 2\la K^{\top}y^{k+1} - K^{\top}y^{k},\hat{x}^k-x^{\star}\ra - 2\theta \la K^{\top}y^{k} - K^{\top}y^{k-1},\hat{x}^{k-1}-x^{\star}\ra\\
			&\overset{\eqref{local_iterations}+ \eqref{eta_x}}{\leq}& \frac{1}{\eta_x}\|x^k-x^{\star}\|^2 + (1-\mu^2_{xy}\beta_y\eta_y)\frac{1}{\eta_y}\|y^k-y^{\star}\|^2 +  \frac{\theta}{2\eta_y}\|y^k-y^{k-1}\|^2 \\
			&&+ \frac{\theta}{8\eta_x}\|x^{k-1}-w^{\star k-1}\|^2 - \frac{1}{8\eta_x}\|x^k-w^{\star k}\|^2 \\
			&&+ 2\la K^{\top}y^{k+1} - K^{\top}y^{k},\hat{x}^k-x^{\star}\ra - 2\theta \la K^{\top}y^{k} - K^{\top}y^{k-1},\hat{x}^{k-1}-x^{\star}\ra .
		\end{eqnarray*}
		
		 \paragraph{Showing that $\Delta^{k+1}\leq \theta \Delta^k$ for $k\geq 1$.}				Using the definition of $\Delta^{k+1}$ (see \eqref{Lyapunov_function}), we have 
		\begin{eqnarray*}
			\Delta^{k+1} &\leq& \frac{1}{\eta_x}\|x^k-x^{\star}\|^2 + (1-\mu^2_{xy}\beta_y\eta_y)\frac{1}{\eta_y}\|y^k-y^{\star}\|^2 +  \frac{\theta}{2\eta_y}\|y^k-y^{k-1}\|^2 \\
			&& + \frac{\theta}{8\eta_x}\|x^{k-1}-w^{\star k-1}\|^2 - 2\theta \la K^{\top}y^{k} - K^{\top}y^{k-1},\hat{x}^{k-1}-x^{\star}\ra\\
			&\leq& \max\left\{\frac{2}{2+ \mu_x\eta_x}, 1 - \mu^2_{xy}\beta_y\eta_y\right\}\left(\left(1+\frac{\mu_x\eta_x}{2}\right)\frac{1}{\eta_x}\|x^k-x^{\star}\|^2 + \frac{1}{\eta_y}\|y^k-y^{\star}\|^2\right)\\
			&&+\theta\left(\frac{1}{2\eta_y}\|y^k-y^{k-1}\|^2 + \frac{1}{8\eta_x}\|x^{k-1}-w^{\star k-1}\|^2 - 2\la K^{\top}y^{k} - K^{\top}y^{k-1},\hat{x}^{k-1}-x^{\star}\ra\right) \\
			&\overset{\eqref{theta}}{\leq}&  \theta \Delta^{k},
		\end{eqnarray*}
		where in the last inequality we take $\theta = \max\left\{\frac{2}{2+\mu_x\eta_x},1-\mu^2_{xy}\beta_y\eta_y\right\} $ (see \eqref{theta}).
		
		\paragraph{Showing that $\Delta^{k+1}\leq \theta \Delta^k$ for $k = 0$.}			We start with inequality \eqref{eq:tyfytsdtyvs} for $k = 0$:	
		\begin{eqnarray}
			\label{eq:poiubvcd}
			\mathcal{V}^{1}
			&\overset{\eqref{eq:tyfytsdtyvs}}{\leq}& \frac{1}{\eta_x}\|x^0-x^{\star}\|^2 + (1-\mu^2_{xy}\beta_y\eta_y)\frac{1}{\eta_y}\|y^0-y^{\star}\|^2  +(2\eta_x+\mu_x\eta^2_x)\|\nabla\Psi^k(\hat{x}^0)\|^2 \notag\\
			&&- \frac{1}{2\eta_x}\|x^0-w^{\star 0}\|^2  +\left(\beta_y- \frac{1}{L_x}\right)\|\nabla G(\hat{x}^0) - \nabla G(x^{\star})\|^2\notag\\
			&& -2\la K^{\top}\Bar{y}^0 - K^{\top}y^{\star},\hat{x}^0-x^{\star}\ra +2\la K^{\top}y^{1} - K^{\top}y^{\star},\hat{x}^0-x^{\star}\ra \notag\\
			&\leq&  \frac{1}{\eta_x}\|x^0-x^{\star}\|^2 + (1-\mu^2_{xy}\beta_y\eta_y)\frac{1}{\eta_y}\|y^0-y^{\star}\|^2  -  \frac{1}{2\eta_x}\|x^0-w^{\star 0}\|^2\notag\\
			&& +(2\eta_x+\mu_x\eta^2_x)\|\nabla\Psi^k(\hat{x}^0)\|^2  + 2\la K^{\top}y^{1} - K^{\top}y^0,\hat{x}^0-x^{\star}\ra,
		\end{eqnarray}
		where in the last inequality we take $\beta_y \leq \frac{1}{L_x}$ (see \eqref{def_beta_y}). Now, we apply Lemma \ref{lem_convergence_norm_of_gradient} to the problem \eqref{auxiliary_problem} for $k = 0$ (see \eqref{complexity_aux_problem}):
		\begin{equation}
			\label{eq:hdudusu}
			\|\nabla\Psi^k(\hat{x}^0)\|^2 \leq \frac{A(1+\eta_xL_x)^2\|x^0-w^{\star 0}\|^2}{\eta_x T^{\alpha}}.
		\end{equation}
		Substituting \eqref{eq:hdudusu} into \eqref{eq:poiubvcd}, we can get 
		\begin{eqnarray}
			\label{eq:dnfjvnfkjn}
			\mathcal{V}^{1}
			&\leq&  \frac{1}{\eta_x}\|x^0-x^{\star}\|^2 + (1-\mu^2_{xy}\beta_y\eta_y)\frac{1}{\eta_y}\|y^0-y^{\star}\|^2  -  {\color{blue}\frac{1}{2\eta_x}\|x^0-w^{\star 0}\|^2}\notag\\
			&& +{\color{blue}(2\eta_x+\mu_x\eta^2_x)\frac{A(1+\eta_xL_x)^2}{\eta_x T^{\alpha}}\|x^0-w^{\star 0}\|^2 }+ 2\la K^{\top}y^{1} - K^{\top}y^0,\hat{x}^0-x^{\star}\ra\notag\\
			&\overset{\eqref{local_iterations}}{\leq}&  \frac{1}{\eta_x}\|x^0-x^{\star}\|^2 + (1-\mu^2_{xy}\beta_y\eta_y)\frac{1}{\eta_y}\|y^0-y^{\star}\|^2  -  {\color{blue}\frac{1}{8\eta_x}\|x^0-w^{\star 0}\|^2}\notag\\
			&& + 2\la K^{\top}y^{1} - K^{\top}y^0,\hat{x}^0-x^{\star}\ra,
		\end{eqnarray}
	where in the last inequality we take $T$ according to \eqref{local_iterations}. 
According to \eqref{Lyapunov_function} and \eqref{Lyapunov_function_0}, we have 
		\begin{eqnarray*}
			\Delta^1 &\overset{\eqref{Lyapunov_function}}{=}&  \mathcal{V}^{1} +  \frac{1}{8\eta_x}\|x^0-w^{\star 0}\|^2 - 2\la K^{\top}y^{1} - K^{\top}y^0,\hat{x}^0-x^{\star}\ra\\
			&\overset{\eqref{eq:dnfjvnfkjn}}{\leq}& \frac{1}{\eta_x}\|x^0-x^{\star}\|^2 + (1-\mu^2_{xy}\beta_y\eta_y)\frac{1}{\eta_y}\|y^0-y^{\star}\|^2\\
			&\leq& \theta\left(\left(1+\frac{\mu_x\eta_x}{2}\right)\frac{1}{\eta_x}\|x^0-x^{\star}\|^2 +\frac{1}{\eta_y}\|y^0-y^{\star}\|^2\right)\\
			&=& \theta \Delta^0,
		\end{eqnarray*}
	where in the last inequality we used the inequalities $1 \leq \theta\left(1+\frac{\mu_x\eta_x}{2}\right)$ and $1-\mu^2_{xy}\beta_y\eta_y\leq \theta$, which follow from \eqref{theta}.

		\paragraph{Showing that $\Delta^{k+1}\geq 0$ for $k\geq 0$.}	Finally, we need to show that $\Delta^k \geq 0$ for every $k$. This is obvious for $k=0$ from . Using the Cauchy-Schwarz inequality, the definition of $L_{xy}$ as the norm of $K$ (see \eqref{L_xy}), and applying Young's inequality,  we get
		\begin{eqnarray*}
			\Delta^{k+1} &\overset{\eqref{Lyapunov_function}}{=}& \left(1+\frac{\mu_x\eta_x}{2}\right)\frac{1}{\eta_x}\|x^{k+1}-x^{\star}\|^2+ \frac{1}{\eta_y}\|y^{k+1} - y^{\star}\|^2 +\frac{1}{2\eta_y}\|y^{k+1}-y^{k}\|^2 \notag\\
			&&+ \frac{1}{8\eta_x}\|x^k-w^{\star k}\|^2 - 2\la K^{\top}y^{k+1} - K^{\top}y^{k},\hat{x}^k-x^{\star}\ra\\
			&\overset{\eqref{L_xy}}{\geq}& \left(1+\frac{\mu_x\eta_x}{2}\right)\frac{1}{\eta_x}\|x^{k+1}-x^{\star}\|^2+ {\color{red}\frac{1}{\eta_y}\|y^{k+1} - y^{\star}\|^2} +\frac{1}{2\eta_y}\|y^{k+1}-y^{k}\|^2 \notag\\
			&&+ \frac{1}{8\eta_x}\|x^k-w^{\star k}\|^2 - 2L_{xy}{\color{red}\|y^{k+1} - y^{k}\|}{\color{blue}\|\hat{x}^k-x^{\star}\|}\\
			&\geq& \left(1+\frac{\mu_x\eta_x}{2}\right)\frac{1}{\eta_x}\|x^{k+1}-x^{\star}\|^2+ \frac{1}{\eta_y}\|y^{k+1} - y^{\star}\|^2 +{\color{red}\left(\frac{1}{2} - L_{xy}B\eta_y\right)\frac{1}{\eta_y}\|y^{k+1}-y^{k}\|^2} \notag\\
			&&+ \frac{1}{8\eta_x}\|x^k-w^{\star k}\|^2  -{\color{blue}\frac{L_{xy}}{B}\|\hat{x}^k-x^{\star}\|^2}\\
			&\geq& \left(1+\frac{\mu_x\eta_x}{2}\right)\frac{1}{\eta_x}\|x^{k+1}-x^{\star}\|^2+ \frac{1}{\eta_y}\|y^{k+1} - y^{\star}\|^2 +\left(\frac{1}{2} - L_{xy}B\eta_y\right)\frac{1}{\eta_y}\|y^{k+1}-y^{k}\|^2 \notag\\
			&&+ \frac{1}{8\eta_x}\|x^k-w^{\star k}\|^2  -\frac{2L_xy}{B}\|\hat{x}^k-x^{k+1}\|^2 -\frac{2L_xy}{B}\|x^{k+1}-x^{\star}\|^2 .
		\end{eqnarray*}
		for any $B > 0$. According to the definition of function $\Psi^k$ (see \ref{auxiliary_problem}), selecting $B = \frac{1}{2L_{xy}\eta_y}$, we get 
		\begin{eqnarray}
			\label{eq:jfnvsjfj}
			\Delta^{k+1} 
			&\geq& \left(1+\frac{\mu_x\eta_x}{2}\right)\frac{1}{\eta_x}\|x^{k+1}-x^{\star}\|^2+ \frac{1}{\eta_y}\|y^{k+1} - y^{\star}\|^2 + \frac{1}{8\eta_x}\|x^k-w^{\star k}\|^2\notag\\
			&& -4L^2_{xy}\eta_y\eta_x\frac{1}{\eta_x}\|\eta_x\nabla\Psi^{k}_{\eta_x}(\hat{x}^k)\|^2 -4L^2_{xy}\eta_y\eta_x\frac{1}{\eta_x}\|x^{k+1}-x^{\star}\|^2\notag\\
			&=& {\color{blue}\left(1-4L^2_{xy}\eta_y\eta_x+\frac{\mu_x\eta_x}{2}\right)\frac{1}{\eta_x}\|x^{k+1}-x^{\star}\|^2}\notag\\
			&& +\frac{1}{\eta_y}\|y^{k+1} - y^{\star}\|^2 + \frac{1}{8\eta_x}\|x^k-w^{\star k}\|^2 -{\color{blue}4L^2_{xy}\eta_y\eta_x\frac{1}{\eta_x}\|\eta_x\nabla\Psi^{k}_{\eta_x}(\hat{x}^k)\|^2 }.
		\end{eqnarray}
		
		Choosing stepsizes $\eta_x$, $\eta_y$ according to \eqref{eta_x} and \eqref{eta_y}, we can derive the following inequality: 
		\begin{equation} 
			\label{eq:cbhisjnxsdj}
			{\color{blue}32L^2_{xy}\eta_x\eta_y \leq 1}.
		\end{equation}
		
		Now, plugging \eqref{eq:cbhisjnxsdj} and \eqref{complexity_aux_problem} into \eqref{eq:jfnvsjfj},  we obtain
		\begin{eqnarray*}
			\Delta^{k+1}
			&\overset{\eqref{eq:cbhisjnxsdj}}{\geq}& {\color{blue}\left(\frac{7}{8}+\frac{\mu_x\eta_x}{2}\right)\frac{1}{\eta_x}\|x^{k+1}-x^{\star}\|^2}+\frac{1}{\eta_y}\|y^{k+1} - y^{\star}\|^2 \\
			&&+ \frac{1}{8\eta_x}\|x^k-w^{\star k}\|^2  -{\color{blue}\frac{1}{8\eta_x}\|\eta_x\nabla\Psi^{k}_{\eta_x}(\hat{x}^k)\|^2} \\
			&\overset{\eqref{complexity_aux_problem}}{\geq}& \left(\frac{7}{8}+\frac{\mu_x\eta_x}{2}\right)\frac{1}{\eta_x}\|x^{k+1}-x^{\star}\|^2+\frac{1}{\eta_y}\|y^{k+1} - y^{\star}\|^2\\
			&& + \frac{1}{8\eta_x}\|x^k-w^{\star k}\|^2  -\frac{1}{8\eta_x}\frac{C\left(1+ \eta_x L_x\right)^2}{T^{\alpha}}\|x^k-w^{\star k}\|^2\\
			&\overset{\eqref{local_iterations}}{\geq}& \frac{1}{2\eta_x}\|x^{k+1}-x^{\star}\|^2+\frac{1}{\eta_y}\|y^{k+1} - y^{\star}\|^2,
		\end{eqnarray*}
		where in last inequality we use \eqref{local_iterations} to evaluate two last terms from below by zero. 
\end{proof}

\subsection{Proof of Theorem \ref{alg:APDA-Inex}}
To prove the informal result \eqref{complexity_alg_1}, we simply apply the above theorem. Using the definition of $\theta$, the theorem implies that
	\begin{equation*}
		k  \geq \frac{1}{1-\theta}\log\frac{1}{\varepsilon} \quad \Rightarrow \quad \Delta^k \leq \varepsilon \Delta^0.
	\end{equation*}

It remains to remark that the number of outer iterations to find $\varepsilon$-solution is equal to 
	\begin{eqnarray*}
		 \mathcal{O}\left(\frac{1}{1-\theta}\log\frac{1}{\varepsilon}\right) &=&  \mathcal{O}\left(\max\left\{1+\frac{2}{\mu_x\eta_x},\frac{1}{\mu^2_{xy}\beta_y\eta_y}\right\}\log\frac{1}{\varepsilon}\right)\\
		&=& \mathcal{O}\left(\max\left\{1+\frac{1}{\mu_x\eta_x},\frac{L_{x}}{\mu^2_{xy}\eta_y},\frac{L^2_{xy}}{\mu^2_{xy}}\right\}\log\frac{1}{\varepsilon}\right) \\
		&=& \mathcal{O}\left(\max\left\{1+\sqrt{\frac{L_x}{\mu_x}}\frac{L_{xy}}{\mu_{xy}},\frac{L^2_{xy}}{\mu^2_{xy}}\right\}\log\frac{1}{\varepsilon}\right).
	\end{eqnarray*}
	In other words, the number of computations of prox is equal to
	\begin{equation*}
		\sharp\text{prox} = \mathcal{O}\left(\max\left\{1+\sqrt{\frac{L_x}{\mu_x}}\frac{L_{xy}}{\mu_{xy}},\frac{L^2_{xy}}{\mu^2_{xy}}\right\}\log\frac{1}{\varepsilon}\right),
	\end{equation*}	
and the number of gradient evaluations is	\begin{eqnarray*}
		\sharp \nabla G &=& k \cdot T\\
		&=& \cO\left(\max\left\{1+\sqrt{\frac{L_x}{\mu_x}}\frac{L_{xy}}{\mu_{xy}},\frac{L^2_{xy}}{\mu^2_{xy}}\right\}\log\frac{1}{\varepsilon}\right) \cdot \cO\left(\left(\frac{L_x}{\mu_x}\right)^{\nicefrac{1}{\alpha}}\right)\\
		&=&\cO\left(\max\left\{\left(\frac{L_x}{\mu_x}\right)^{\nicefrac{1}{\alpha}}+\left(\frac{L_x}{\mu_x}\right)^{\frac{2+\alpha}{2\alpha}}\frac{L_{xy}}{\mu_{xy}},\left(\frac{L_x}{\mu_x}\right)^{\nicefrac{1}{\alpha}}\frac{L^2_{xy}}{\mu^2_{xy}}\right\}\log\frac{1}{\varepsilon}\right).
	\end{eqnarray*}

\clearpage		
\section{Analysis of the Accelerated Primal-Dual Algorithm with Inexact Prox and Accelerated Gossip (Algorithm~\ref{alg:APDA-3})}

	
\subsection{Gossip matrices}	

In Algorithm~\ref{alg:APDA-3} we work with a more general notion of a gossip matrix (beyond graph Laplacians), defined next.

	\begin{assumption}[see \citep{scaman2017optimal}]
		\label{asm:gossip_matix}
			Let $\cG = (\cV, \cE)$ be a connected communication network. The communication process is represented via multiplication by a matrix $\hat{W}\in \R^{n\times n}$ which satisfies the following conditions:
			\begin{itemize}
				\item $\hat{W}$ is  symmetric,
				\item $\hat{W}$ is  positive semi-definite,
				\item the kernel of   $\hat{W}$ satisfies ${\rm Ker}~\hat{W} = {\rm span}\{(1,\dots,1)^{\top} \in \R^{n}\}$, and
				\item $\hat{W}_{i,j} \neq 0$ if  and only if  $i = j$ or $(i,j)\in \cE$.
			\end{itemize}
	\end{assumption}

It is easy to see that $\sigma(\hat{W}) \subset \sigma(W)$, where $\sigma(\cdot)$  denotes the spectrum,  and $W = \hat{W}\otimes I_{dn}$.

\subsection{Theorem~\ref{thm:APDA-3} follows from Theorem \ref{thm:APDA-Inex} }	

Our main result for Algorithm~\ref{alg:APDA-3}, i.e., Theorem~\ref{thm:APDA-3},  follows from Theorem \ref{thm:APDA-Inex}. Below we state it more formally.

	\begin{theorem}[Formal version of Theorem~\ref{thm:APDA-3}; convergence  for Algorithm \ref{alg:APDA-3}]
		\label{cor_conv_alg_3}
		Let us invoke the same assumptions as those made in Theorem~\ref{thm:APDA-Inex}. Additionally, let  Assumption~\ref{asm:gossip_matix} hold. Assume that the auxiliary problem \eqref{auxiliary_problem_acc_gossip} is solved by  one of the methods from Lemma~\ref{lem_convergence_norm_of_gradient}. Set the parameters $N$, $c_1$, $c_2$, $c_3$ to 
		\begin{equation*}
			N = \left\lfloor\sqrt{\chi}\right\rfloor,\quad c_1 = \frac{\sqrt{\chi}-1}{\sqrt{\chi}+1}, \quad  c_2=\frac{\sqrt{\chi}+1}{\sqrt{\chi}-1}, \quad c_3 = \frac{2\chi}{\left(1+ \chi\right)\lambda_{\max}\left(W\right)}
		\end{equation*}
		and let 
		\begin{equation*}
			\lambda_1 = 1+\frac{2c^N_1}{1+c^{2N}_1}, \quad \lambda_2 = 1-\frac{2c^N_1}{1+c^{2N}_1},
		\end{equation*}
		\begin{equation*}
			\beta_y = \min\left\{\frac{1}{L_x}, \frac{1}{2\lambda^2_1\eta_y}\right\},\quad \eta_x = \frac{1}{2\sqrt{L_x\mu_x}}\frac{\lambda_2}{\lambda_1}, \quad \eta_y = \frac{\sqrt{L_x\mu_x}}{\sqrt{2}\lambda_1\lambda_2},
		\end{equation*}
		\begin{equation*}
			\theta = \max\left\{\frac{2}{2+ \mu_x\eta_x}, 1 - \lambda^2_2\beta_y\eta_y\right\},
		\end{equation*}
		\begin{equation*}
			T = \sqrt[\alpha]{20A}\left(1+\sqrt{\frac{L_x}{\mu_x}}\right)^{\nicefrac{2}{\alpha}}.
		\end{equation*}
		Then for the Lyapunov function $\Delta^k$ from Theorem~\ref{th_conv_sc_1}, we have
		\begin{equation*}
			0\leq \Delta^{k}\leq \theta^k\Delta^0 \qquad \forall k\geq 0,
		\end{equation*}
		Moreover, for every $\varepsilon > 0$,  Algorithm~\ref{alg:APDA-3} finds $(\x^k, \y^k)$ for which $\Delta^k\leq\varepsilon\Delta^0$ in at most $\cO\left(\kappa^{\frac{2+\alpha}{2\alpha}}\log\left(\nicefrac{1}{\varepsilon}\right)\right)$ gradient computations and at most $\cO\left(\sqrt{\kappa\chi}\log\left(\nicefrac{1}{\varepsilon}\right)\right)$ communication rounds.
	\end{theorem}
	\begin{proof}
		The main idea of the proof is this methods is to show two following things: i) Theorem \ref{thm:APDA-Inex} holds true with some replacements, i.e. $L_{xy}\rightarrow \lambda_1$, $\mu_{xy} \rightarrow \lambda_2$, where $\lambda_1$, $\lambda_2$ is upper bound and lower bound  of the spectrum of a matrix, which is defined below; ii) Estimate $\lambda_1$ and $\lambda_2$. 
		
		The proof of Theorem~\ref{cor_conv_alg_3} is similar to the proof of Corollary~1 from \citep{kovalev2020optimal}. According to the proof from~\citep{kovalev2020optimal}, we need to replace 
		\begin{equation*}
				L_{xy} = \lambda_{\max}(\sqrt{W}) \rightarrow \lambda_1 = 1+\frac{2c^N_1}{1+c^{2N}_1}, \quad  \mu_{xy} =  \lambda^+_{\min}(\sqrt{W}) \rightarrow \lambda_2 = 1-\frac{2c^N_1}{1+c^{2N}_1}.
		\end{equation*}
		
	\end{proof}

	\clearpage

	\section{Proof of Lemma \ref{lem_convergence_norm_of_gradient}}
	\begin{proof}
	We consider the three options separately. We will use $f$ instead of $\Psi$ and $x$ instead of $w$ in the proof. Let $T=2K$.

\begin{itemize}
\item [(i)]	{\bf $K$ iterations of GD followed by $K$ iterations GD.} 	We consider the GD method with stepsize $\gamma=\frac{1}{L}$,
which performs the iterations
		\begin{equation*}
			x^{k+1} = x^k - \gamma \nabla f(x^k) = x^k - \frac{1}{L} \nabla f(x^k),
		\end{equation*}
where $L$ is the smoothness constant of $f$. We assume that $f$ is convex.

		{\bf $\bullet$ Gradient decreases monotonically.}		According to the  chain of inequalities
		\begin{eqnarray}
			\|\nabla f(x^{k+1})\|^2 &=& \|\nabla f(x^{k+1}) - \nabla f(x^{k})\|^2 + 2\la  \nabla  f(x^{k+1}) - \nabla f(x^{k}), \nabla f(x^k)\ra + \|\nabla f(x^k)\|^2\notag \\
			&=& \|\nabla f(x^{k+1}) - \nabla f(x^{k})\|^2 - \frac{2}{\gamma}\la \nabla f(x^{k+1}) - \nabla f(x^{k}), x^{k+1}-x^k\ra + \|\nabla f(x^k)\|^2\notag \\
			&\leq& \left(1 - \frac{2}{\gamma L}\right)\|\nabla f(x^{k+1}) - \nabla f(x^{k})\|^2 + \|\nabla f(x^k)\|^2 \notag \\
			&\leq& \|\nabla f(x^k)\|^2,\label{eq:basic_gd_2}
		\end{eqnarray}		
		where the first inequality follows from convexity and $L$-smoothness, 
		the gradient norm decreases monotonically.
				
		{\bf $\bullet$ Bound on best gradient norm.}	
	Further, using $L$-smoothness of $f$, we obtain
	\begin{eqnarray*}
		f(x^{k+1})&\leq& f(x^k)+ \la \nabla f(x^k), x^{k+1} - x^k\ra +\frac{L}{2}\|x^{k+1} - x^k\|^2\\
		&=&f(x^k)-\gamma \|\nabla f(x^k)\|^2 +\frac{\gamma^2L}{2}\|\nabla f(x^k)\|^2\\
		&=&f(x^k)-\frac{1}{2L}\|\nabla f(x^k)\|^2.
		\end{eqnarray*}
	Summing up from $k=K$  to $k = 2K$, we get 
	\begin{equation}
		\sum^{2K}_{k=K}\|\nabla f(x^k)\|^2 \leq 2L\left( f(x^K)- f(x^{2K}) \right) \leq 2L\left( f(x^K)- f^{\star}\right),\notag
	\end{equation}
which implies 
\begin{equation}
	\label{eq:basic_gd}
	\min_{k\in[K,2K]}\|\nabla f(x^k)\|^2 \leq \frac{2L\left( f(x^K)- f^{\star} \right)}{K+1}.
\end{equation}

	{\bf $\bullet$ Bound on function suboptimality.} 		It is known that for convex $L$-smooth functions, Gradient Descent (GD) with constant stepsize $\gamma = \frac{1}{L}$, i.e., the method
		\begin{equation*}
			x^{k+1} = x^k - \gamma \nabla f(x^k) = x^k - \frac{1}{L} \nabla f(x^k),
		\end{equation*}
	satisfies \citep{NesterovBook} 	
		\begin{equation}\label{eq:ug9gd8fd_098D}
			f(x^K) -f^{\star} \leq \frac{L\|x^0-x^{\star}\|^2}{2K} .
		\end{equation}

{\bf $\bullet$ Bound on last gradient norm.}	
By combining \eqref{eq:basic_gd_2}, \eqref{eq:ug9gd8fd_098D} and \eqref{eq:basic_gd}, we get
		\begin{eqnarray*}
			\|\nabla f(x^{2K})\|^2 & \overset{\eqref{eq:basic_gd_2}}{=} & \min_{k\in[K, 2K]}\|\nabla f(x^k)\|^2 \\
			&\overset{\eqref{eq:basic_gd}}{\leq} & \frac{2L(f(x^K) - f^{\star})}{K+1} \\
			& \leq & \frac{2L(f(x^K) - f^{\star})}{K} \\
			& \overset{\eqref{eq:ug9gd8fd_098D}}{\leq} & \frac{L^2\|x^0-x^{\star}\|^2}{K^2}.
		\end{eqnarray*}

\item [(ii)] {\bf $K$ iterations of FGD followed by $K$ iterations GD.} A better rate can be obtained if we replace the first half of the iterative process with the Fast Gradient Descent (FGD) method \citep{NesterovBook}. 	

{\bf $\bullet$ Bound on function suboptimality.} 	 If we employ the Fast Gradient Descent (FGD) method during the first $K$ iterations, we get (see, for example, \citep{NesterovBook}):
		\begin{equation}
			f(x^K) -f^{\star} \leq \frac{4L\|x^0-x^{\star}\|^2}{K^2}.\label{eq:h809fd0fd}
		\end{equation}
		
	{\bf $\bullet$ Bound on last gradient norm.} Since in the last $K$ iterations we use GD, inequalities \eqref{eq:basic_gd_2} and \eqref{eq:basic_gd} still apply. It remains to combine \eqref{eq:basic_gd_2}, \eqref{eq:basic_gd} and \eqref{eq:h809fd0fd}:		
\begin{eqnarray*}
			\|\nabla f(x^{2K})\|^2 & \overset{\eqref{eq:basic_gd_2}}{=} & \min_{k\in[K, 2K]}\|\nabla f(x^k)\|^2 \\
			&\overset{\eqref{eq:basic_gd}}{\leq} & \frac{2L(f(x^K) - f^{\star})}{K+1} \\
			& \leq & \frac{2L(f(x^K) - f^{\star})}{K} \\
			& \overset{\eqref{eq:h809fd0fd}}{\leq} &  \frac{8L^2\|x^0-x^{\star}\|^2}{K^3}.
		\end{eqnarray*}

\item	[(iii)]	{\bf $K$ iterations of FGD followed by $K$ iterations FSFOM.} A better rate can be obtained if we further replace the second half of the iterative process with the  FSFOM method of \citet{kim2021optimizing}. 	

{\bf $\bullet$ Bound on the gradient norm of FSFOM.}  The FSFOM method satisfies the following inequality  (see Theorem~6.1 from \citep{kim2021optimizing}):
		\begin{equation}
			\label{eq:conv_norm_fsfom}
			\|\nabla f(x^{2K})\|^2 \leq \frac{4L(f(x^K)-f^{\star})}{K^2}.
		\end{equation}
		
{\bf $\bullet$ Bound on last gradient norm.}		
		By combining  \eqref{eq:conv_norm_fsfom} and \eqref{eq:h809fd0fd}, we obtain
	$$
					\|\nabla f(x^{2K})\|^2 
					\overset{\eqref{eq:conv_norm_fsfom}}{\leq} \frac{4L(f(x^K) - f^{\star})}{K^2} 
					\overset{\eqref{eq:h809fd0fd}}{\leq} \frac{16L^2\|x^0-x^{\star}\|^2}{K^4}.
		$$
	
\end{itemize}

Results from parts (i), (ii) and (iii) proved above can be written in a unified form as \begin{equation*}
			\|\nabla f(x^{T})\|^2 \leq \frac{AL^2\|x^0-x^{\star}\|^2}{T^{\alpha}},
		\end{equation*}
		where $T=2K$, and $A$ is a constant which depends on the combination of the methods used during the first $K$ iterations and the last $K$ iterations.

	\end{proof}

\end{document}